\documentclass[11pt]{article}
\usepackage{graphicx} 

\usepackage{amsthm,amsfonts,amsmath,changepage,threeparttable,booktabs,multirow,pifont,xspace}
\usepackage{a4wide}
\usepackage{algorithm}
\usepackage{algorithmic}
\usepackage{authblk}
\usepackage{hyperref}

\usepackage{booktabs}
\usepackage{threeparttable}
\usepackage{multirow}
\usepackage{caption}
\usepackage{bbding}

\title{High-Probability Convergence in Decentralized Stochastic Optimization with Gradient Tracking}
\author{Aleksandar Armacki, Haoyuan Cai, Ali H. Sayed}
\affil{\'Ecole Polytechnique F\'ed\'erale de Lausanne, Lausanne, Switzerland, \\
\texttt{\{aleksandar.armacki, haoyuan.cai, ali.sayed\}@epfl.ch}}

\date{}

\usepackage{nicefrac}
\usepackage{amsmath}
\usepackage{amsthm}
\usepackage{amssymb}
\usepackage{amsfonts}
\usepackage{comment}
\usepackage{mathtools}
\usepackage{algorithm}
\usepackage{graphicx}
\usepackage{xspace}

\DeclareMathOperator*{\argmin}{arg\,min}

\usepackage{xcolor}
\usepackage{enumitem}

\usepackage{bbding}

\newtheorem{theorem}{Theorem}
\newtheorem{definition}{Definition}
\newtheorem{assumption}{Assumption}
\newtheorem{lemma}{Lemma}
\newtheorem{remark}{Remark}
\newtheorem{proposition}{Proposition}

\newcommand{\lp}{\left(}
\newcommand{\rp}{\right)}

\newcommand{\lcb}{\left\{}
\newcommand{\rcb}{\right\}}
\newcommand{\lbr}{\left[}
\newcommand{\rbr}{\right]}

\newcommand{\bigO}{\mathcal{O}}

\newcommand{\sfo}{$\mathcal{SFO}$\xspace}

\newcommand{\sgd}{\textbf{\texttt{SGD}}\xspace}
\newcommand{\dsgd}{\textbf{\texttt{DSGD}}\xspace}
\newcommand{\gtdsgd}{\textbf{\texttt{GT-DSGD}}\xspace}

\newcommand{\bg}{{\mathbf g}}
\newcommand{\bgt}{{\mathbf g^{t}}}

\newcommand{\bx}{{\mathbf x}}
\newcommand{\bxt}{{\mathbf x^{t}}}

\newcommand{\bxtp}{{\mathbf x^{t+1}}}

\newcommand{\bz}{{\mathbf z}}
\newcommand{\bzt}{{\mathbf z^{t}}}
\newcommand{\bztp}{{\mathbf z^{t+1}}}

\newcommand{\Prob}{\mathbb{P}}

\newcommand{\by}{{\mathbf y}}
\newcommand{\byt}{{\mathbf y^{t}}}
\newcommand{\bytp}{{\mathbf y^{t+1}}}

\usepackage{xcolor}

\usepackage{tikz}
\usetikzlibrary{calc,trees,positioning,arrows,chains,shapes.geometric,%
	decorations.pathreplacing,decorations.pathmorphing,shapes,%
	matrix,shapes.symbols}

\tikzstyle{startstop} = [rectangle, draw, rounded corners, align=center, minimum width=3cm, minimum height=1cm,text centered]
\tikzstyle{decision} = [diamond, draw, fill=blue!20, 
text width=4.5em, text badly centered, node distance=3cm, inner sep=0pt]
\tikzstyle{block} = [rectangle, draw, fill=blue!10, align=center, rounded corners, minimum width=3cm, minimum height=1cm]
\tikzstyle{blockcast} = [rectangle, draw, fill=red!10, align=center, rounded corners, minimum width=3cm, minimum height=0.45cm]
\tikzstyle{line} = [draw, -latex']
\tikzstyle{cloud} = [draw, ellipse,fill=red!20, node distance=3cm,
minimum height=2em]

\newcommand{\R}{\mathbb{R}}
\newcommand{\E}{\mathbb{E}}
\newcommand{\N}{\mathbb{N}}
\newcommand{\D}{\mathcal{D}}

\newcommand{\calN}{\mathcal{N}}

\newcommand{\Fx}{\mathcal{F}_X}
\newcommand{\Ft}{\mathcal{F}_t}

\newcommand{\bW}{\mathbf{W}}
\newcommand{\tbW}{\widetilde{\mathbf{W}}}
\newcommand{\bJ}{\mathbf{J}}
\newcommand{\bI}{\mathbf{I}}
\newcommand{\bA}{\mathbf{A}}
\newcommand{\bAt}{\mathbf{A}^t}
\newcommand{\ox}{\overline{x}}
\newcommand{\obx}{\overline{\bx}}

\newcommand{\obg}{\overline{\bg}}
\newcommand{\og}{\overline{g}}
\newcommand{\onab}{\overline \nabla}

\newcommand{\nbf}{\nabla \mathbf{f}}
\newcommand{\nbft}{\nabla \mathbf{f}^t}
\newcommand{\nbftp}{\nabla \mathbf{f}^{t+1}}

\newcommand{\onft}{\overline{\nabla} f^t}
\newcommand{\onbf}{\overline{\nabla} \mathbf{f}}
\newcommand{\onbft}{\overline{\nabla} \mathbf{f}^t}

\newcommand{\oz}{\overline{z}}
\newcommand{\ozt}{\overline{z}^t}
\newcommand{\obzt}{\overline{\bz}^t}
\newcommand{\oy}{\overline{y}}
\newcommand{\oyt}{\overline{y}^t}
\newcommand{\oby}{\overline{\by}}

\newcommand{\sigmax}{\sigma_{\max}}

\newcommand{\xit}{x^t_i}

\newcommand{\xitp}{x^{t+1}_i}
\newcommand{\yit}{y^t_i}

\newcommand{\git}{g^t_i}

\newcommand{\ogt}{\overline{g}^t}
\newcommand{\obgt}{\overline{\bg}^t}
\newcommand{\zit}{z^t_i}

\newcommand{\oxt}{\overline{x}^t}
\newcommand{\oxtp}{\overline{x}^{t+1}}
\newcommand{\obxt}{\overline{\bx}^t}
\newcommand{\obxtp}{\overline{\bx}^{t+1}}

\newcommand{\obyt}{\overline{\by}^t}
\newcommand{\obytp}{\overline{\by}^{t+1}}

\newcommand{\Ct}{C^t}
\newcommand{\oCT}{\overline{C}^T}
\newcommand{\oCTm}{\overline{C}^{T-1}}
\newcommand{\Dt}{D^t}
\newcommand{\oDT}{\overline{D}^T}
\newcommand{\Et}{E^t}
\newcommand{\oET}{\overline{E}^T}

\newcommand{\fs}{f^\star}
\newcommand{\Xt}{X^t}
\newcommand{\Xtp}{X^{t+1}}

\newcommand{\Yt}{Y^t}
\newcommand{\Ytp}{Y^{t+1}}

\begin{document}

\maketitle

\begin{abstract}
   We study high-probability (HP) convergence guarantees in decentralized stochastic optimization, where multiple agents collaborate to jointly train a model over a network. Existing HP results in decentralized settings almost exclusively focus on the Decentralized Stochastic Gradient Descent (\dsgd) algorithm, which requires strong assumptions, such as bounded data heterogeneity, or strong convexity of each agent's cost. This is contrary to the mean-squared error (MSE) results, where methods incorporating bias-correction techniques are known to converge under relaxed assumptions and achieve better practical performance. In this paper we provide the first step toward bridging the gap, by studying HP convergence of \dsgd incorporating the gradient tracking technique, in the presence of noise satisfying a relaxed sub-Gaussian condition. We show that the resulting method, dubbed \gtdsgd, achieves order-optimal HP convergence rates for both non-convex and Polyak-\L{}ojasiewicz costs, of order $\mathcal{O}\Big(\frac{\log(\nicefrac{1}{\delta})}{\sqrt{nT}}\Big)$ and $\mathcal{O}\Big(\frac{\log(\nicefrac{1}{\delta})}{nT}\Big)$, respectively, where $n$ is the number of agents, $T$ is the time horizon and $\delta \in (0,1)$ is the confidence parameter. Our results establish that \gtdsgd converges in the HP sense under the same conditions on the cost as in the MSE sense, while achieving comparable transient times. To the best of our knowledge, these are the first HP guarantees for decentralized optimization methods incorporating bias-correction. Numerical experiments on real and synthetic data verify our theoretical findings, underlining the superior performance of \gtdsgd and highlighting that the benefits of incorporating bias-correction are also maintained in the HP sense. 
\end{abstract}

\section{Introduction}

Decentralized learning is a popular paradigm where multiple agents collaborate to jointly train a model, encountered in a wide range of applications, including federated training of models on mobile devices \cite{konecny2016federated}, controlling and coordinating robot swarms \cite{bullo2009distributed}, or distributed control and power grids \cite{dist-learn-noncvx}. Formally, the problem can be cast as follows
\begin{equation}\label{eq:decentr-opt}
    \argmin_{x \in \R^d}\bigg\{ f(x) = \frac{1}{n}\sum_{i \in [n]} f_i(x)\bigg\},
\end{equation} where $x \in \R^d$ represents model parameters, $n \geq 2$ is the number of agents, while $f_i: \R^d \mapsto \R$ is the cost function of agent $i$. Many approaches have been proposed to solve \eqref{eq:decentr-opt} in the decentralized settings, with methods utilizing bias-correction techniques, such as gradient tracking (GT) \cite{scutari-next,nedic-tracking,qu-harnessing}, exact diffusion (ED) \cite{yuan-ed-1,yuan-ed-2,pmlr-v80-tang18a} and EXTRA \cite{shi-extra,nids}, being the most widely used. Compared to the classical \dsgd algorithm, which successively applies local (stochastic) updates and consensus steps, bias-correction methods apply state tracking techniques to either the gradient estimate or agents' models, enabling them to eliminate the effects of data heterogeneity and successfully track the global (i.e., average) network state. These methods are known to bring various benefits, such as exact convergence under a fixed step-size in the deterministic setting, e.g., \cite{nedic-tracking,qu-harnessing,yuan-ed-2,shi-extra}, as well as convergence under relaxed assumptions and improved transient times in the stochastic setting, e.g., \cite{pmlr-v80-tang18a,ran-improved,unified-refined}.     

While the average behaviour of these algorithms across a large number of runs has been well understood, e.g., in the mean-squared error (MSE) sense \cite{adaptive-learning-ali,Sayed_inference}, in recent times the focus has shifted to guarantees with respect to individual runs, motivated by applications such as neural networks and large language models, where a single training run is incredibly expensive in terms of both time and computational resources. Among them, high-probability (HP) convergence, e.g., \cite{harvey2019tight,li2020high,liu2023high,nguyen2023improved,sadiev2023highprobability} is particularly appealing, as for a non-negative process $\{X_t\}$, it establishes the following bound, for any $\delta \in (0,1)$, $t \geq 1$ and some $\beta > 0$
\begin{equation}\label{eq:conv-in-hp}
    \Prob\bigg(X_t > \frac{\log(\nicefrac{1}{\delta})}{t^\beta} \bigg) \leq \delta.
\end{equation} The guarantees stemming from \eqref{eq:conv-in-hp} are strong, showing a mild logarithmic dependence on the confidence parameter $\delta$, while also implying exponentially decaying tail bounds.\footnote{Setting $\epsilon = \frac{\log(\nicefrac{1}{\delta})}{t^\beta}$ in \eqref{eq:conv-in-hp} results in $\Prob\big(X_t > \epsilon\big) \leq \exp(-\epsilon t^\beta)$, for any $\epsilon > 0$. For comparison, using a MSE bound $\E[X_t^2] \leq t^{-\beta}$ and Chebyshev inequality implies $\Prob\big(X_t > \epsilon\big) \leq \frac{1}{\epsilon^2t^\beta}$, i.e., tails decaying polynomially in $t$.} As we highlight in the following literature review, although HP convergence has been extensively studied in centralized settings, only a handful of results exist in the decentralized setup. Moreover, the existing works primarily focus on HP guarantees of vanilla \dsgd, leaving HP convergence of the more powerful methods incorporating bias-correction largely unexplored.

\subsection{Literature Review}

We next review the literature on centralized HP and decentralized MSE and HP guarantees.

\begin{table*}[t]
\caption{Convergence to stationarity in decentralized non-convex optimization. All guarantees are stated in terms of the stationarity measure $\frac{1}{nT}\sum_{i \in [n]}\sum_{t \in [T]}\|\nabla f(\xit)\|^2$. ``Result'' indicates the original type of guarantee provided in the work. ``Heterogeneity'' indicates whether results hold under arbitrary heterogeneity level (\ding{52}), or if an explicit heterogeneity bound is required (\ding{55}). ``HP Rate'' indicates the HP rate with respect to $n$, $T$ and the confidence parameter $\delta \in (0,1)$, where the HP rate from MSE results is obtained via Markov's inequality. ``Transient Time'' indicates the time needed to cancel the effects of network connectivity and attain the global rate $\mathcal{O}\Big(\frac{1}{\sqrt{nT}}\Big)$, with $\lambda \in [0,1)$ being the network connectivity parameter (see Section \ref{sec:problem} ahead for a formal definition). For fairness of comparison with other HP results, we present our results under sub-Gaussian noise, i.e., with $\rho = 0$ (see Definition \ref{def:relax-sub-Gauss} ahead for details).}
\label{tab:non-conv}
\begin{adjustwidth}{-1in}{-1in} 
\begin{center}
\begin{threeparttable}
\begin{small}
\begin{sc}
\begin{tabular}{cccccc}
\toprule
\multicolumn{1}{c}{\scriptsize Work} & \multicolumn{1}{c}{\scriptsize Method} & \multicolumn{1}{c}{\scriptsize Result} & \multicolumn{1}{c}{\scriptsize Heterogeneity} & \multicolumn{1}{c}{\scriptsize HP Rate} & \multicolumn{1}{c}{\scriptsize Transient Time}\\
\midrule
\scriptsize\cite{dsgd-high-prob} & \scriptsize \dsgd & \scriptsize HP & \ding{55}$^\dagger$ & \scriptsize $\mathcal{O}\Big(\frac{\log(\nicefrac{1}{\delta})}{\sqrt{T}}\Big)$ & \scriptsize $-^\ddagger$ \\
\midrule
\scriptsize\cite{armacki2025dsgd} & \scriptsize \dsgd & \scriptsize HP & \ding{55} & \scriptsize $\mathcal{O}\Big(\frac{\log(\nicefrac{1}{\delta})}{\sqrt{nT}}\Big)$ & \scriptsize $\mathcal{O}\Big(\frac{n^3}{(1-\lambda)^4}\Big)$ \\
\midrule
\scriptsize\cite{ran-improved} & \scriptsize \gtdsgd & \scriptsize MSE & \ding{52} & \scriptsize $\mathcal{O}\Big(\frac{1}{\delta\sqrt{nT}}\Big)$ & \scriptsize $\mathcal{O}\Big(\frac{n^3}{(1-\lambda^2)^6}\Big)^\S$ \\
\midrule
\scriptsize Theorem \ref{thm:main-non-cvx} & \scriptsize \gtdsgd & \scriptsize HP & \ding{52} & \scriptsize $\mathcal{O}\Big(\frac{\log(\nicefrac{1}{\delta})}{\sqrt{nT}}\Big)$ & \scriptsize $\mathcal{O}\Big(\frac{n^3}{(1-\lambda^2)^8}\Big)$ \\
\bottomrule
\end{tabular}
\end{sc}
\end{small}
\begin{tablenotes}\scriptsize
    \item[$\dagger$] The authors in \cite{dsgd-high-prob} require uniformly bounded gradients, implying bounded heterogeneity.

    \item[$\ddagger$] The authors in \cite{dsgd-high-prob} establish rates with no linear speed-up in the number of agents, hence no transient times can be derived from their rates.

    \item[$\S$] The authors in \cite{ran-improved} establish a transient time from a \gtdsgd-specific analysis. This was improved to $\mathcal{O}\Big(\max\Big\{\frac{n^3}{(1-\lambda)^2}, \frac{n}{(1-\lambda)^{8/3}}\Big\}\Big)$ in \cite{unified-refined}, through a unified analysis framework. Since we focus specifically on \gtdsgd, we compare with the transient time from \cite{ran-improved}.  
\end{tablenotes}
\end{threeparttable}
\end{center}
\vskip -0.1in
\end{adjustwidth}
\end{table*}

\paragraph{Centralized HP Results.} HP convergence guarantees have been studied extensively in the centralized setting. For example, the authors in \cite{nemirovski2009robust} show the optimal convergence rate $\mathcal{O}\Big(\frac{\log(\nicefrac{1}{\delta})}{\sqrt{T}}\Big)$ of \sgd for convex costs under light-tailed (i.e., sub-Gaussian) stochastic gradients, while the works \cite{ghadimi2013stochastic,li2020high} show the same optimal rate for \sgd and momentum \sgd under light-tailed noise and non-convex costs. The authors in \cite{harvey2019tight} and \cite{pmlr-v206-bajovic23a} respectively provide optimal HP convergence rates $\mathcal{O}\Big(\frac{\log(\nicefrac{1}{\delta})}{T}\Big)$ of the last iterate of \sgd for non-smooth and smooth strongly convex costs. The work \cite{liu2023high} generalizes and unifies previous HP results on non-convex and convex costs, providing order-optimal rates for several algorithms, including \sgd and AdaGrad, for both smooth non-convex and (possibly) non-smooth convex costs, while the authors in \cite{liu2024revisiting} establish unified HP convergence of \sgd for smooth and non-smooth convex and strongly convex costs. A related line of works study HP convergence under noise with tails heavier than sub-Gaussian, starting with \cite{madden2020high,li2022high,eldowa2024general}, which consider HP convergence of \sgd under sub-Weibull noise, while works like \cite{gorbunov2020stochastic,parletta2022high} study clipped \sgd in the presence of noise with bounded variance. This is further generalized in works like  \cite{cutkosky2021high,sadiev2023highprobability,nguyen2023improved,liu2023breaking,hubler2025normalization,kornilov2025sign}, which study variants of clipped, normalized or sign \sgd in the presence of noise with bounded moment of order $p \in (1,2]$, establishing rates depending on the noise moment,\footnote{If the noise has bounded $p$-th moment, the authors either establish rates of order $\mathcal{O}\Big(T^{\frac{2(1-p)}{3p-2}}\Big)$, or if a mini-batch is used, the oracle complexity for reaching an $\epsilon$-stationary point, e.g., \cite{arjevani2022}, of order $\mathcal{O}\Big(\epsilon^{\frac{3p-2}{1-p}}\Big)$.} as well as  \cite{gorbunov2023breaking,armacki2024_ldp+mse,armacki2025optimal,armacki2025high}, where the authors consider noise with potentially unbounded moments and additional structure, like symmetry, establishing rates independent of noise moments.\footnote{In the sense that the rates have constant exponents, independent of noise, e.g., of order $\mathcal{O}\Big(T^{-1/2}\Big)$ in \cite{armacki2025optimal}.}

\begin{table*}[t]
\caption{Convergence to optimum in decentralized P\L{}/strongly convex optimization. All guarantees are stated in terms of the optimality measure $\frac{1}{n}\sum_{i \in [n]}\big(f(x_i^T) - \fs\big)$. ``Result'' indicates the original type of guarantee provided in the work. ``Local Conditions'' indicates whether smoothness is the only condition needed on agents' local costs $f_i$ (\ding{52}), or if additional conditions on local costs are required (\ding{55}). ``HP Rate'' indicates the HP rate with respect to $n$, $T$ and $\delta$, where the HP rate from MSE results is obtained via Markov's inequality. ``Transient Time'' indicates the time needed to cancel the effects of network connectivity and attain the global rate $\mathcal{O}\Big(\frac{1}{nT}\Big)$. For fairness of comparison with other HP results, we present our results for sub-Gaussian noise, i.e., with $\rho = 0$ (see Definition \ref{def:relax-sub-Gauss} ahead for details).}
\label{tab:PL}
\begin{adjustwidth}{-1in}{-1in} 
\begin{center}
\begin{threeparttable}
\begin{small}
\begin{sc}
\begin{tabular}{cccccc}
\toprule
\multicolumn{1}{c}{\scriptsize Work} & \multicolumn{1}{c}{\scriptsize Method} & \multicolumn{1}{c}{\scriptsize Result} & \multicolumn{1}{c}{\scriptsize Local Conditions} & \multicolumn{1}{c}{\scriptsize HP Rate} & \multicolumn{1}{c}{\scriptsize Transient Time}\\
\midrule
\scriptsize\cite{dsgd-high-prob} & \scriptsize \dsgd & \scriptsize HP & \ding{55}$^\dagger$ & \scriptsize $\mathcal{O}\Big(\frac{\log(\nicefrac{1}{\delta})}{T}\Big)$ & \scriptsize $-$ \\
\midrule
\scriptsize\cite{armacki2025dsgd} & \scriptsize \dsgd & \scriptsize HP & \ding{55}$^\ddagger$ & \scriptsize $\mathcal{O}\Big(\frac{\log(\nicefrac{1}{\delta})}{nT}\Big)$ & \scriptsize $\mathcal{O}\Big(\max\Big\{\frac{n}{1-\lambda},\frac{1}{(1-\lambda)^2}\Big\}\Big)$ \\
\midrule
\scriptsize\cite{ran-improved} & \scriptsize \gtdsgd & \scriptsize MSE & \ding{52} & \scriptsize $\mathcal{O}\Big(\frac{1}{nT\delta}\Big)$ & \scriptsize $\mathcal{O}\Big(\frac{n}{(1-\lambda^2)^3}\Big)^\S$ \\
\midrule
\scriptsize Theorem \ref{thm:main-PL} & \scriptsize \gtdsgd & \scriptsize HP & \ding{52} & \scriptsize $\mathcal{O}\Big(\frac{\log(\nicefrac{1}{\delta})}{nT}\Big)$ & \scriptsize $\mathcal{O}\Big(\frac{n^{1+\epsilon}}{(1-\lambda^2)^{4+2\epsilon}}\Big)^\P$ \\
\bottomrule
\end{tabular}
\end{sc}
\end{small}
\begin{tablenotes}\scriptsize
    \item[$\dagger$] The authors in \cite{dsgd-high-prob} require local costs $f_i$ to have uniformly bounded gradients across the entire sample path, i.e., that there exists a $G > 0$, such that $\|\nabla f_i(\xit)\| \leq G$, for all $i \in [n]$ and $t \geq 1$.

    \item[$\ddagger$] The authors in \cite{armacki2025dsgd} require local costs $f_i$ to be twice continuously differentiable and strongly convex.

    \item[$\S$] Similarly to the comments in Table \ref{tab:non-conv}, the authors in \cite{unified-refined} provide an improved transient time for \gtdsgd, of order $\mathcal{O}\Big(\max\Big\{\frac{n}{1-\lambda}, \frac{1}{(1-\lambda)^{4/3}}\Big\}\Big)$, but for fairness, we compare with results from \cite{ran-improved}.  

    \item[$\P$] The value $\epsilon \in (0,1]$ is a tunable constant, implying that our transient time can be made arbitrarily close to $\bigO\Big(\frac{n}{(1-\lambda^2)^4} \Big)$. For further details, see Remark \ref{rmk:trans-time-PL} and Appendix \ref{sup:trans-time}.
\end{tablenotes}
\end{threeparttable}
\end{center}
\vskip -0.1in
\end{adjustwidth}
\end{table*}

\paragraph{Decentralized MSE Results.} MSE convergence is perhaps the most widely encountered result, typically studied under different conditions on the second noise moment, e.g., \cite{adaptive-learning-ali,khaled2022better}. Following one of these settings, the authors in \cite{conv-rates-dsgd-jakovetic} show that \dsgd converges at rate $\mathcal{O}\Big(\frac{1}{T}\Big)$ for strongly convex costs over random graphs. The works \cite{vlaski-dsgd-nonconv-1,vlaski-dsgd-nonconv-2} study MSE guarantees of \dsgd for non-convex costs over directed, strongly connected graphs and show that it can escape saddle points with high probability. The work \cite{wang-cooperative} proposes a general framework over undirected networks, dubbed cooperative \sgd, establishing, among others, the optimal rate $\mathcal{O}\Big(\frac{1}{\sqrt{nT}}\Big)$ for non-convex costs. The authors in \cite{koloskova2023grad_clip} provide unified guarantees for \dsgd with local updates and changing network topology, with optimal rates and linear speed-up in the number of agents for non-convex and (strongly) convex costs. Notably, all of the said works require variations of the bounded heterogeneity condition for non-convex costs, which is relaxed for strongly convex/Polyak-\L{}ojasiewicz (PL) costs, at the expense of requiring each agent's local cost to be strongly convex/P\L{}. These additional requirements are removed by endowing \dsgd with bias-correction techniques, like GT and ED. For example, the authors in \cite{ran-vr} show that \gtdsgd using variance reduction converges at a linear rate for strongly convex costs, while the authors in \cite{ran-improved} establish optimal rates of \gtdsgd with linear speed-up in the number of agents, for both non-convex and P\L{} costs, of order $\mathcal{O}\Big(\frac{1}{\sqrt{nT}}\Big)$ and $\mathcal{O}\Big(\frac{1}{nT}\Big)$. The authors in \cite{pmlr-v80-tang18a} provide matching order-optimal convergence rate for the method dubbed $D^2$, which combines \dsgd and the ED bias-correction technique, over undirected graphs and non-convex costs. The work \cite{unified-refined} provides unified guarantees over undirected networks for a general framework dubbed SUDA, which subsumes \gtdsgd and $D^2$ as special cases, notably showing improved transient times for both non-convex and P\L{} costs. The work \cite{yu2026decentralized} shows that \gtdsgd with local momentum and normalization achieves optimal rates under heavy-tailed noise with bounded moment $p \in (1,2]$ and undirected graphs. It is worth mentioning a rich line of works studying MSE guarantees for decentralized problems like estimation, detection, multi-objective and multi-task optimization, see, e.g., \cite{consensus+innovation1,chen-decent-pareto,nassif-multitask} and references therein. 

\paragraph{Decentralized HP Results.} Compared to MSE guarantees, there is a significantly smaller body of work on HP convergence of decentralized algorithms. For example, the authors in \cite{dsgd-high-prob} study HP convergence of \dsgd over directed graphs, for general non-convex and P\L{} costs under light-tailed noise, requiring uniformly bounded gradients, as well as asymptotically vanishing noise for P\L{} costs, while showing rates $\mathcal{O}\Big(\frac{\log(\nicefrac{1}{\delta})}{\sqrt{T}}\Big)$ and $\mathcal{O}\Big(\frac{\log(\nicefrac{1}{\delta})}{T}\Big)$. The works \cite{dsgd-online-noncooperative-games,dsgd-online-stochastic} study HP convergence of a decentralized mirror descent algorithm under light-tailed noise, for online noncooperative games and dynamic regrets, respectively, requiring bounded gradients and compact domains. The authors in \cite{clipped-dual-avg,clipped-dsgd2} study HP convergence for convex costs under heavy-tailed noise with bounded $p$-th moments, of clipped dual averaging and clipped \dsgd, respectively, while the authors in \cite{clipped-dsgd1} provide HP guarantees for clipped \dsgd for online learning under noise with bounded $p$-th moment, for both general convex and non-convex costs, requiring both uniformly bounded gradients and iterates. A common thread for all these works is the need for uniformly bounded gradients, or algorithmic modifications like gradient clipping, which ensure that the gradients stay bounded, as well as the fact that none of the said works establish linear speed-up in the number of agents. A recent work \cite{armacki2025dsgd} made a breakthrough in this regard, showing that \dsgd over undirected graphs converges with order-optimal HP rates $\mathcal{O}\Big(\frac{\log(\nicefrac{1}{\delta})}{\sqrt{nT}}\Big)$ and $\mathcal{O}\Big(\frac{\log(\nicefrac{1}{\delta})}{nT}\Big)$, achieving linear speed-up for both non-convex and strongly convex costs, while relaxing the various strong assumptions used in prior works. However, this result still suffers from the well-known deficiencies of \dsgd, requiring an explicit bound on the heterogeneity for non-convex costs, as well as additional conditions on costs of each agent in the strongly convex case. Notably, none of the existing works study HP convergence of the more powerful methods using bias-correction techniques, which are known to converge under relaxed assumptions in the MSE sense and exhibit superior performance in practice. Before closing this section, we mention a related line of work that study asymptotic large deviations behaviour in decentralized problems like detection, inference and social learning, see, e.g., \cite{bajovic-detection-ld,bajovic-detection-ld2,matta-diffusion-1,matta-difussion-2,asl2021,bajovic-inference-ld,social-learning-book}, and references therein.

\subsection{Contributions}

In this work we provide the first HP convergence result for decentralized methods using bias-correction, by studying a variant of \dsgd incorporating the GT technique (\gtdsgd), under noise satisfying a relaxed sub-Gaussian condition. Our contributions are as follows.

\begin{enumerate}
    \item We study HP convergence guarantees of \gtdsgd over networks inducing a symmetric communication matrix, in the presence of noise satisfying a relaxed sub-Gaussian condition (see Definition \ref{def:relax-sub-Gauss} ahead). The novel condition allows the moment-generating function (MGF) of the noise to grow with the norm-squared of the gradient of the global cost and recovers the classical sub-Gaussian (i.e., light-tailed) noise as a special case.

    \item For general non-convex costs we show that \gtdsgd removes the heterogeneity bound required for HP convergence of \dsgd, e.g., \cite{armacki2025dsgd}. Further, we establish the HP convergence of the agent and time averaged gradient norm-squared, with the rate $\mathcal{O}\Big(\frac{\log(\nicefrac{1}{\delta})}{\sqrt{nT}}\Big)$, achieving linear speed-up in the number of agents and matching the known order-optimal MSE rates of \gtdsgd. Our bounds indicate HP transient times which are comparable to transient times stemming from MSE rates in \cite{ran-improved}. A detailed comparison with existing HP and MSE results for non-convex costs is provided in Table \ref{tab:non-conv}.

    \item For P\L{} costs we show that \gtdsgd removes the need for additional conditions on agents' local costs required for HP convergence of \dsgd in, e.g., \cite{dsgd-high-prob,armacki2025dsgd}. Moreover, we establish the HP convergence of the average of agents' last iterate, with the rate $\mathcal{O}\Big(\frac{\log(\nicefrac{1}{\delta})}{nT}\Big)$, again achieving linear speed-up and matching the MSE rates. Similarly, our bounds indicate HP transient times comparable to transient times stemming from MSE rates in \cite{ran-improved}. A detailed comparison with existing HP and MSE results for P\L{}/strongly convex costs is provided in Table \ref{tab:PL}.

    \item We provide numerical experiments on both synthetic and real data, verifying several facets of our theoretical results, including exponential tail probability decay and linear speed-up in the HP sense. Further, we demonstrate that our theory holds in real, mini-batch setting, where the noise is not necessarily sub-Gaussian, and show that \gtdsgd consistently performs better than vanilla \dsgd, confirming that the benefits of bias-correction methods observed in the MSE sense are maintained in the HP sense.     
\end{enumerate}

\paragraph{Technical Challenges and Novelty.} Toward establishing our results, we faced multiple challenges, resolved by introducing several novelties, as outlined next. Compared to \cite{dsgd-high-prob,armacki2025dsgd}, where the authors study HP convergence of \dsgd, we incorporate the GT technique, which requires a significantly different analysis to account for relaxed assumptions and the interplay of state and tracker variables stemming from the GT. To that end, we establish novel bounds on the MGFs of the consensus gap and the gradient tracker accounting for their joint dynamics, in the form of Lemmas \ref{lm:MGF-decay-PL} and \ref{lm:bdd-MGF-PL}. Next, we introduce a relaxed sub-Gaussian condition, which allows the noise MGF to grow with the norm of the global gradient, and provide Lemmas \ref{lm:noise-properties}-\ref{lm:avg-noise-properties}, that characterize its properties. Compared to the work \cite{ran-improved}, which studies MSE convergence of \gtdsgd, we study HP guarantees, which requires a fundamentally different approach, by analyzing the MGF of the process of interest. To that end, we introduce several novel technical results, like the aforementioned ones, as well as Lemma \ref{lm:mgf-bound-str-cvx}.

\paragraph{Paper Organization.} The rest of the paper is organized as follows. Section \ref{sec:problem} states the problem, Section \ref{sec:method} introduces the \gtdsgd method, Section \ref{sec:main} presents the main results, Section \ref{sec:num} provides numerical results, while Section \ref{sec:conc} concludes the paper. Appendix contains results omitted from the main body. The remainder of this section introduces some notation.    

\paragraph{Notation.} We use $\N$, $\R$ and $\R^d$ to denote positive integers, real numbers and $d$-dimensional vectors, respectively. For $m \in \N$, we use $[m] = \{1,\ldots,m\}$ to denote positive integers up to and including $m$. The notation $\langle \cdot,\cdot\rangle$ stands for the Euclidean inner product, while $\|\cdot\|$ is used for both vector and matrix induced norms. We use $\mathbf{1}_n$ and $I_d$ to denote the $n$-dimensional vector of ones and $d$-dimensional identity matrix, respectively. Subscripts are used to denote agents, while superscripts denote the iteration counter, e.g., $\xit$ refers to the model of agent $i$ in iteration $t$. The ``big O'' notation $\bigO(\cdot)$ hides only global constants, unless stated otherwise.

\section{Problem Setup}\label{sec:problem}

In this section we formally introduce the problem of decentralized stochastic optimization. As discussed in the introduction, we consider a network of $n \geq 2$ agents that want to jointly train a model, by solving \eqref{eq:decentr-opt}. Communication links between agents are modeled via a static graph $G = (V,E)$, where $V = [n]$ is the set of vertices (i.e., agents), while $E \subseteq V \times V$ is the set of edges (i.e., communication links). The graph $G$ induces a weight matrix $W \in \R^{n \times n}$, whose $(i,j)$-th component $[W]_{ij} = w_{ij}$ is strictly positive if $(i,j) \in E$, otherwise $w_{ij} = 0$. We further make the following assumption on the weight matrix.

\begin{assumption}\label{asmpt:network}
    The weight matrix $W \in \R^{n \times n}$ is primitive and doubly stochastic.
\end{assumption}

Assumption \ref{asmpt:network} is a standard assumption in decentralized optimization, satisfied by weight matrices induced by connected undirected graphs, as well as a class of strongly-connected directed graphs with doubly stochastic weights, e.g., \cite{ran-vr}. Let $\lambda \coloneqq \|W - J\|$ denote the network connectivity parameter, where $J \coloneqq \frac{1}{n}\mathbf{1}_n\mathbf{1}_n^\top \in \R^{n \times n}$ is the ideal communication matrix. It can be shown that Assumption \ref{asmpt:network} implies $\lambda \in [0,1)$, e.g., \cite{Horn_Johnson_2012}. Next, we state the assumptions on the local and global cost functions.

\begin{assumption}\label{asmpt:cost}
    The global cost is bounded from below and agents' costs have $L$-Lipschitz gradients, i.e., $\fs \coloneqq \inf_{x \in \R^d}f(x) > -\infty$ and for all $i \in [n]$ and $x, y\in \R^d$, we have
    \begin{equation*}
        \|\nabla f_i(x) - \nabla f_i(y)\| \leq L\|x - y\|.
    \end{equation*}
\end{assumption}

Assumption \ref{asmpt:cost} is standard for general non-convex costs, e.g., \cite{ghadimi2013stochastic,ran-improved,unified-refined}. We will also use the following assumption in some results, known as the P\L{} condition, e.g., \cite{karimi2016linear}.

\begin{assumption}\label{asmpt:PL}
    There exists a $\mu > 0$, such that for every $x \in \R^d$, the global cost satisfies
    \begin{equation*}
        2\mu\big(f(x)-\fs\big) \leq \|\nabla f(x)\|^2.
    \end{equation*}
\end{assumption}

Assumption \ref{asmpt:PL} is a popular instance of structured non-convexity conditions and subsumes as a special case strongly convex costs. To solve \eqref{eq:decentr-opt}, we assume each agent has access to a Stochastic First-order Oracle (\sfo), which, when queried by agent $i \in [n]$ with input $x \in \R^d$, returns $g_i \in \R^d$, a stochastic estimator of the true gradient $\nabla f_i(x)$. The \sfo model is widely used and subsumes the following popular learning settings.

\paragraph{1. Batch (i.e., offline) learning:} the cost of each agent $i \in [n]$ is given by $f_i(x) = \frac{1}{m_i}\sum_{r \in [m_i]}\ell(x;\xi_{i,r})$, where $\{\xi_{i,r}\}_{r \in [m_i]} \subset \Xi_i$ is a finite dataset, with $\ell_i: \R^d \times \Xi_i \mapsto \R$ being a loss function. When queried, the \sfo chooses a set of indices $S_i \subset [m_i]$ uniformly at random and outputs $g_i = \frac{1}{|S_i|}\sum_{r^\prime \in S_i}\nabla \ell(x;\xi_{i,r^\prime})$, where $1 \leq |S_i| < m_i$;

\paragraph{2. Streaming (i.e., online) learning:} the cost of each agent $i \in [n]$ is given by $f_i(x) = \E_{\xi_i \sim \D_i}\big[ \ell_i(x;\xi_i) \big]$, where $\xi_i \in \Xi_i$ is a random variable following an unknown distribution $\D_i$. When queried, the \sfo generates a mini-batch $\{\xi_{i,r}\}_{r \in S_i}$ of independent, identically distributed copies of $\xi_i$ and outputs $g_i = \frac{1}{|S_i|}\sum_{r \in S_i}\nabla \ell_i(x;\xi_{i,r})$, where $|S_i| \geq 1$. \\

Prior to stating our assumptions on the estimator returned by the \sfo, we define an important concept, namely that of \emph{relaxed sub-Gaussian} random vectors.

\begin{definition}\label{def:relax-sub-Gauss}
    Let $X, Z \in \R^d$ be random vectors with $Z \coloneqq Z(X)$ and let $\Fx \coloneqq \sigma(X)$ be the $\sigma$-algebra induced by $X$. We say that $Z$ is $(\sigma,\rho)$-sub-Gaussian (with respect to $X$), if for some $\sigma >0$, $\rho \geq 0$ and $\mathcal{T}: \R^d \mapsto \R^d$, we have 
    \begin{equation*}
        \E\bigg[\exp\bigg(\frac{\|Z\|^2}{\sigma^2} \bigg)  \: \big\vert \: \Fx \bigg] \leq \exp\big(1 + \rho\|\mathcal{T}(X)\| \big), \: \text{ almost surely}.    
    \end{equation*}
\end{definition}

Definition \ref{def:relax-sub-Gauss} is inspired by observations that the noise often grows with the cost sub-optimality and the various second noise moment conditions used in MSE analysis, e.g., \cite{adaptive-learning-ali,Sayed_inference,khaled2022better}. For example, if $\mathcal{T}(X) = \nabla f(X)$, then the above condition allows the noise MGF to grow with the stationarity gap of $X$, while if $\rho = 0$, one recovers the usual $\sigma$-sub-Gaussianity definition, e.g., \cite{vershynin_2018}. To the best of our knowledge, this condition has not previously been considered in the HP literature. We provide some properties of random vectors satisfying Definition \ref{def:relax-sub-Gauss} in Section \ref{subsec:prelim}. Prior to stating the assumption on the noise, let $z_i^t \coloneqq g_i^t - \nabla f_i(x_i^t)$ denote the stochastic noise at agent $i \in [n]$ and time $t \geq 1$, and let $\Ft \coloneqq \sigma\big(\{\{x_i^1\}_{i \in [n]},\ldots,\{x_i^t\}_{i \in [n]}\}\big)$ be the $\sigma$-algebra induced by agents' models up to time $t$.

\begin{assumption}\label{asmpt:noise}
    At any time $t \geq 1$, the stochastic quantities satisfy the following.
    \begin{enumerate}[leftmargin=*, label=\arabic*.]
        \item The random samples $\{\xi^k_i \}_{i \in [n],k \in [t]}$ are independent across agents and iterations. 
        
        \item The stochastic noise is conditionally zero-mean, i.e., $\E[\zit \: \vert \: \Ft ] = 0$, almost surely.

        \item The noise at agent $i \in [n]$ and time $t \geq 1$ is $(\sigma_i,\rho\alpha_t^{2+\varepsilon})$-sub-Gaussian for some $\varepsilon > 0$, i.e., 
        \begin{equation*}
            \E\bigg[\exp\bigg(\frac{\|\zit\|^2}{\sigma_i^2}\bigg) \: \Big\vert \: \Ft \bigg] \leq \exp\big(1 + \alpha_t^{2+\varepsilon}\rho\|\nabla f(\xit)\| \big), \: \text{ almost surely}.
        \end{equation*}
    \end{enumerate}
\end{assumption}
    
The first two conditions in Assumption \ref{asmpt:noise} are standard in (decentralized) stochastic optimization. The third condition requires the noise $\zit$ to be relaxed sub-Gaussian with respect to $\xit$ and $\mathcal{T}(\cdot) = \nabla f(\cdot)$, generalizing the standard sub-Gaussian condition used in both centralized and decentralized settings, e.g., \cite{li2020high,liu2023high,pmlr-v206-bajovic23a,dsgd-high-prob,armacki2025dsgd}. Finally, note that unlike the MSE analysis, where the noise variance is typically allowed to grow unrestrictedly with the gradient norm, i.e., $\E[\|\zit\|^2 \: \vert \: \Ft] \leq \sigma^2 + \rho^2\|\nabla f(\xit)\|^2$, Assumption \ref{asmpt:noise} requires the gradient norm to be controlled by the step-size, which is necessitated by the properties of the resulting noise MGF, see Remark \ref{rmk:noise} ahead for a further discussion.

\section{The Gradient Tracking Method}\label{sec:method}

In this section we formally introduce the \gtdsgd method, which incorporates the GT technique into the standard \dsgd algorithm. 

\begin{algorithm}[tb]
\caption{\gtdsgd}
\label{alg:dsgd}
\begin{algorithmic}[1]
   \REQUIRE{Model initialization $x^1_i \in \R^{d}$, gradient tracker initialization $y_i^0 = 0$ and $g_i^0 = 0$, $i \in [n]$, step-size schedule $\{\alpha_t\}_{t \in \N}$;}
   \FOR{$t = 1,2,\ldots$, each agent $i \in [n]$ in parallel}
        \STATE Query the \sfo to obtain $\git$;  

        \STATE Perform the gradient tracker update: $\yit = \sum_{j \in \calN_i}w_{ij}\big(y_j^{t-1} + g^t_j - g^{t-1}_j \big)$;
        
        \STATE Perform the model update: $\xitp = \sum_{j \in \calN_i}w_{ij}\big(x_j^t - \alpha_ty_j^t \big)$;
    \ENDFOR
\end{algorithmic}
\end{algorithm}

Gradient tracking is a popular bias-correction technique, extensively studied in the decentralized optimization literature, e.g., \cite{scutari-next,nedic-tracking,qu-harnessing,ran-improved,unified-refined,Pu2021,ran-vr,jakovetic-unification,xin-general}. Informally, when using GT, each agent maintains and updates two variables: (i) the local model and (ii) the local gradient tracker, with the goal of the second one being to track the global, network-average gradient. The version of \gtdsgd considered in our work is based on the Adapt-Then-Combine (i.e., diffusion) approach \cite{diffusion1,diffusion2,diffusion3}, and formally consists of the following steps. At the start of training, agents choose a shared step-size schedule $\{\alpha_t\}_{t \in \N}$, a deterministically selected initial local model $x^1_i \in \R^d$,\footnote{The initial models can be any real vectors, different across agents, but they need to be deterministic quantities, for the sake of theoretical analysis in Section \ref{sec:main} ahead.} and initialize their local gradient tracker with zero vector, i.e., $y_i^0 = 0$. In iteration $t \geq 1$, each agent $i \in [n]$ queries the \sfo with their current model $\xit$ and receives $\git$. Agents then update their local gradient tracker according to the rule
\begin{equation}\label{eq:gt-track-update}
    y_i^{t} = \sum_{j \in \calN_i}w_{ij}\big(y_i^{t-1} + g_j^t - g_j^{t-1}\big),
\end{equation} where $\calN_i \coloneqq \{j \in V: \{i,j\} \in E\} \cup \{i\}$ represents the set of agents (i.e., \emph{neighbours}) with whom agent $i$ can communicate (including $i$ itself). The new model at each agent is then produced by performing the following diffusion step 
\begin{equation}\label{eq:gt-model-update}
    \xitp = \sum_{j \in \calN_i}w_{ij}\big(x^{t}_j - \alpha_t y_j^t\big).
\end{equation} The full \gtdsgd method is summarized in Algorithm \ref{alg:dsgd}.

\section{Main Results}\label{sec:main}

In this section we present the main results. Subsection \ref{subsec:prelim} states the preliminaries, Subsection \ref{subsec:non-conv} provides results for non-convex costs, while Subsection \ref{subsec:PL} presents results for P\L{} costs.

\subsection{Preliminaries}\label{subsec:prelim} 

In this section we provide some results used in the analysis, starting with properties of the noise satisfying Definition \ref{def:relax-sub-Gauss}. For ease of notation, let $\E_{X}[\:\cdot\:] \coloneqq \E[\:\cdot\:\vert\Fx]$ and $\Prob_{X}(\:\cdot\:) \coloneqq \Prob(\:\cdot\:\vert\Fx)$ denote the expectation and probability conditioned on the $\sigma$-algebra induced by $X$.

\begin{lemma}\label{lm:noise-properties}
    Let $X_i, Z_i \in \R^d$, $i \in [n]$, be random vectors, with $Z_i \coloneqq Z_i(X_i) \in \R^d$ being zero-mean and $(\sigma_i,\rho_i)$-sub-Gaussian conditioned on $X_i$. Then the following statements hold.
    \begin{enumerate}
        \item For any $\mathcal{F}_{X_i}$-measurable $v \in \R^d$, we have
        \begin{equation}\label{eq:mgf-inner-prod}
            \E_{X_i}\big[\exp\lp\langle v, Z_i \rangle\rp \big] \leq \exp\lp \frac{3\sigma_i^2\|v\|^2}{4} + \rho_i\|\mathcal{T}(X_i)\| \rp .
        \end{equation}

        \item For $\epsilon > 0$, we have
        \begin{equation}\label{eq:tail-bdd}
            \Prob_{X_i}\big(\|Z_i\| > \epsilon\big) \leq 2\exp\bigg(-\frac{\epsilon^2}{2\sigma_i^2} + \frac{\rho_i}{2}\|\mathcal{T}(X_i)\| \bigg)
        \end{equation}

        \item For any $p \in \N$, we have
        \begin{equation}\label{eq:p-th-moment}
            \E_{X_i}\|Z_i\|^{2p} \leq (2p)^{p+1}\sigma_i^{2p}\exp\bigg(\frac{\rho_i}{2}\|\mathcal{T}(X_i)\| \bigg). 
        \end{equation}
    
        \item Let $\Fx = \sigma\big(\big\{X_1,\ldots,X_n\big\}\big)$ denote the sigma-algebra generated by $X_i$'s and let $\sigma^2 = \sum_{i \in [n]}\sigma_i^2$. If $Z_1,\ldots,Z_n$ are independent conditioned on $\Fx$, then for any $\epsilon > 0$, we have
        \begin{equation}\label{eq:mgf-norm}
            \E_X\bigg[\exp\bigg(\frac{n\|\overline{Z}\|^2}{96\sigma^2}\bigg)\bigg] \leq 2d\exp\Bigg(1 + \frac{\sum_{i \in [n]}\rho_i^2\|\mathcal{T}(X_i)\|^2}{32\sigma^2}\Bigg).
        \end{equation}
    \end{enumerate}
\end{lemma}

\begin{remark}\label{rmk:noise}
    In view of Lemma \ref{lm:noise-properties}, one can now infer the need for $\rho$ to be multiplied by the step-size in Assumption \ref{asmpt:noise}: the first property in Lemma \ref{lm:noise-properties} indicates that any factor multiplying $\langle v, Z\rangle$ will impact only the first term in the resulting bound, i.e., for any $a \in \R$
    \begin{equation*}
        \E_X\big[\exp\lp a\langle v, Z \rangle\rp \big] \leq \exp\lp \frac{3a^2\sigma^2\|v\|^2}{4} + \rho\|\mathcal{T}(X)\| \rp .
    \end{equation*} In particular, if $Z$ is the stochastic gradient noise, i.e., $Z(X) = g(X) - \nabla f(X)$, with $v = \mathcal{T}(X) = \nabla f(X)$ and $a = \alpha_t$, we then have
    \begin{equation*}
        \E_X\big[\exp\lp \alpha_t\langle \nabla f(X), Z \rangle\rp \big] \leq \exp\lp \frac{3\alpha_t^2\sigma^2\|\nabla f(X)\|^2}{4} + \rho\|\nabla f(X)\| \rp,
    \end{equation*} implying that we lose control of the growth of the gradient norm, which is necessary in the analysis. For comparison, this is much easier to handle in the MSE analysis, as the inner product vanishes, since $\nabla f(X)$ is $\Fx$-measurable and $Z$ is conditionally zero-mean.  
\end{remark}

\begin{remark}\label{rmk:dimension}
    The fourth property in Lemma \ref{lm:noise-properties} is crucial to achieving linear speed-up in the number of agents, as it shows that the noise parameter of the averaged noise scales as $\bigO\Big(\frac{\sigma^2}{n}\Big)$. However, the bound in \eqref{eq:mgf-norm} also introduces a mild logarithmic dependence on the problem dimension $d$ in the final bounds (see Theorems \ref{thm:main-non-cvx}-\ref{thm:main-PL} ahead), which is not present in MSE results. We note that this is a consequence of working with the MGF of the average of (relaxed) sub-Gaussian vectors, also present in, e.g., \cite{jin2019short,armacki2025dsgd}, and not a byproduct of our analysis. As noted in \cite{jin2019short}, it is not clear if this dependence can be fully removed.
\end{remark}

Next, we define some important concepts used in the analysis. Let $\oxt \coloneqq \frac{1}{n}\sum_{i \in [n]}\xit$ denote the (ideal) network-average model at time $t$. Using \eqref{eq:gt-track-update}-\eqref{eq:gt-model-update}, the fact that $W$ is doubly stochastic and $y^0_i = g_i^0 = 0$ for all $i \in [n]$, it can be readily seen that $\oyt = \ogt$ and hence
\begin{equation*}
    \oxtp = \oxt - \alpha_t\oyt = \oxt - \alpha_t\ogt,
\end{equation*} where $\ogt \coloneqq \frac{1}{n}\sum_{i \in [n]}\git$ is the network-average stochastic gradient. Additionally, let $\ozt \coloneqq \frac{1}{n}\sum_{i \in [n]}z_i^t$ and $\onft \coloneqq \frac{1}{n}\sum_{i \in [n]}\nabla f_i(\xit)$ denote the network-average noise and gradient, noting that $\ogt = \onft + \ozt$, and let $\sigma^2 \coloneqq \frac{1}{n}\sum_{i \in [n]}\sigma^2_i$ and $\sigmax^2 \coloneqq \max_{i \in [n]}\sigma^2_i$ denote the averaged and maximum noise parameter. We then have the following important result.

\begin{lemma}\label{lm:avg-noise-properties}
    If Assumption \ref{asmpt:noise} holds, then for any $t \geq 1$ and any $\Ft$-measurable $v \in \R^d$, the average noise $\ozt$ satisfies the following.
    \begin{enumerate}
        \item $\E\big[\exp\big(\langle v, \ozt \rangle \big) \: \vert \: \Ft \big] \leq \exp\bigg(\frac{3\sigma^2\|v\|^2}{4n} + \alpha_t^{2+\varepsilon}n\rho + \alpha_t^{2+\varepsilon}\rho\sum_{i \in [n]}\|\nabla f(\xit)\|^2 \bigg)$.
    
        \item $\E\bigg[\exp\bigg(\frac{n\|\ozt\|^2}{96\sigma^2} \bigg) \: \Big\vert \: \Ft \bigg] \leq 2d\exp\bigg(1 + \frac{\alpha_t^{4+2\varepsilon}\rho^2}{32\sigma^2}\sum_{i \in [n]}\|\nabla f(\xit)\|^2 \bigg)$.
    \end{enumerate}
\end{lemma}

\subsection{Non-convex Costs}\label{subsec:non-conv}

In this section we present results for non-convex costs. Throughout this section, we use a fixed step-size, i.e., $\alpha_t \equiv \alpha$, for all $t \geq 1$. We start by providing the following descent inequality.

\begin{lemma}\label{lm:descent-inequality}
    Let Assumption \ref{asmpt:cost} hold. If $\alpha \leq \frac{1}{4L}$, we have, for any $t \geq 1$
    \begin{align*}
        f(\oxtp) \leq \underbrace{f(\oxt) - \frac{\alpha}{2}\| \nabla f(\oxt)\|^2 + \alpha\langle \nabla f(\oxt), \ozt \rangle + \alpha^2L\|\ozt\|^2}_{\text{centralized}} + \underbrace{\frac{\alpha L^2}{2n}\sum_{i \in [n]}\|\xit - \oxt\|^2 - \frac{\alpha}{4}\|\onft\|^2}_{\text{decentralized}}. 
    \end{align*}
\end{lemma}

Lemma \ref{lm:descent-inequality} is a standard result and the starting point in our analysis. The right-hand side of the above inequality consists of terms that arise in centralized \sgd plus two terms stemming from the decentralized nature of \gtdsgd. The next result provides a bound on the consensus gap $\sum_{i \in [n]}\|\xit - \oxt\|^2$. Prior to stating the result, recall that the initial models $x^1_i \in \R^d$, $i \in [n]$, are selected deterministically and let $\Delta_x \coloneqq \frac{1}{n}\sum_{i \in [n]}\|x^1_i - \ox^1\|^2$ denote the initial consensus gap. We then have the following bound.

\begin{lemma}\label{lm:consensus-bound}
    Let Assumptions \ref{asmpt:network} and \ref{asmpt:cost} hold. If $\alpha \leq \frac{(1-\lambda^2)^2}{16\lambda^2L\sqrt{3}}$, we have, for any $T \geq 2$
    \begin{align*}
        &\quad\frac{1}{n}\sum_{t \in [T]}\sum_{i \in [n]}\|\xit - \oxt\|^2 \leq \frac{4\Delta_x}{(1-\lambda^2)} + \frac{32\alpha^2\lambda^2}{n(1-\lambda^2)^3}\sum_{i \in [n]}\|y^1_i - \oy^1\|^2 \\ 
        &+ \frac{512\alpha^2\lambda^4}{n(1-\lambda^2)^4}\sum_{t \in [T]}\sum_{i \in [n]}\|\zit\|^2 + \frac{768\alpha^4\lambda^4L^2}{(1-\lambda^2)^4}\sum_{t \in [T]}\Big(\|\onft\|^2 + \|\ozt\|^2 \Big).
    \end{align*}
\end{lemma}

Lemma \ref{lm:consensus-bound} provides a deterministic bound on the consensus gap induced by \gtdsgd. Let $\Delta_f \coloneqq f(\ox^1) - \fs$ and $\Delta_g \coloneqq \frac{1}{n}\sum_{i \in [n]}\|\nabla f_i(x^1_i) - \frac{1}{n}\sum_{j \in [n]} \nabla f_j(x_j^1)\|^2$ denote the global optimality gap and the average gradient heterogeneity at the initial models, respectively. Armed with Lemmas \ref{lm:descent-inequality}-\ref{lm:consensus-bound}, we are ready to state the main result.

\begin{theorem}\label{thm:main-non-cvx}
    Let Assumptions \ref{asmpt:network}, \ref{asmpt:cost} and \ref{asmpt:noise} hold. If for any $T \geq 2$ the step-size is chosen as $\alpha = \min\Big\{\frac{\sqrt{n}}{\sqrt{T}},C \Big\}$, where $0 < C \leq \min\Big\{\frac{(1-\lambda^2)^2}{16\lambda^2L\sqrt{3}}, \frac{1-\lambda^2}{4\lambda L\sqrt[4]{12}},\frac{(1-\lambda^2)^{4/3}}{4\lambda^{4/3}L\sqrt[3]{12}}, \frac{\sqrt[3]{n}(1-\lambda^2)^{4/3}}{\lambda^{4/3}\sigmax^{2/3}L^{2/3}\sqrt[3]{1614}}, \linebreak \frac{1}{4L}, \frac{n}{9\sigma^2} \sqrt{\frac{n}{282e\sigma^2dL}}, \frac{\sigma\sqrt{32}}{\rho}, \sqrt[1+\varepsilon]{\frac{1}{16n\rho^2}} \Big\}$, then for any $\delta \in (0,1)$, with probability at least $1 - \delta$
    \begin{align*}
          &\quad\quad\quad\frac{1}{nT}\sum_{t \in [T]}\sum_{i \in [n]}\|\nabla f(\xit)\|^2 \leq \bigO\bigg(\frac{\log(\nicefrac{1}{\delta}) + \Delta_f + \sigma^2L\big(1 + \log(2d)\big)}{\sqrt{nT}} \\ 
          &+ \frac{n^{\nicefrac{(3+\varepsilon)}{2}}\rho}{T^{\nicefrac{(1+\varepsilon)}{2}}}+ \frac{\log(\nicefrac{1}{\delta}) + \Delta_f}{CT} + \frac{L^2\Delta_x}{n(1-\lambda^2)T} + \frac{n\sigma^2\lambda^4L^2}{(1-\lambda^2)^4T} + \frac{nL^2\Delta_g}{d(1-\lambda^2)^3T^2}\bigg).
    \end{align*}
\end{theorem}

Theorem \ref{thm:main-non-cvx} establishes the HP convergence guarantee of \gtdsgd for general non-convex costs, with leading term of the order $\bigO\Big(\frac{\log(\nicefrac{d}{\delta})}{\sqrt{nT}}\Big)$. Some remarks are now in order.

\begin{remark}
    The bound in Theorem \ref{thm:main-non-cvx} captures the dependence on many problem parameters, such as optimality and consensus gaps $\Delta_f$ and $\Delta_x$, heterogeneity $\Delta_g$, smoothness $L$, noise $\sigma^2$ and network connectivity through $(1-\lambda^2)$. Compared to MSE result \cite{ran-improved}, the bound exhibits the same dependence on standard problem parameters, e.g., the network connectivity only affects higher-order terms, with an additional mild logarithmic dependence on the problem dimension, as discussed in Remark \ref{rmk:dimension}. Further, we can see that the relaxed sub-Gaussainity parameters $\rho$ and $\varepsilon$ affect the final rate and the step-size, via the terms $\frac{\rho n^{\nicefrac{(3+\varepsilon)}{2}}}{T^{\nicefrac{(1+\varepsilon)}{2}}}$ and $\min\Big\{\frac{\sigma\sqrt{32}}{\rho},\sqrt[1+ \varepsilon]{\frac{1}{16n\rho^2}}\Big\}$, respectively, doing so in a coupled manner, i.e., appearing together. In the special case of standard sub-Gaussian noise, i.e., $\rho = 0$, the first term becomes zero, while the second becomes infinity, hence we get the expression $C \leq \min\{A,+\infty\} = A$, where $A$ captures the other problem related terms in the definition of $C$, independent of both $\varepsilon$ and $\rho$. 
\end{remark}

\begin{remark}
    The leading term in the above rate is of the order $\bigO\Big(\frac{1}{\sqrt{nT}}\Big)$, showing that linear speed-up is achieved in the HP sense. Moreover, the rate in Theorem \ref{thm:main-non-cvx} can be used to derive the transient time of \gtdsgd in the HP sense, which is of order $\bigO\Big(\max\Big\{\frac{n^3}{(1-\lambda^2)^8}, \rho^{\frac{2}{\varepsilon}}n^{\frac{4+\varepsilon}{\varepsilon}}\Big\}\Big)$. The second term in the expression stems from the relaxed sub-Gaussianity condition and can be eliminated when reduced to standard sub-Gaussianity, i.e., $\rho = 0$. In this case, our transient time becomes $O\Big(\frac{n^3}{(1-\lambda^2)^8}\Big)$, showing the same dependence on the number of agents and a slightly worse dependence on network connectivity compared to the transient time $\bigO\Big(\frac{n^3}{(1-\lambda^2)^6}\Big)$ of \gtdsgd implied by the MSE rate in \cite{ran-improved}. This gap stems from the fact that the MSE analysis deals directly with the quantities of interest, while the HP analysis deals with their MGFs, significantly complicating the analysis and resulting in slightly looser bounds. The reader is referred to Appendix \ref{sup:trans-time} for detailed derivations of the transient time and further discussion on the gap in transient times stemming from MSE and HP results.
\end{remark}

\begin{remark}\label{rmk:transient-time-dsgd}
    Compared to recent works \cite{dsgd-high-prob,armacki2025dsgd}, which provide HP guarantees for vanilla \dsgd, our results differ in the following. First, we establish convergence under relaxed assumptions on the cost, by removing the bounded gradients/heterogeneity conditions required in \cite{dsgd-high-prob,armacki2025dsgd}, showing that the benefits of bias-correction for decentralized methods carry over in the HP sense. Next, we consider a more general noise condition, which allows the noise MGF to grow with the norm of the gradient, subsuming the standard sub-Gaussian noise used in \cite{dsgd-high-prob,armacki2025dsgd}. Further, the use of GT necessitates a significantly different and more involved analysis to account for the joint dynamics of the state and tracker variables. Finally, we show that \gtdsgd achieves linear speed-up, with a slightly worse transient time than that of \dsgd from \cite{armacki2025dsgd}. However, this is a known downside of \gtdsgd-specific analysis, that can potentially be rectified via a unified analysis framework, which is beyond the scope of the current work.\footnote{For the MSE case, the \gtdsgd-specific analysis, e.g., \cite{ran-improved}, implies the transient time $\bigO\Big(\frac{n^3}{(1-\lambda^2)^6}\Big)$, which is strictly worse than the \dsgd transient time $\bigO\Big(\frac{n^3}{(1-\lambda^2)^4}\Big)$ from \cite{pmlr-v119-koloskova20a}. As mentioned in Table \ref{tab:non-conv}, the transient time of \gtdsgd can be improved to $\mathcal{O}\Big(\max\Big\{\frac{n^3}{(1-\lambda)^2}, \frac{n}{(1-\lambda)^{8/3}}\Big\}\Big)$, by considering a unified framework, e.g., \cite{unified-refined}, which subsumes \gtdsgd as a special case. However, this would require a fundamentally different analysis approach and different recursions, which is beyond the scope of our work and an interesting future direction.}
\end{remark}

\begin{remark}
    Compared to the MSE results for \gtdsgd in \cite{ran-improved}, our rate is of the same order, exhibiting an additional mild dependence on the dimension and a slightly worse transient time, as discussed in previous remarks. On the technical side, our results are derived under different noise conditions, which require fundamentally different analysis approach and working directly with the MGF of the process of interest, as well as different intermediate results, highlighted and discussed further throughout the Appendix (see, e.g., Remarks \ref{rmk:sharp-constants} and \ref{rmk:trans-time-PL-analysis}). 
\end{remark}

\subsection{P\L{} Costs}\label{subsec:PL}

In this section we provide guarantees for P\L{} costs. Throughout this section, we assume that agents have a shared model initialization, i.e., $x^1_i = x_j^1$, for all $i,j \in [n]$, which can be easily achieved, e.g., by deploying a gossip algorithm before the start of training, and implies that $\Delta_x = 0$. Furthermore, we use a time-varying step-size schedule $\{\alpha_t\}_{t \in \N}$, which will be specified later. We start by providing a refined descent inequality for P\L{} costs.

\begin{lemma}\label{lm:descent-inequality-PL}
    Let Assumptions \ref{asmpt:cost} and \ref{asmpt:PL} hold. If $\alpha_t \leq \frac{1}{2L}$, for all $t \geq 1$, we then have
    \begin{align*}
        f(\oxtp) - \fs \leq \underbrace{(1-\alpha_t\mu)\big(f(\oxt) - \fs\big) - \alpha_t\langle \nabla f(\oxt), \ozt \rangle + \alpha_t^2L\|\ozt\|^2}_{\text{centralized}} + \underbrace{\frac{\alpha_t L^2}{2n}\sum_{i \in [n]}\|\xit - \oxt\|^2}_{\text{decentralized}}. 
    \end{align*}
\end{lemma}

Similarly to Lemma \ref{lm:descent-inequality} in the previous section, Lemma \ref{lm:descent-inequality-PL} is a deterministic inequality used as a starting point in our analysis, again consisting of terms that arise from centralized \sgd plus the consensus gap. Our next result bounds the consensus gap, however, due to the additional structure of P\L{} costs, we provide a stronger result than in Lemma \ref{lm:consensus-bound}, by directly bounding the MGF of the consensus gap. The formal result is stated next.

\begin{lemma}\label{lm:MGF-decay-PL}
    Let Assumptions \ref{asmpt:network}-\ref{asmpt:noise} hold. Define $K_{t+1} \coloneqq (t+t_0+2)K$, for some $K > 0$. If the step-size is chosen as $\alpha_t = \frac{a}{\mu(t+t_0)}$, for some $a > 0$ and $t_0 \geq \max\Big\{\frac{16a^2\lambda^2L^2K}{\mu^2(1-\lambda^2)^4},\frac{12}{1-\lambda^2}\Big\}$, we then have, for all $t \geq 1$ and any $\nu \in (0,1]$
    \begin{align*}
        \E\bigg[\exp\bigg(\frac{\nu K_{t+1}}{n}\|\bxtp - \obxtp\|^2\bigg)\bigg] \leq \exp\big(\nu\alpha_{t+1}^2K_{t+1}C  \big),
    \end{align*} where $C \coloneqq \max\Big\{\frac{32\lambda^4L^2(\sigma^2 + 4L\Delta_f)}{(1-\lambda^2)^2}, \frac{3072\sigma^2\lambda^4L^2(2 + \log(2d))}{(1-\lambda^2)^4}, \frac{6\lambda^2L^2\Delta_f}{(1-\lambda^2)^4}, \frac{6\lambda^2L^2(2+\log(2d))}{(1-\lambda^2)^4} \Big\}$.
\end{lemma}

Lemma \ref{lm:MGF-decay-PL} shows that the MGF of the scaled consensus gap is bounded by $\bigO\big(e^{\alpha_t^2K_t}\big)$. Denote by $\kappa \coloneqq \frac{L}{\mu}$ the condition number. We are now ready to state the main result.

\begin{theorem}\label{thm:main-PL}
    Let Assumptions \ref{asmpt:network}-\ref{asmpt:noise} hold and let the step-size be given by $\alpha_t = \frac{a}{\mu(t+t_0)}$, where $a \geq 6$ and $t_0 \geq \max\Big\{\frac{12}{1-\lambda^2}, \frac{64na^3\kappa^3\lambda^2L}{(1-\lambda^2)^4}, \frac{9216\sigmax^4\lambda^4}{n^2}, \frac{576a^4\kappa^2}{\mu^2}, \frac{384a^2\sigma^2\kappa}{\mu}, \frac{a\sqrt[4+2\varepsilon]{12\rho^2L}}{\mu}, \frac{a\sqrt[5+2\varepsilon]{18\rho^2L^2}}{\mu}, \linebreak 2a\kappa, \frac{6a}{\mu}, \frac{2a\sqrt{L}\max\{4\sigma\sqrt{3},3\sqrt{2L},48\sigma\lambda\sqrt{3L}\}}{\mu\sqrt{n}}, \frac{a\sqrt[2+\varepsilon]{\rho}\max\{\sigma_{\max}^{\nicefrac{2}{(2+\varepsilon)}},1/\sqrt[2+\varepsilon]{32\sigma^2}\}}{\mu}, \frac{a\sqrt[1+\varepsilon]{4n\rho}\max\{1,1/\sqrt[1+\varepsilon]{\mu},\sqrt[1+\varepsilon]{4\kappa}\}}{\mu}, \linebreak 2a\kappa\max\big\{3L,640\lambda^2\big\}, \frac{2a\lambda\kappa\sqrt{3}\max\{\sqrt{\kappa},8\lambda L\sqrt{5}\}}{(1-\lambda^2)^2} \Big\}$. Then for any $\delta \in (0,1)$ and $T \geq 2$, with probability at least $1 - \delta$ 
    \begin{align*}
        \frac{1}{n}\sum_{i \in [n]}\big(f(x_i^T)-\fs\big) = \mathcal{O}\bigg(\frac{\nu^{-1}\log(\nicefrac{2}{\delta}) + \sigma^2\kappa\big(1+\log(2d)\big)/\mu}{n(T+t_0)} + \frac{\kappa C(1 + \kappa)}{\mu(T+t_0)^2} + \frac{(t_0+1)^{a/2}\Delta_f}{(T+t_0)^{a/2}} \bigg),
    \end{align*} where $\nu \coloneqq \min\Big\{1,\frac{n}{96\sigmax^2\lambda^2},\frac{\mu}{24a^2\kappa}\Big\}$ and $C > 0$ is defined in Lemma \ref{lm:MGF-decay-PL}.
\end{theorem} 

Theorem \ref{thm:main-PL} establishes HP convergence guarantees of the last iterate of \gtdsgd for P\L{} costs, with an improved rate of the order $\bigO\Big(\frac{\log(\nicefrac{d}{\delta})}{nT}\Big)$. Some remarks are now in order.

\begin{remark}
    Theorem \ref{thm:main-PL} utilizes the additional structure of the cost to provide an improved rate compared to Theorem \ref{thm:main-non-cvx}, of the order $\bigO\Big(\frac{1}{nT}\Big)$, which is optimal for P\L{}/strongly convex costs. The bound in Theorem \ref{thm:main-PL} captures the dependence on the same problem parameters as in the MSE result \cite{ran-improved}, with an additional $\log(d)$ factor, which is again a consequence of sub-Gaussianity and the HP analysis. We can also see that the relaxed sub-Gaussainity parameters $\rho$ and $\varepsilon$ affect the constant $t_0$ in a couple manner, i.e., they appear jointly and can be readily eliminated under the standard sub-Gaussian condition, i.e., when $\rho = 0$. 
\end{remark}

\begin{remark}\label{rmk:trans-time-PL}
    The leading term is of the order $\bigO\Big(\frac{1}{nT}\Big)$, showing that linear speed-up is achieved in the HP sense. Moreover, it can be shown that the rate in Theorem \ref{thm:main-PL} results in the transient time $\bigO\Big(\frac{n^{(a+2)/(a-2)}}{(1-\lambda^2)^{4a/(a-2)}}\Big)$ for \gtdsgd in the HP sense. Choosing $a = \frac{4+2\epsilon}{\epsilon}$, for any $\epsilon \in (0,1]$, we get the transient time $\bigO\Big(\frac{n^{1+\epsilon}}{(1-\lambda^2)^{4 + 2\epsilon}}\Big)$ which can be made arbitrarily close to $\bigO\Big(\frac{n}{(1-\lambda^2)^{4}}\Big)$, at the cost of slightly larger multiplicative constants in the rate. Compared to the transient time $\bigO\Big(\frac{n}{(1-\lambda^2)^3}\Big)$ of \gtdsgd implied by the MSE rate in \cite{ran-improved}, our result (almost) matches the dependence on the number of agents, with a slightly worse dependence on the network connectivity, again stemming from differences in noise conditions and the analysis. For a detailed derivation of the transient time, the reader is referred to Appendix \ref{sup:trans-time}.
\end{remark}

\begin{remark}
    Compared to HP guarantees for vanilla \dsgd in \cite{dsgd-high-prob,armacki2025dsgd}, we again relax the assumptions, by removing the bounded gradients/strong convexity of individual costs required in \cite{dsgd-high-prob,armacki2025dsgd}, showing that the benefits of bias-correction are maintained in the HP sense. Further, compared to \cite{dsgd-high-prob}, who analyze HP convergence of \dsgd for P\L{} costs and noise that asymptotically vanishes,\footnote{The authors in \cite{dsgd-high-prob} make two strong assumptions in the case of P\L{} costs, namely that: (i) the local gradients are uniformly bounded path-wise, i.e., $\|\nabla f_i(\xit)\| \leq G$ for every $i \in [n]$ and $t \geq 1$, almost surely; (ii) the noise is $(\alpha_t\sigma)$-sub-Gaussian, i.e., $\E\Big[\exp\Big(\frac{\|\zit\|^2}{\alpha_t^2\sigma^2}\Big)\Big] \leq \exp(1)$, where $\alpha_t = \frac{1}{t+t_0}$ is the step-size. This condition implies (via Markov's inequality) that $\Prob\Big(\|\zit\|^2 \leq \frac{\sigma^2(\log(\nicefrac{1}{\delta) + 1)}}{(t + t_0)^2} \Big) \geq 1 - \delta$, for any $\delta \in (0,1)$.} we significantly relax the noise condition, while compared to \cite{armacki2025dsgd}, who consider strongly convex costs and sub-Gaussian noise, we simultaneously relax the cost and noise condition, by studying P\L{} costs and relaxed sub-Gaussian noise. As discussed in Table \ref{tab:PL} and Remark \ref{rmk:transient-time-dsgd}, while our transient times for \gtdsgd are worse than those for \dsgd in \cite{armacki2025dsgd}, this is a byproduct of the \gtdsgd-specific analysis.  
\end{remark}

\section{Numerical Results}\label{sec:num}

In this section we provide numerical results. Subsection \ref{subsec:num-synth} presents results on synthetic data, while Subsection \ref{subsec:num-real} presents results on real data.

\subsection{Synthetic Data}\label{subsec:num-synth}

\paragraph{Methodology.} We consider an illustrative case of local quadratic costs, given by $f_i(x) = \frac{1}{2}x^\top A_ix + b_i^\top x$, where $A_i \in \R^{d \times d}$ is positive definite, making each $f_i$ strongly convex. The weight matrix $W \in \R^{n \times n}$ is computed from an undirected graph $G = (V,E)$, using the Metropolis-Hastings weight scheme, e.g., \cite{XIAO200465}. When queried by user $i \in [n]$ in iteration $t \geq 1$, the \sfo returns a noisy gradient of $f_i$ evaluated at $\xit$, i.e., $\git = A_i\xit + b_i + \zit$, where $\zit \in \R^d$ is a zero-mean Gaussian random vector. We compare the performance of \gtdsgd and \dsgd, with both using a time-varying step-size $\alpha_t = \frac{1}{t+1}$ and shared initialization $x^1_i = 0$, for all $i \in [n]$, with the aim of testing the following facets of our theory.
\begin{enumerate}
    \item \emph{Exponentially decaying tails} - we aim to verify that the tail probability decays at an exponential scale.

    \item \emph{Linear speed-up} - we aim to verify that the tail probability decays faster as the number of agents increases.  
\end{enumerate}

\begin{figure*}[t]
\centering
\begin{tabular}{ccc}
\includegraphics[scale=0.3]{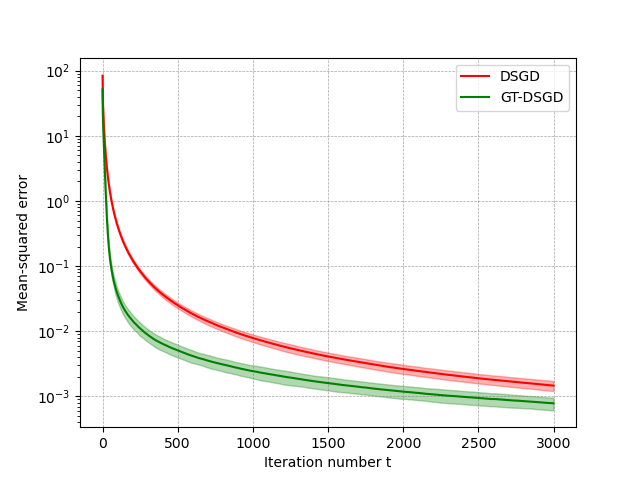}
&
\includegraphics[scale=0.3]{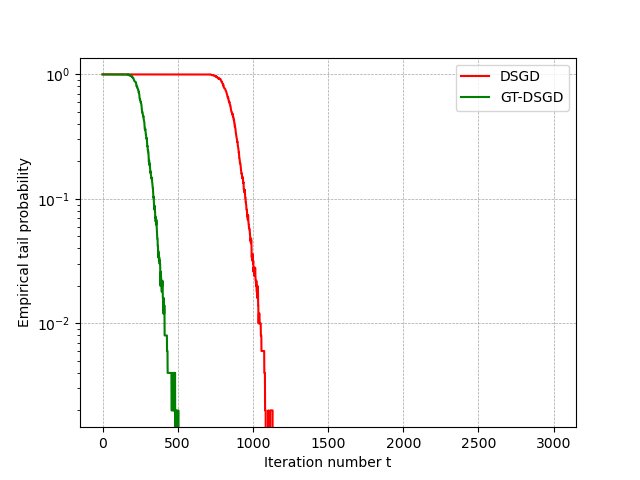}
&
\includegraphics[scale=0.3]{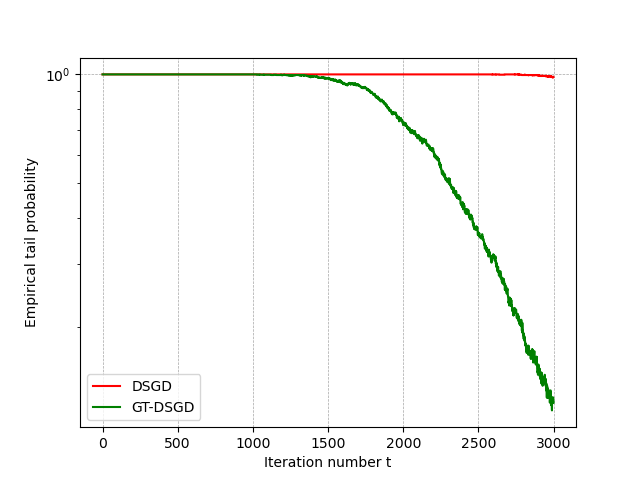}
\end{tabular}
\caption{Exponential tail decay on synthetic data. Left to right: MSE performance and tail decay with threshold $\epsilon = \big\{0.01,0.001\big\}$. We can see that \gtdsgd achieves faster tail decay, across all values of threshold $\epsilon$, illustrating the benefit of bias-correction in the HP sense.}
\label{fig:synth-exp}
\end{figure*}

To verify the two facets, we use the \emph{empirical tail probability} $\Prob^t_{n,\epsilon}$, computed as follows. We first run each method for $T$ iterations and repeat it over $R$ runs. Next, for any $\epsilon > 0$ and each $t \in [T]$, we compute $\Prob^t_{n,\epsilon} = \frac{1}{R}\sum_{r \in [R]}\mathbb{I}\Big(\frac{1}{n}\sum_{i \in [n]}\|x_i^{t,r} - x^\star\|^2 > \epsilon\Big)$, where $x_i^{t,r} \in \R^d$ is the model of user $i$ in iteration $t$ and run $r$, with $x^\star = \argmin_{x \in \R^d}f(x)$ and $\mathbb{I}(A)$ being the indicator of event $A$. The empirical tail probability is a proxy to the true tail probability and is the main metric in our experiments. To further illustrate our results, we also compute the \emph{empirical mean-squared error} $\E^t_n$, which is the optimality gap at time $t$ averaged across all agents and runs, i.e., $\E^t_n = \frac{1}{nR}\sum_{i \in [n]}\sum_{r \in [R]}\|x_i^{t,r} - x^\star\|^2$. We next present the results.

\paragraph{Exponentially decaying tails.} To verify that the (empirical) tail probability decays at an exponential rate, we consider a fixed network of $n = 10$ agents, communicating over a ring graph. For local costs, each matrix $A_i$ is generated using python's \texttt{sklearn} library function \texttt{make\_sparse\_spd\_matrix}, with dimension $d = 50$ and value $\texttt{alpha} = 0.9$, while vectors $b_i$ are drawn from a multivariate normal distribution $\mathcal{N}(\mathbf{0}_d,\sigma_i^2 I_d)$, where $\sigma_i^2 = i$ for $i \in \{1,\ldots,10\}$, ensuring that the data is heterogeneous across agents. We run both algoriths for $T = 3000$ iterations and repeat across $R = 500$ runs. To compute the empirical probability, we consider threshold values $\epsilon = \big\{0.01, 0.001\big\}$. The results are presented in Figure \ref{fig:synth-exp}, where the plots show the MSE and tail behaviour for different values of $\epsilon$, left to right. We can see that the empirical tail probability of \gtdsgd decays exponentially fast for all values of $\epsilon$, as predicted by our theory. Moreover, the tails of \gtdsgd decay faster than those of \dsgd, showing the benefit of bias-correction is maintained in the HP sense.   

\paragraph{Linear speed-up.} To verify that \gtdsgd achieves linear speed-up in the number of agents, we consider three networks, with $n = \{10, 25,50\}$ agents. In all three cases, agents communicate over Erd\H{o}s--R\'enyi graphs with Metropolis-Hastings weights that satisfy $\lambda = \|W - J\| \approx 0.9$. The local quadratic costs are generated as follows: the matrices $A_i$ are fixed across agents, i.e., $A_i \equiv A = S + I_d$, where $S \in \R^{d \times d}$ is generated using \texttt{sklearn}'s \texttt{make\_sparse\_spd\_matrix}, with $d = 50$ and $\texttt{alpha} = 0.9$, while $b_i$'s are given by $b_i = \beta_i\mathbf{1}_d$, where $\beta_i \in \{-2,-1,0,1,3\}$ are selected uniformly at random, in equal proportion across agents (i.e., one fifth of agents has $\beta_i = -2$, one fifth has $\beta_i = -1$, etc). Generating the network and cost in this manner ensures that network connectivity and data heterogeneity are constant across all three experiments, allowing us to properly capture the effect of linear speed-up. We run \gtdsgd for $T = 1500$ iterations and repeat across $R = 250$ runs. To compute the empirical probability, we again use thresholds $\epsilon = \big\{0.01, 0.001\big\}$. The results are presented in Figure \ref{fig:synth-speed}, where the plots left to right show the MSE and tail probability for different values of $\epsilon$. We can again see that the empirical tail probability decays exponentially fast for all values of $\epsilon$, with the decay being consistently faster for larger number of agents, demonstrating linear speed-up.

\begin{figure*}[t]
\centering
\begin{tabular}{ccc}
\includegraphics[scale=0.3]{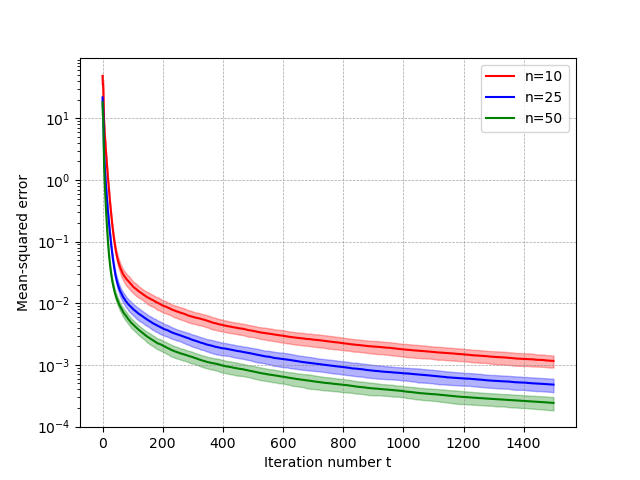}
&
\includegraphics[scale=0.3]{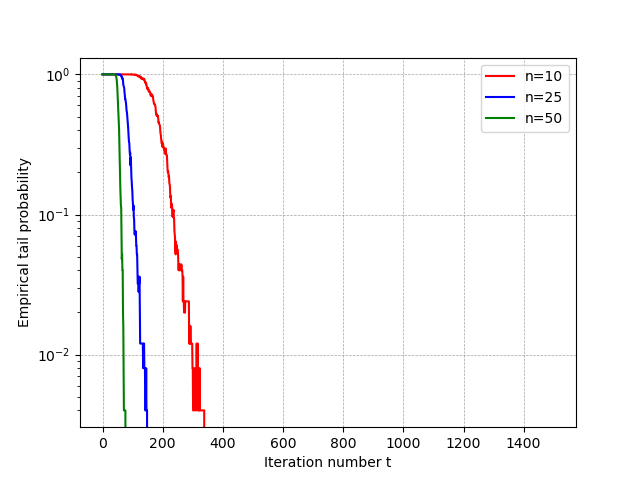}
&
\includegraphics[scale=0.3]{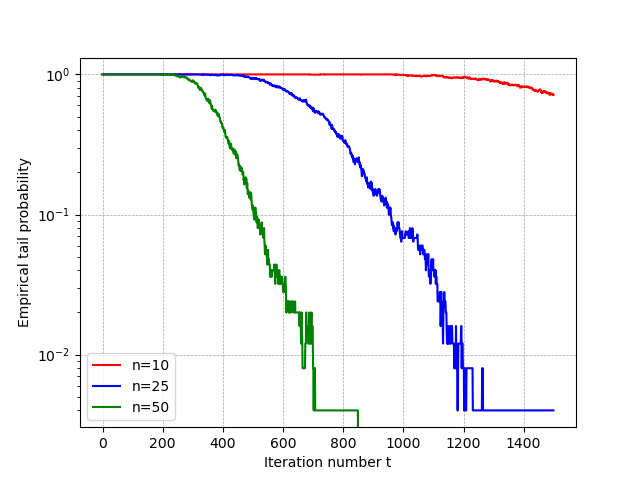}
\end{tabular}
\caption{Linear speed-up on synthetic data. Left to right: MSE performance and tail decay with threshold $\epsilon = \big\{0.01,0.001\big\}$. We can see that the tails of \gtdsgd decay faster for larger networks, across all values of threshold $\epsilon$, verifying that linear speed-up is achieved.}
\label{fig:synth-speed}
\end{figure*}

\subsection{Real Data}\label{subsec:num-real}

\paragraph{Methodology.} We next validate our theory on a non-convex problem with real data. In particular, we consider a binary classification task using logistic regression with a non-convex regularizer, e.g., \cite{antoniadis2011penalized}, with agents' local costs given by
\begin{equation*}
    f_i(x) = \frac{1}{m_i}\sum_{r \in [m_i]} \log\big(1+\exp(-y_{i,r}\langle h_{i,r}, x\rangle)\big) + \eta \sum_{k \in [d]} \frac{[x]_k^2}{1+[x]_k^2},     
\end{equation*} where $h_{i,r} \in \mathbb{R}^{d}$ and $y_{i,r} \in \{+1, - 1\}$ are respectively the feature vector and the associated label, $\eta > 0$ is a user-specified penalty parameter, while $[x]_k$ denotes the $k$-th component of the model vector $x$. To evaluate the performance, we use the ``mushroom'', ``a9a'' and ``ijcnn1'' datasets from the LIBSVM library \cite{chang2011libsvm}, each providing a varying degree of data heterogeneity. We split the data uniformly across agents, so that all agents have an equal-sized local dataset before training, i.e., $m_i = \frac{m}{n}$, where $m$ is the size of the original dataset. Similarly to synthetic data experiments, we use Erd\H{o}s-R\'enyi graphs with Metropolis-Hastings weights. For each experiment, we fix the following parameters: step-size $\alpha = 0.1$ and $\eta = 0.1$, giving both algorithms a large learning rate and a non-trivial effect of the non-convex regularizer. Since the MSE is not computable in the non-convex case, we evaluate the performance using the average gradient norm-squared, i.e., $G^{t,r}_n = \frac{1}{nt}\sum_{\tau \in [t]}\sum_{i \in [n]}\|\nabla f(x_i^{\tau,r})\|^2$, so that the empirical tail probability is computed as $\Prob^t_{n,\epsilon} = \frac{1}{R}\sum_{r \in [R]}\mathbb{I}\big(G^{t,r}_n > \epsilon\big)$. Similarly to the previous section, we also compute and visualize the average performance across all runs, i.e., $\E^t_n = \frac{1}{R}\sum_{r \in [R]}G^{t,r}_n$. We note that in our experiments on real data, the stochastic noise comes from the mini-batch choice, hence the resulting noise is \emph{not necessarily sub-Gaussian}. As such, our relaxed sub-Gaussian condition provides a much more appropriate noise model.

\begin{figure*}[t]
\centering
\begin{tabular}{ccc}
\includegraphics[scale=0.3]{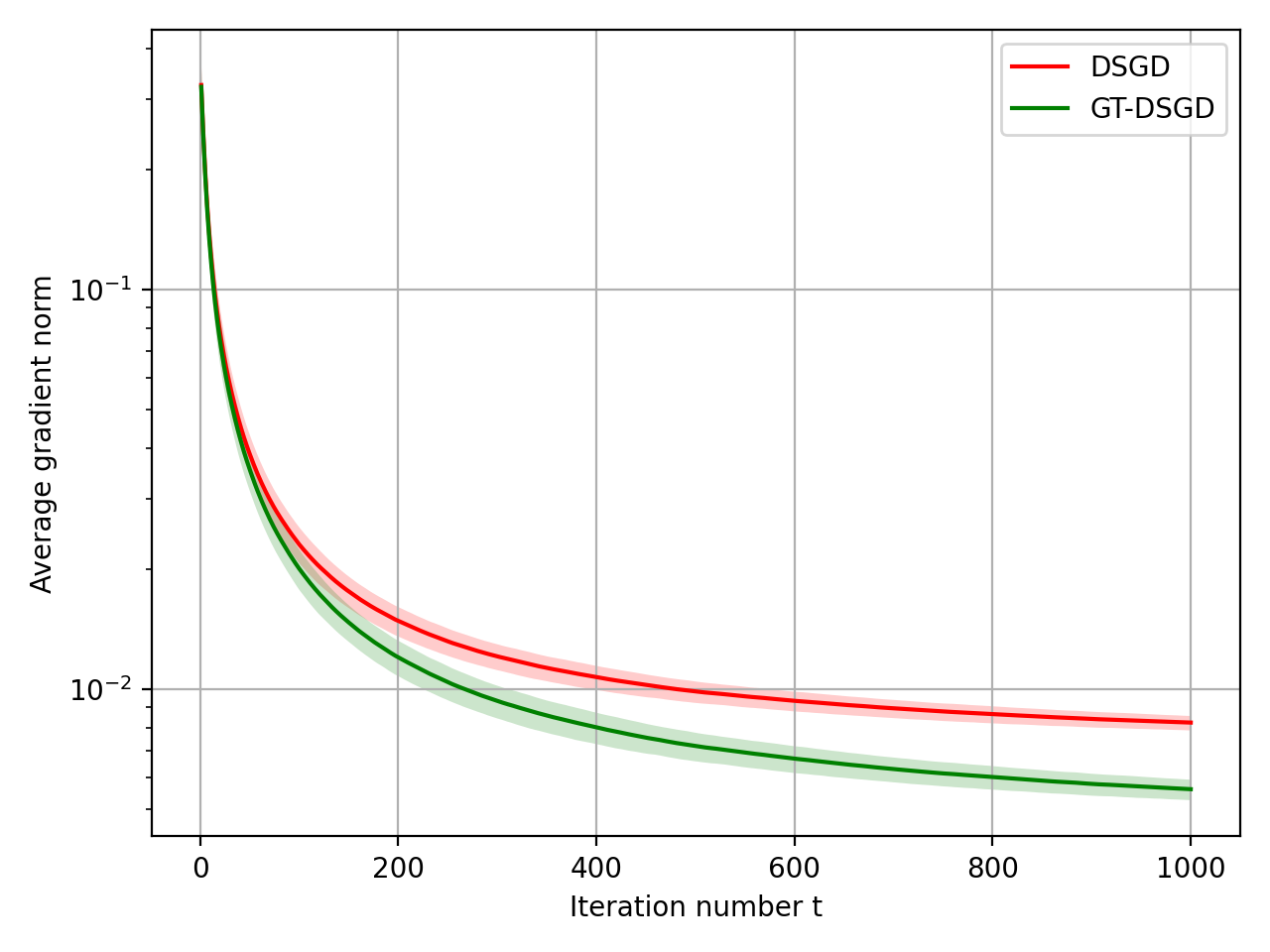}
&
\includegraphics[scale=0.3]{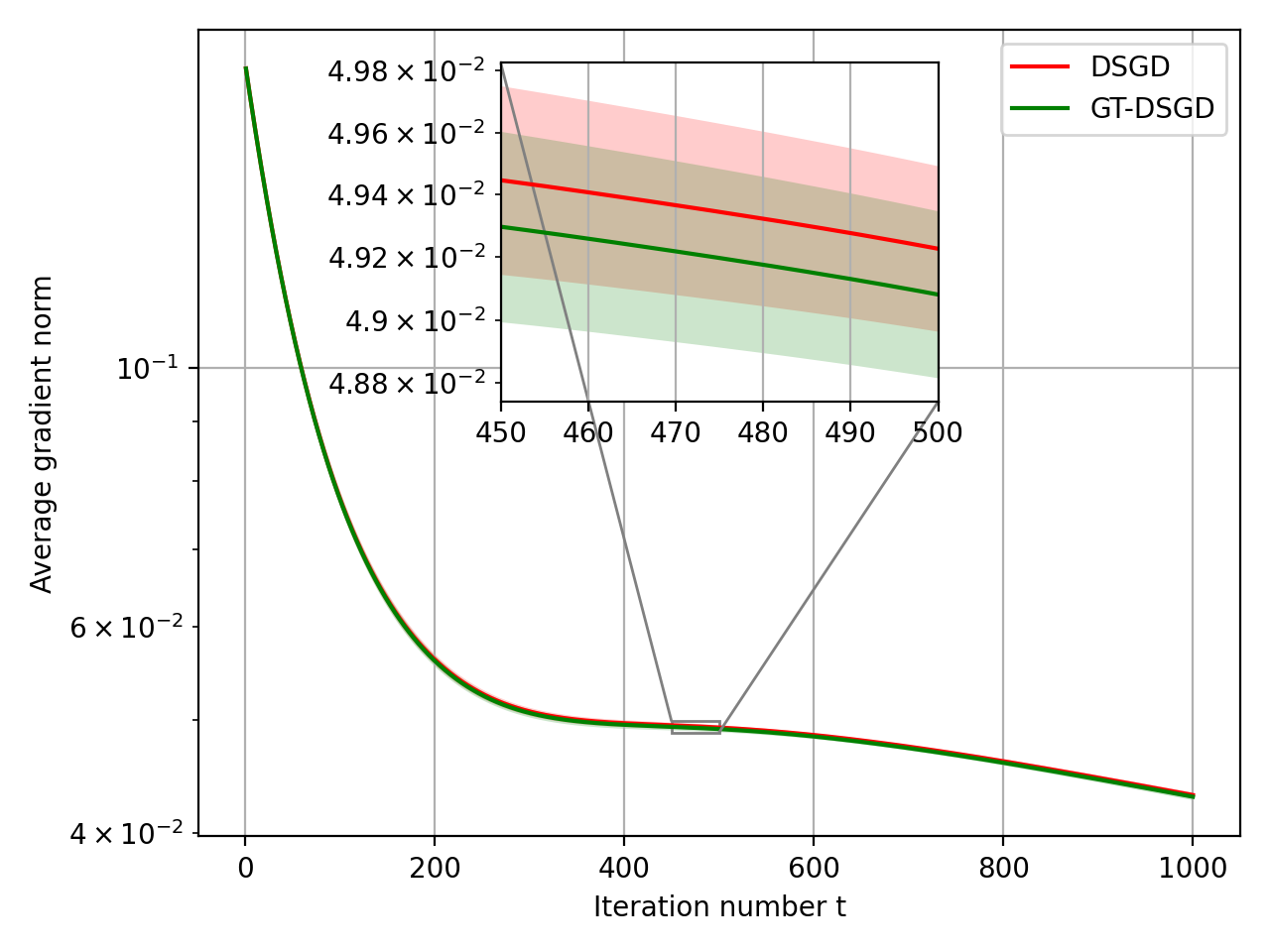}
&
\includegraphics[scale=0.3]{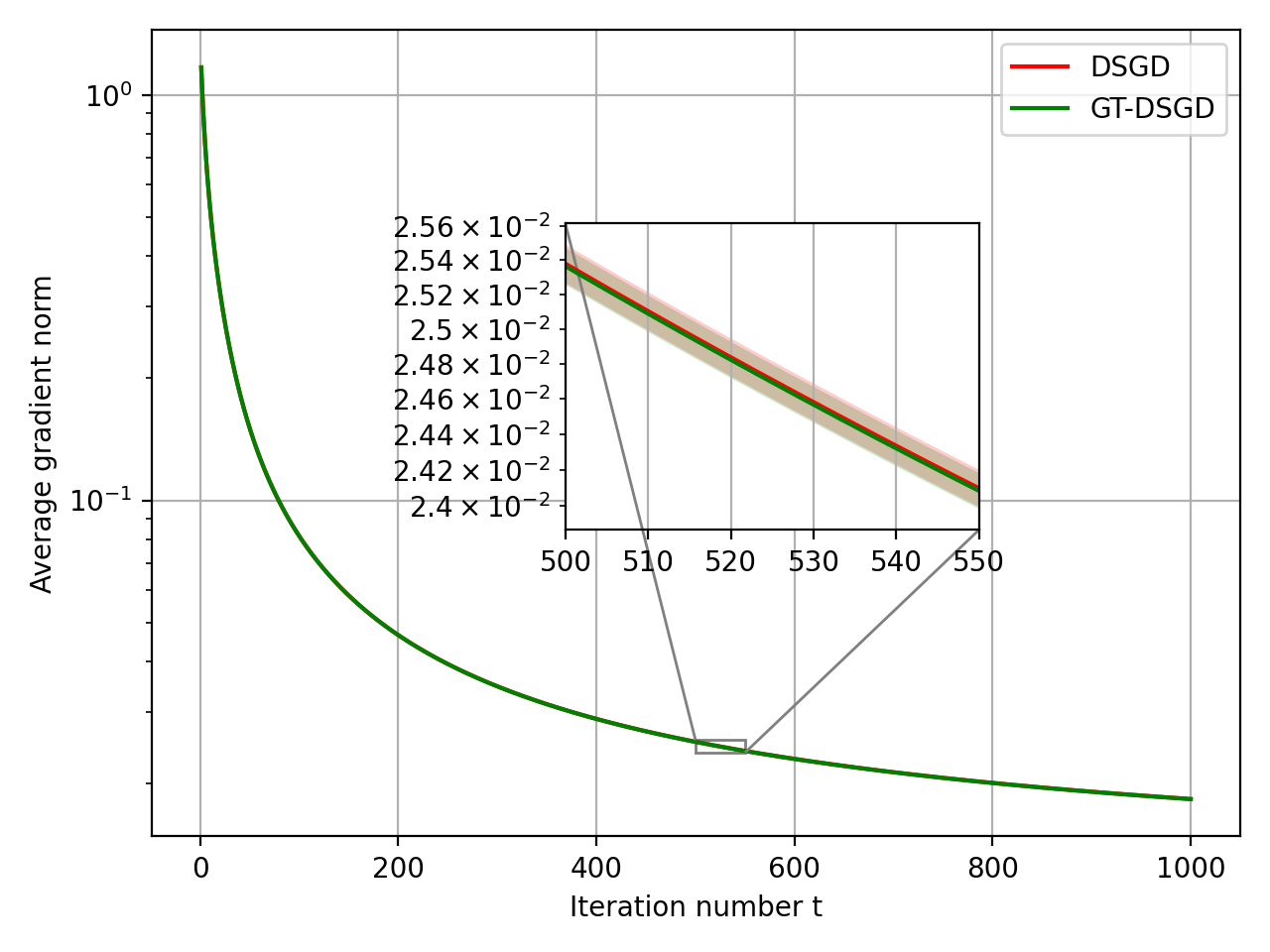}
\end{tabular}
\begin{tabular}{ccc}
\includegraphics[scale=0.3]{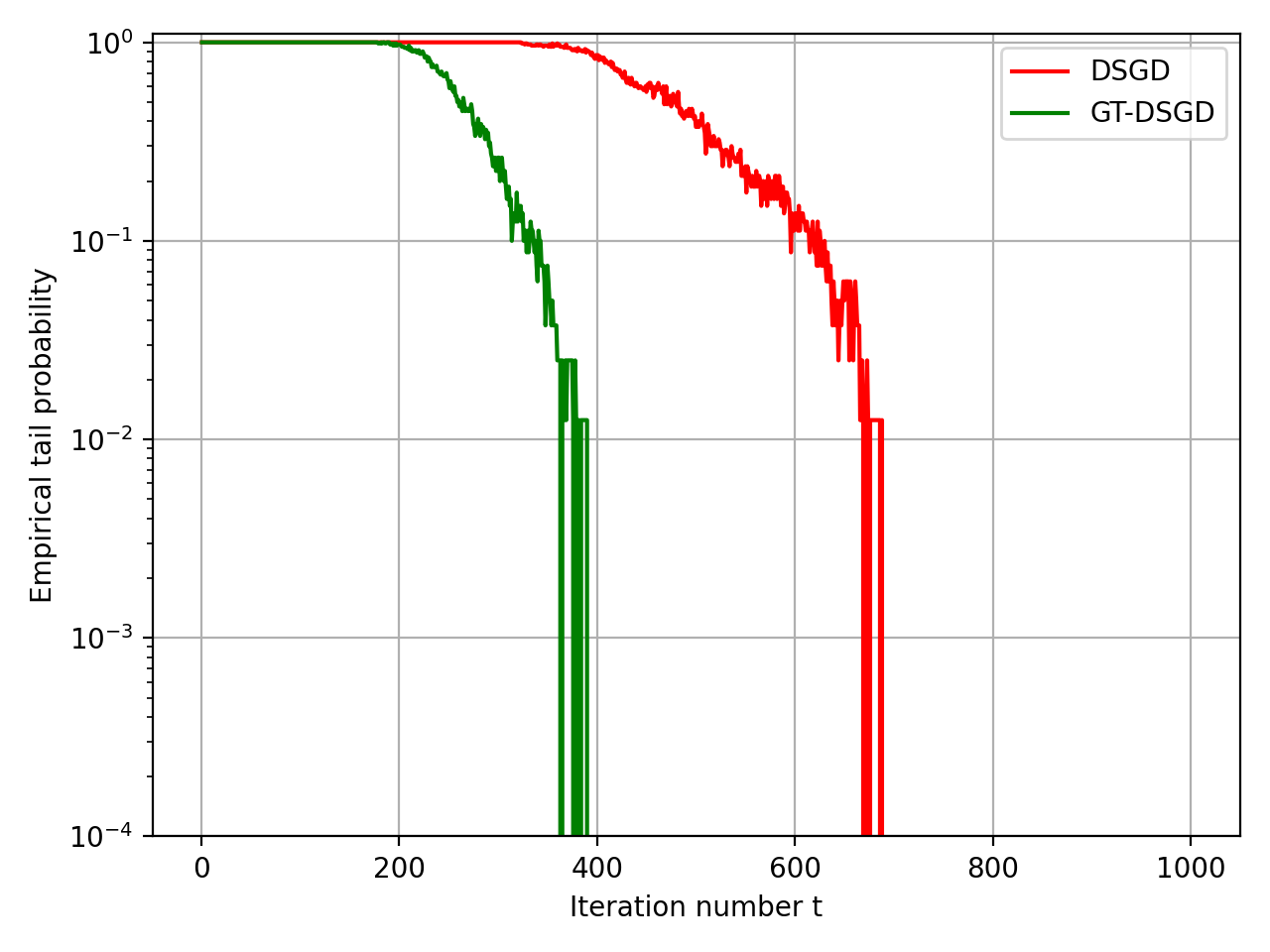}
&
\includegraphics[scale=0.3]{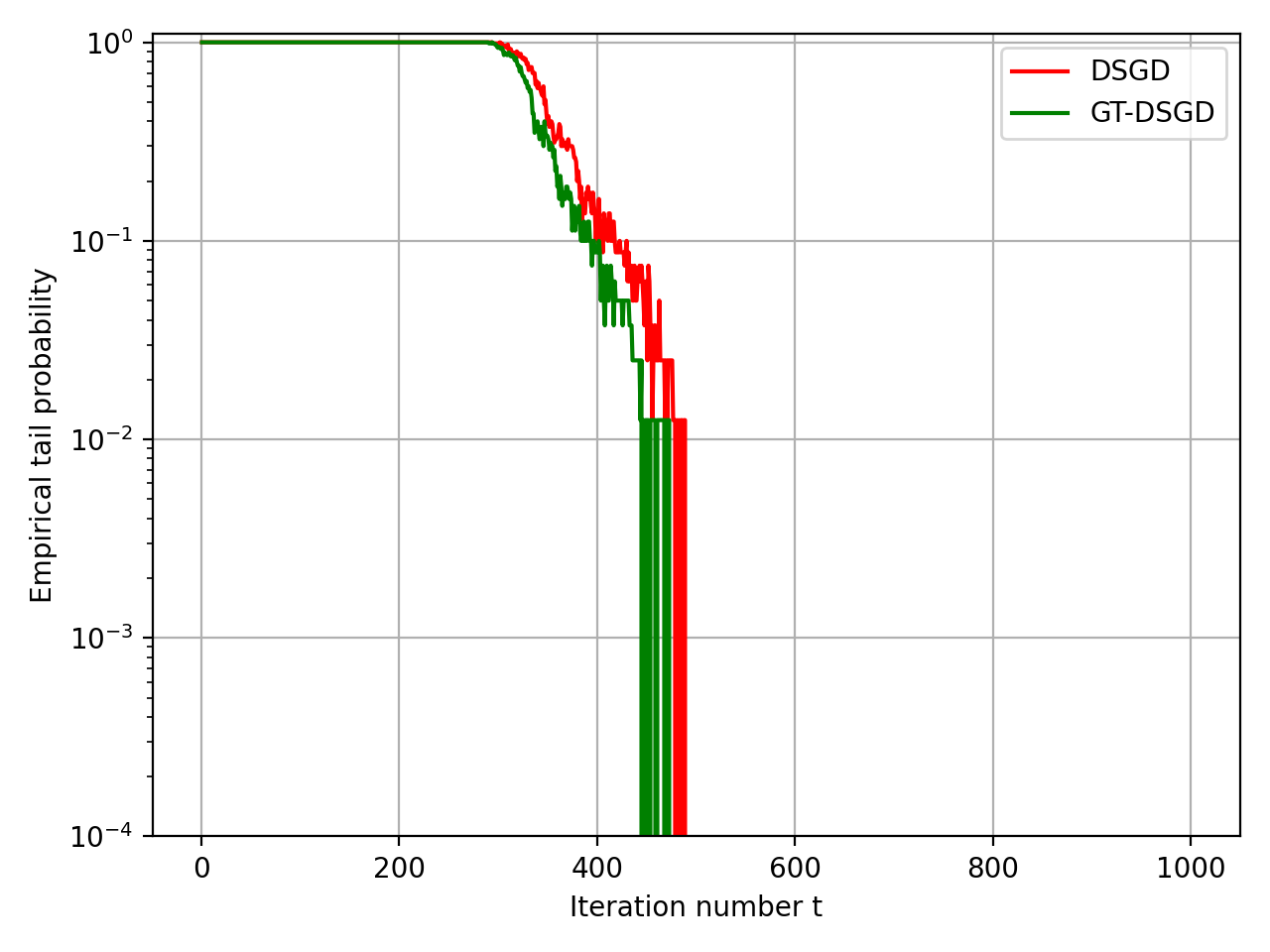}
&
\includegraphics[scale=0.3]{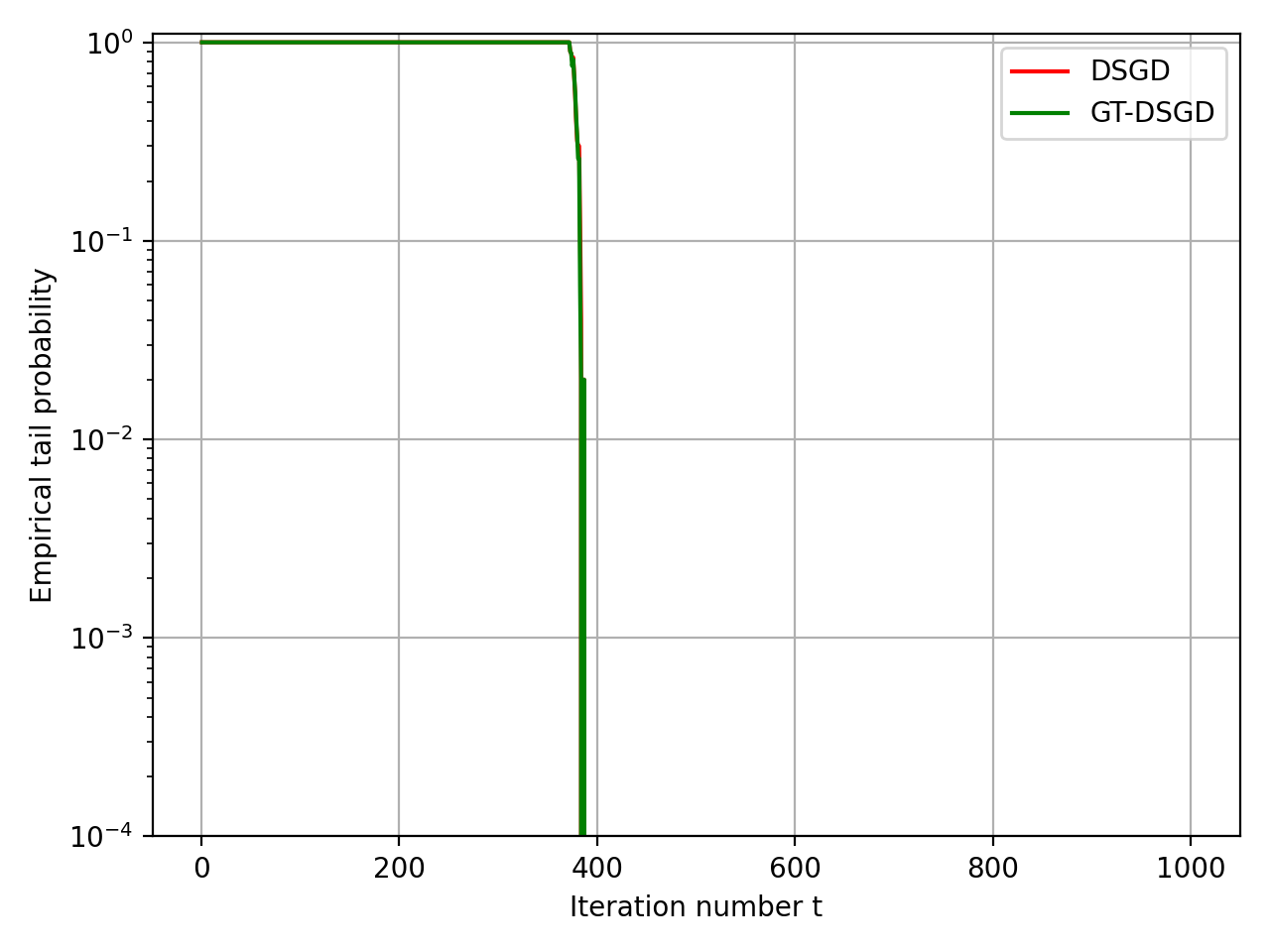}
\end{tabular}
\caption{Exponential tail decay on real data. Top to bottom: average gradient norm across all runs and its empirical tail probability. Left to right: performance on datasets ``a9a'', ``ijcnn1'' and ``mushroom''. For the empirical tail probability, we use threshold values $\epsilon = \{0.01,0.05,0.01\}$ for the respective datasets. We can see that \gtdsgd performs on par or better than \dsgd, with exponentially decaying tails across all three datasets.}
\label{fig:real-exp}
\end{figure*}

\paragraph{Exponentially decaying tails.} We start by testing the decay rate of the empirical tail probability, comparing the performance of \gtdsgd and \dsgd. We fix a network of $n = 30$ agents communicating over an Erd\H{o}s-R\'enyi graph. We run both methods for $T = 1000$ iterations, repeated across $R=100$ runs. The results are presented in Figure \ref{fig:real-exp}, where the top row visualizes the average performance across all runs $\E^t_n$, the bottom row visualizes the empirical tail probability $\Prob^t_{n,\epsilon}$, while figures left to right correspond to datasets ``a9a'', ``ijcnn1'' and ``mushroom'', respectively. The threshold values are chosen as $\epsilon = \{0.01, 0.05, 0.01 \}$, for the respective datasets, based on the achieved average values in the plots above. As can be seen from the figure, \gtdsgd clearly outperforms \dsgd on the ``a9a'' dataset, while performing similarly to \dsgd on ``ijcnn1'' and ``mushroom'' datasets, exhibiting exponential tail decay across all three datasets. The smaller gap between \gtdsgd and \dsgd on ``ijcnn1'' and ``mushroom'' datasets stems from the fact that these datasets exhibit a much smaller degree of heterogeneity, whilst \gtdsgd is known to be more robust to data heterogeneity. 

\paragraph{Linear speedup.}  To verify that \gtdsgd achieves linear speed-up, we again consider three networks with $n = \{10, 30, 50\}$ agents, communicating over Erd\H{o}s--R\'enyi graphs. To ensure that the network connectivity is consistent, we enforce the condition $\lambda = \|W-J\| \approx 0.6$ on the resulting weight matrices, with $\alpha, T, R$ remaining unchanged. The results are presented in Figure \ref{fig:real-speed}, with the outline of figures the same as in the previous set of experiments. We can again see that the average gradient norm and empirical tail probability decay faster with larger $n$, with the speed-up effect more evident on the more heterogeneous ``a9a'' dataset.

\begin{figure*}[t]
\centering
\begin{tabular}{ccc}
\includegraphics[scale=0.3]{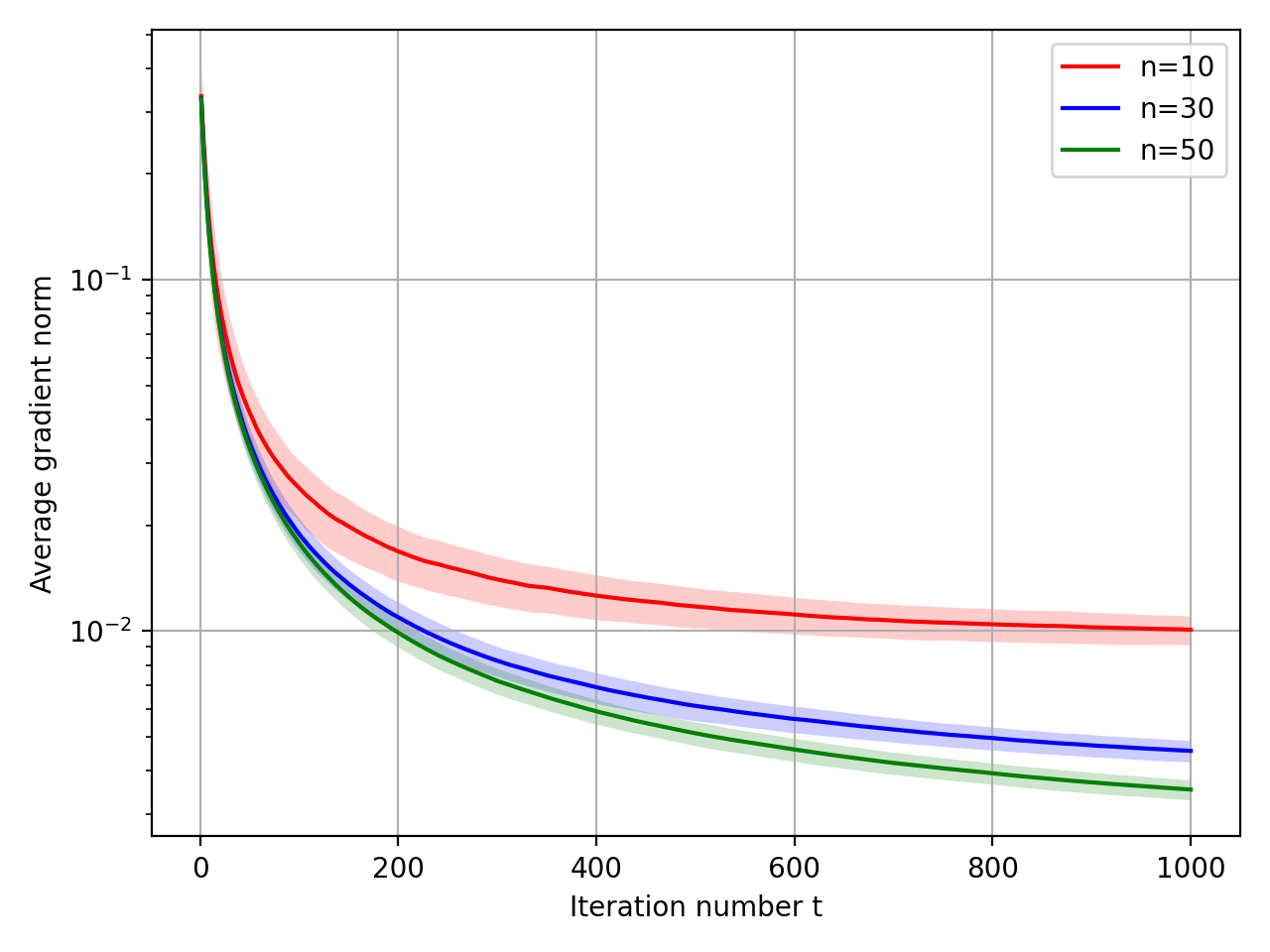}
&
\includegraphics[scale=0.3]{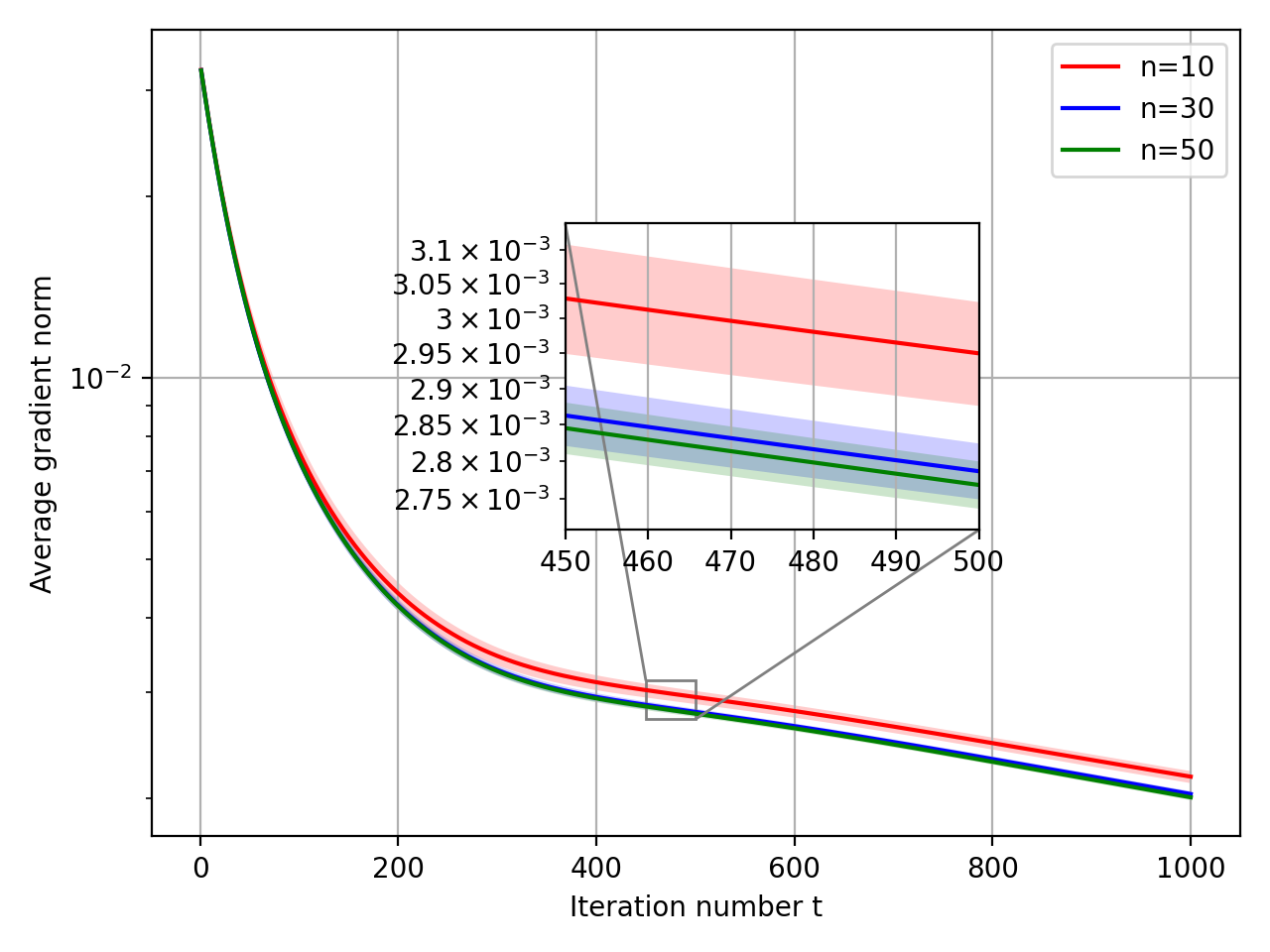}
&
\includegraphics[scale=0.3]{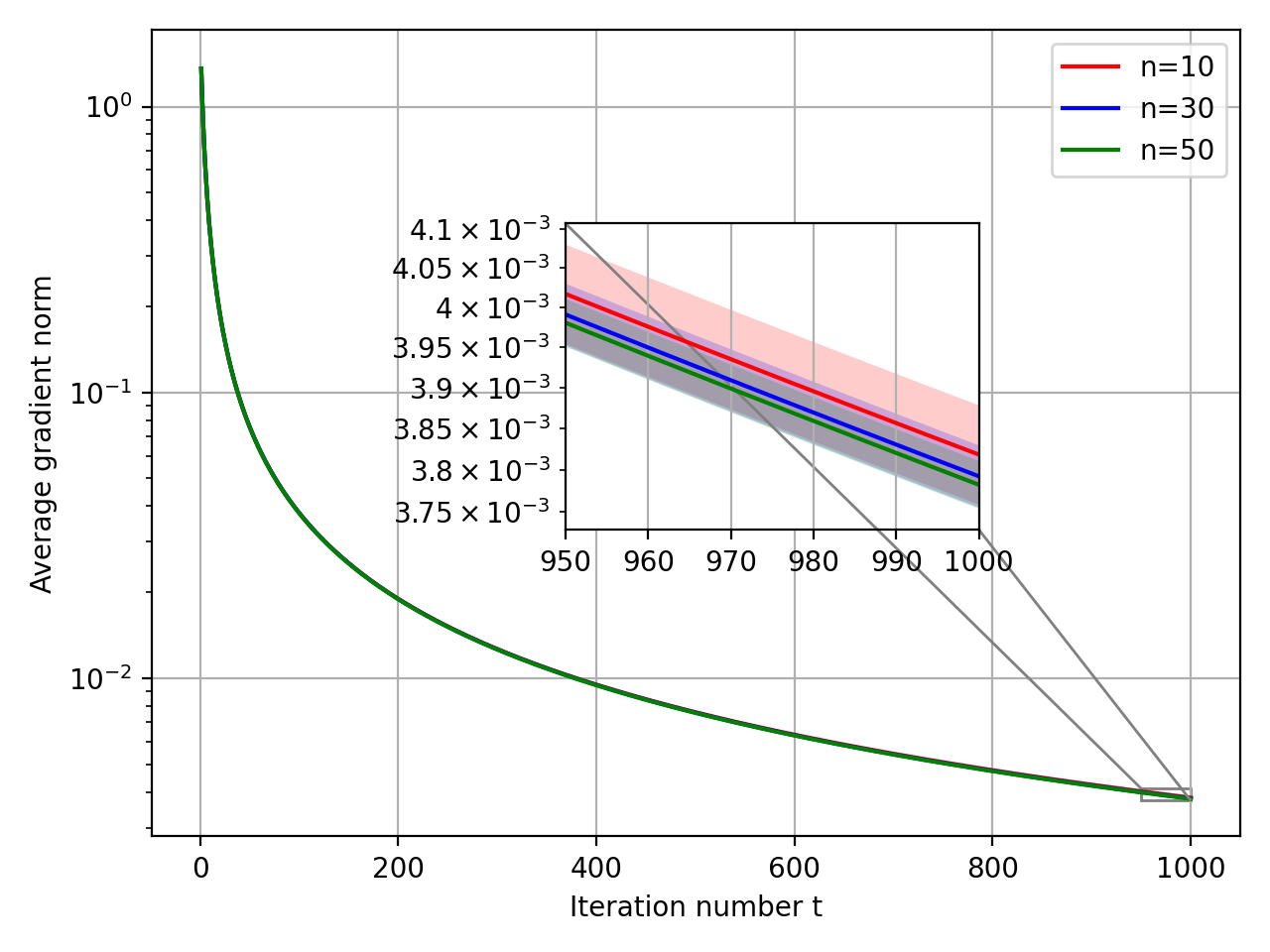}
\end{tabular}
\begin{tabular}{ccc}
\includegraphics[scale=0.3]{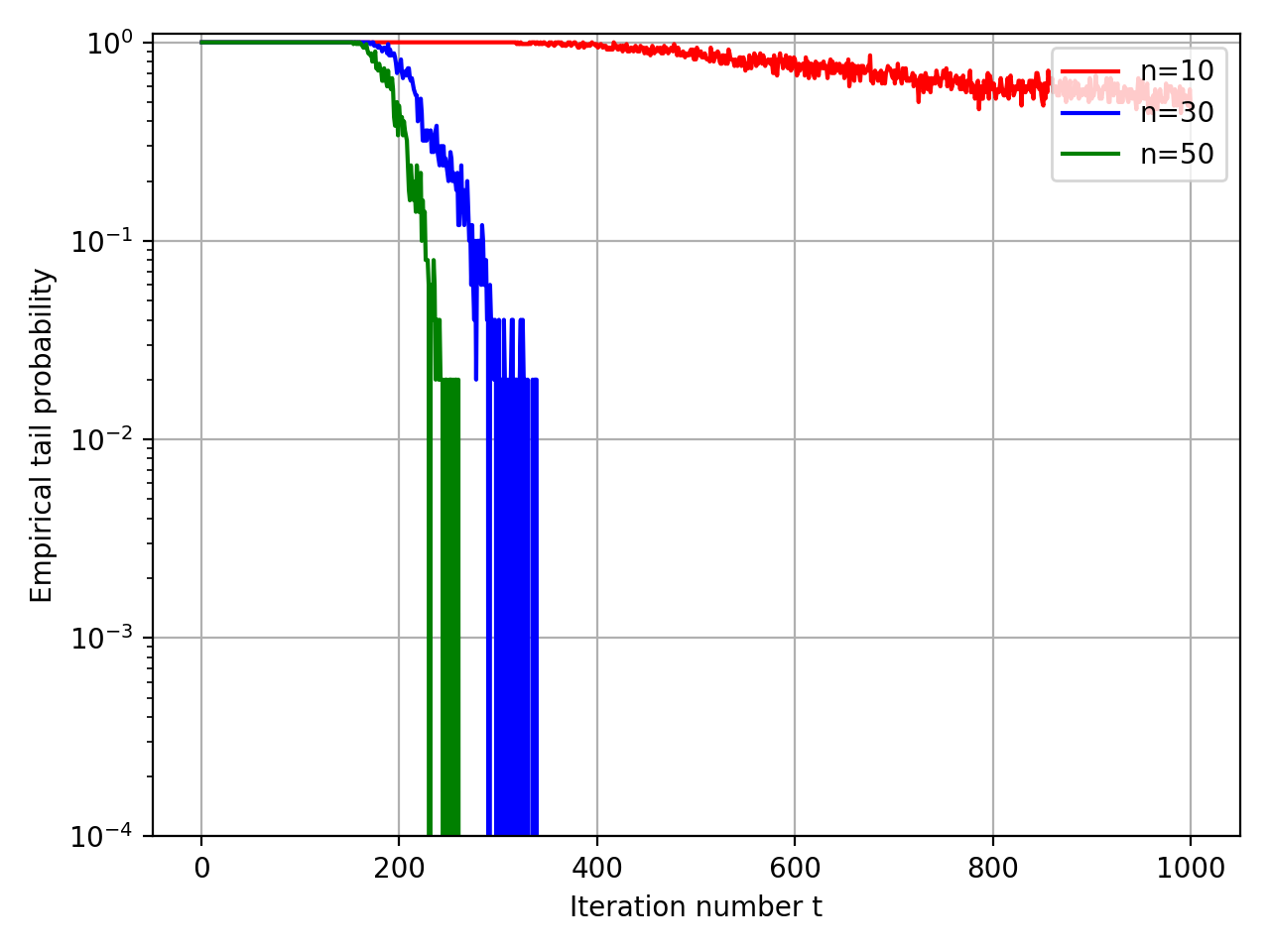}
&
\includegraphics[scale=0.3]{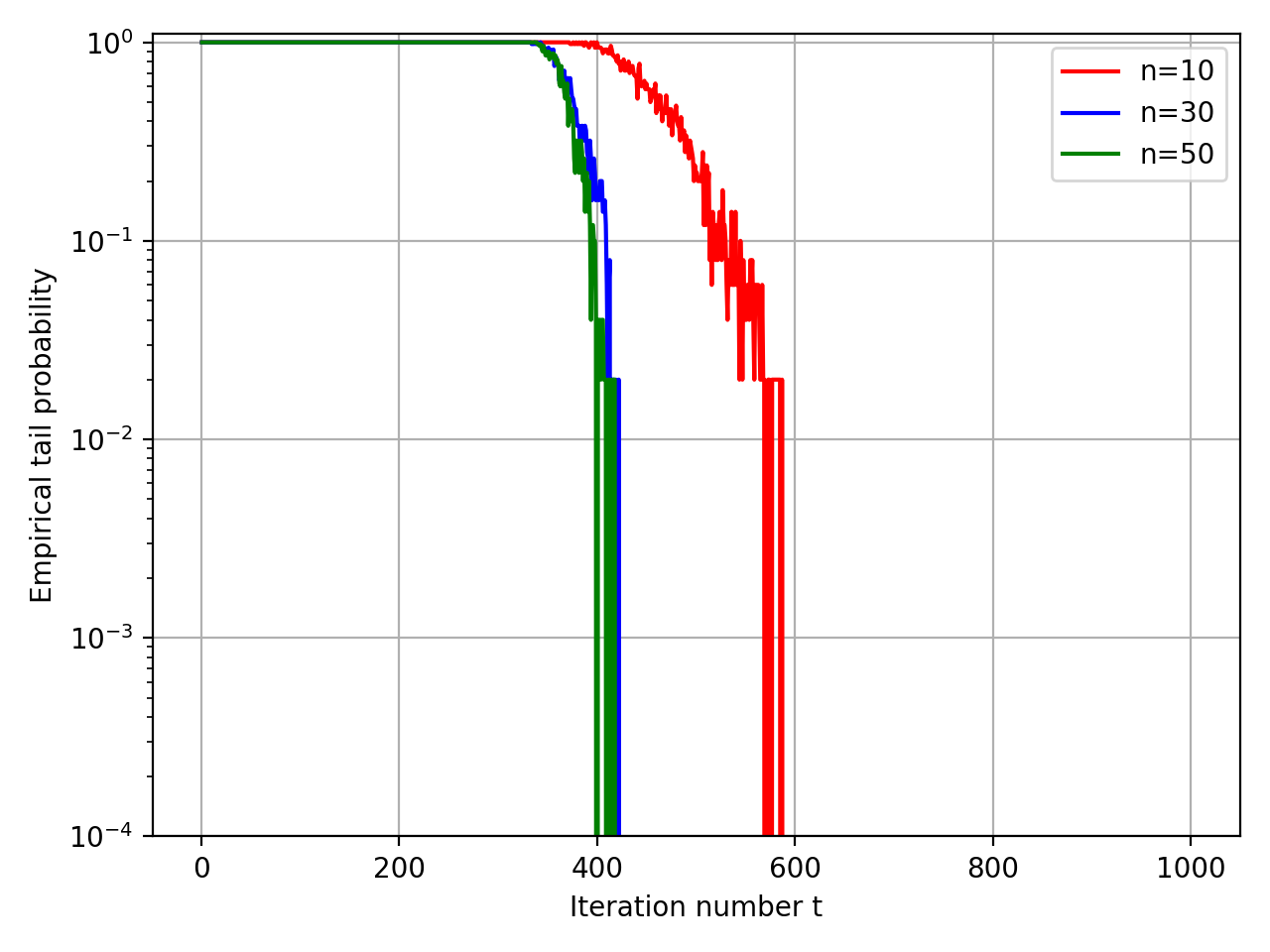}
&
\includegraphics[scale=0.3]{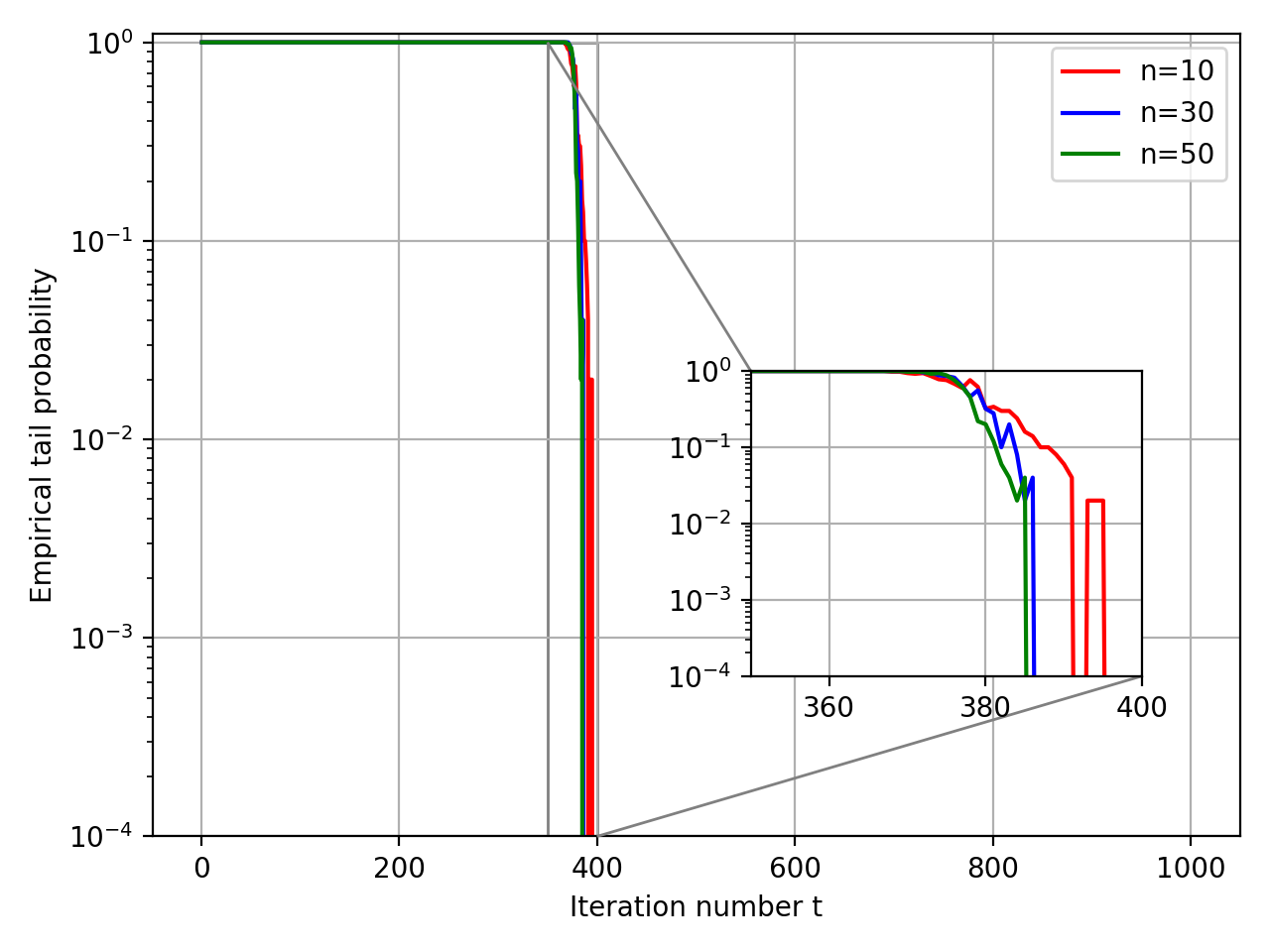}
\end{tabular}
\caption{Linear speed-up on real data. Top to bottom: average gradient norm across all runs and its empirical tail probability. Left to right: performance on datasets ``a9a'', ``ijcnn1'' and ``mushroom''. For the empirical tail probability, we use threshold values $\epsilon = \{0.01,0.003,0.01\}$ for the respective datasets. We can see that the tails of \gtdsgd consistently decay faster for larger networks, verifying the effect of linear speed-up.}
\label{fig:real-speed}
\end{figure*}

\section{Conclusion}\label{sec:conc}

In this work we provided the first HP convergence guarantees of decentralized methods incorporating bias-correction, by studying \gtdsgd in the presence of noise satisfying a relaxed sub-Gaussian condition. We show that \gtdsgd converges at order-optimal rates for both non-convex and P\L{} costs, while relaxing restrictive assumptions like bounded (gradient) heterogeneity, required by vanilla \dsgd. Moreover, our rates achieve linear speed-up in the number of agents, with transient times comparable to the ones stemming from MSE results. Numerical experiments on both real and synthetic data confirm our findings, showing that \gtdsgd exhibits exponentially decaying tails and linear speed-up in the HP sense. Future work includes studying HP guarantees in decentralized settings under heavy-tailed noise and the use of nonlinearities like clipping and normalization, e.g., \cite{nguyen2023improved,hubler2025normalization,armacki2025optimal,yu2026decentralized}, as well as establishing HP guarantees via a unified bias-correction framework, e.g., \cite{unified-refined}, which subsumes both ED and GT, while potentially yielding sharper transient times.  

\bibliographystyle{ieeetr}
\bibliography{bibliography}

\appendix

\section{Introduction}

The Appendix contains results omitted from the main body. Appendix \ref{sup:useful-results} collects some important facts used in our proofs, Appendix \ref{sup:technical} provides useful technical results, Appendix \ref{sup:analysis-setup} defines some notions used in the analysis, Appendix \ref{sup:non-conv} and \ref{sup:PL} contain proofs for non-convex and P\L{} costs, while Appendix \ref{sup:trans-time} provides further details on deriving transient times.

\section{Useful Inequalities}\label{sup:useful-results}

In this section we outline some well-known inequalities and results used in our proofs.

\begin{proposition}[Jensen's inequality]\label{prop:Jensen}
    Let $X \in \R$ be an integrable random variable. Then, for any convex function $h: \R \mapsto \R$, we have $h(\E[X]) \leq \E[h(X)]$. Moreover, if $h$ is concave, the reverse inequality holds, i.e., we then have $\E[h(X)] \leq h(\E[X])$.
\end{proposition}

\begin{proposition}[Cauchy-Schwartz inequality]\label{prop:Cauchy-Schwartz}
    For any $a,b \in \R^d$, we have $|\langle a,b \rangle| \leq \|a\|\|b\|$.
\end{proposition}

As a consequence of the Cauchy-Schwartz inequality, we have the following result.

\begin{proposition}[Young's inequality]\label{prop:Young}
    For any $a,b \in \R$ and any $\epsilon > 0$, we have $ab \leq \frac{\epsilon a^2}{2} + \frac{b^2}{2\epsilon}$. As a consequence, for any $\theta > 0$, we have $(a+b)^2 \leq (1+\theta)a^2 + (1+\theta^{-1})b^2$.
\end{proposition}

\begin{proposition}[H\"{o}lder's inequality]\label{prop:Holder}
    For any random variables $X,Y \in \R$ and any $p,q \in [1,\infty]$ such that $\frac{1}{p} + \frac{1}{q} = 1$, we have $\E|XY| \leq \sqrt[p]{\E|X|^p}\sqrt[q]{\E|Y|^q}$.
\end{proposition}

We also have the following important consequence of H\"{o}lder's inequality.

\begin{proposition}\label{prop:gen-Holder}
    For any random variables $\{X_i\}_{i \in [k]}$, we have $\E\Big[\prod_{i \in [k]}|X_i|\Big] \leq \prod_{i \in [k]}\sqrt[k]{\E|X_i|^k}$.
\end{proposition}

Next, we state a useful result from \cite{ran-improved}.

\begin{proposition}\label{prop:ran}
    For any $c,t_0 > 0$ and $0 \leq a \leq b$, we have $\prod_{k = a}^b\Big(1 - \frac{c}{k+t_0} \Big) \leq \frac{(a+t_0)^c}{(b+1+t_0)^c}$.
\end{proposition}

Finally, we state some important consequences of Assumption \ref{asmpt:cost}, see, e.g., \cite{bertsekas-gradient,nesterov-lectures_on_cvxopt}.

\begin{proposition}\label{prop:descent-ineq}
    Under Assumption \ref{asmpt:cost}, for any $i \in [n]$ and $x,y \in \R^d$, the following hold.
    \begin{enumerate}
         \item $f_i(x) \leq f_i(y) + \langle \nabla f_i(y), x-y\rangle + \frac{L}{2}\|x-y\|^2.$

         \item $f$ has $L$-Lipschitz continuous gradients.

         \item $f(x) \leq f(y) + \langle \nabla f(y), x-y\rangle + \frac{L}{2}\|x-y\|^2.$
    
        \item $\|\nabla f(x)\|^2 \leq 2L(f(x)-\fs)$.
    \end{enumerate}
\end{proposition}

\section{Technical Results}\label{sup:technical}

In this section we prove Lemmas \ref{lm:noise-properties} and \ref{lm:avg-noise-properties} and provide two important technical result, namely Lemmas \ref{lm:bdd-integral} and \ref{lm:mgf-bound-str-cvx}. We start by proving Lemma \ref{lm:noise-properties}.

\begin{proof}[Proof of Lemma \ref{lm:noise-properties}]
    \begin{enumerate}[leftmargin=*]
        \item Let $v \in \R^d$ be $\Fx$-measurable and define $Y = \frac{Z}{\sigma}$. We follow a similar idea to the one in \cite[Lemma 1]{li2020high} and consider two cases. First, if $\|v\| \leq \frac{4}{3}$, we then have
        \begin{align}\label{eq:mgf-bound-1}
            \E_X\lbr\exp\lp\langle v, Y \rangle\rp \rbr &\stackrel{(a)}{\leq} \E_X\lbr \langle v, Y\rangle + \exp\lp\frac{9}{16}\langle v, Y \rangle^2\rp \rbr \stackrel{(b)}{\leq} \E_X\lbr \exp\lp\frac{9}{16}\|v\|^2\| Y \|^2\rp \rbr \nonumber \\
            &\stackrel{(c)}{\leq} \bigg(\E_X\lbr \exp\lp\| Y \|^2\rp \rbr\bigg)^{9\|v\|^2/16} \stackrel{(d)}{\leq} \exp\bigg(1 + \rho\|\mathcal{T}(X)\| \bigg)^{9\|v\|^2/16} \nonumber \\
            &\stackrel{(e)}{\leq} \exp\bigg(\frac{3\|v\|^2}{4} + \rho\|\mathcal{T}(X)\| \bigg), 
        \end{align} where $(a)$ follows from the inequality $\exp(x) \leq x + \exp\Big(\frac{9x^2}{16}\Big)$, in $(b)$ we used Cauchy-Schwartz inequality and the fact that $Z$ is conditionally zero-mean, $(c)$ follows from the fact that $\frac{9\|v\|^2}{16} \leq 1$ and the reverse Jensen's inequality, in $(d)$ we used Definition \ref{def:relax-sub-Gauss}, while $(e)$ follows from the fact that $\frac{9\|v\|^2}{16} \leq \frac{3\|v\|^2}{4}$. Alternatively, if $\|v\| > \frac{4}{3}$, we have
        \begin{align}\label{eq:mgf-bound-2}
            &\E_X\lbr\exp\lp\langle v, Y \rangle\rp \rbr \leq \E_X\lbr\exp\lp \frac{3\|v\|^2}{8} + \frac{2\|Y\|^2}{3} \rp \rbr \nonumber \\
            &\leq \exp\bigg( \frac{3\|v\|^2}{8} + \frac{2}{3} + \rho\|\mathcal{T}(X)\| \bigg) \leq \exp\bigg(\frac{3\|v\|^2}{4} + \rho\|\mathcal{T}(X)\|\bigg),
        \end{align} where the first inequality follows from Young's inequality with $\epsilon = \frac{3}{4}$, the second follows from the reverse Jensen's inequality and Definition \ref{def:relax-sub-Gauss}, while the third follows since $\|v\| > \frac{4}{3}$ implies $\frac{2}{3} \leq \frac{3\|v\|^2}{8}$. Combining \eqref{eq:mgf-bound-1}-\eqref{eq:mgf-bound-2}, we get $\E_X[\exp(\langle v, Y\rangle)] \leq \exp\Big(\frac
        {3\|v\|^2}{4} + \rho\|\mathcal{T}(X)\|\Big)$, for any $\Fx$-measurable vector $v \in \R^d$. Noting that $\langle v,Z\rangle = \langle \sigma v,Y\rangle$, the claim follows.

        \item Fix an $\epsilon > 0$. Using the exponential Markov inequality, it follows that
        \begin{align*}
            \Prob_{X_i}\big(\|Z_i\| &> \epsilon \big) = \Prob_{X_i}\bigg(\frac{\|Z_i\|^2}{2\sigma_i^2} > \frac{\epsilon^2}{2\sigma_i^2}\bigg) \leq \exp\bigg(-\frac{\epsilon^2}{2\sigma_i^2} \bigg)\E_{X_i}\bigg[\exp\bigg(\frac{\|Z_i\|^2}{2\sigma_i^2}\bigg) \bigg] \\ 
            &\leq \exp\bigg(-\frac{\epsilon^2}{2\sigma_i^2} \bigg)\Bigg(\E_{X_i}\bigg[\exp\bigg(\frac{\|Z_i\|^2}{\sigma_i^2}\bigg) \bigg]\Bigg)^{1/2} \leq \exp\bigg(-\frac{\epsilon^2}{2\sigma_i^2} + \frac{1}{2} + \frac{\rho_i}{2}\|\mathcal{T}(X_i)\| \bigg) \\
            &\leq 2\exp\bigg(-\frac{\epsilon^2}{2\sigma_i^2} + \frac{\rho_i}{2}\|\mathcal{T}(X_i)\|\bigg),
        \end{align*} where the second inequality follows from Proposition \ref{prop:Jensen}, the third follows from Definition \ref{def:relax-sub-Gauss}, while in the last inequality we used .

        \item Using the layer cake representation for expectation of a non-negative random variable and defining $Y = \frac{Z_i}{\sigma_i\sqrt{2}}$, it follows that
        \begin{align*}
            \E_{X_i}\|Y\|^{2p} &= \int_{0}^\infty u\Prob_{X_i}(\|Y\|^{2p} > u)du \stackrel{(a)}{=} p\int_{0}^{\infty}t^{p-1}\Prob_{X_i}\big(\|Z_i\|^2>2\sigma_i^2t\big)dt \\ 
            &\stackrel{(b)}{\leq} 2p\exp\bigg(\frac{\rho_i}{2}\|\mathcal{T}(X_i)\|\bigg)\int_{0}^\infty t^{p-1}\exp\big(-t\big)dt \stackrel{(c)}{=} 2p\exp\bigg(\frac{\rho_i}{2}\|\mathcal{T}(X_i)\|\bigg)\Gamma(p) \\
            &\stackrel{(d)}{\leq} 2p^{p+1}\exp\bigg(\frac{\rho_i}{2}\|\mathcal{T}(X_i)\|\bigg),
        \end{align*} where $(a)$ follows from the substitution $u = t^p$, in $(b)$ we used the previously proven property, $(c)$ follows from the definition of the $\Gamma$ function, while in $(d)$ we used $\Gamma(a) \leq a^a$. Multiplying both sides by $(2\sigma_i^2)^{p}$ completes the proof.

        \item To prove the desired property, we start by following similar steps as in \cite[Lemma 4]{jin2019short} and define the matrix $Y = \begin{bmatrix}
            0 & Z_i^\top \\
            Z_i & 0
        \end{bmatrix} \in \R^{(d + 1) \times (d+1)}$ and note that: (i) $Y$ is a rank-two matrix with eigenvalues $\pm\|Z_i\|$; (ii) $\E[Y^{2p+1}] = 0$; (iii) $\|Y\|^{2p} \leq \|Z_i\|^{2p}$. Using these facts, it follows that, for any $\lambda > 0$
        \begin{align*}
            \E_{X_i}\Big[e^{\lambda Y} \Big] &= I + \sum_{p = 1}^\infty\frac{\lambda^{2p}\E_{X_i}[Y^{2p}]}{(2p)!} \preceq \Bigg(1 + \sum_{p = 1}^\infty\frac{\lambda^{2p}\E_{X_i}\|Z_i\|^{2p}}{(2p)!}\Bigg)I \\
            &\stackrel{(i)}{\preceq} 2\exp\bigg(\frac{\rho_i}{2}\|\mathcal{T}(X_i)\|\bigg)\Bigg(1 + \sum_{p = 1}^\infty\frac{p(2p\lambda^2\sigma_i^2)^p}{(2p)!}\Bigg)I \\
            &\stackrel{(ii)}{\preceq} 2\exp\bigg(\frac{\rho_i}{2}\|\mathcal{T}(X_i)\|\bigg)\Bigg(1 + \sum_{p = 1}^\infty\frac{(2\lambda^2\sigma_i^2)^p}{p!}\Bigg)I \\
            &= 2\exp\bigg(2\lambda^2\sigma_i^2 + \frac{\rho_i}{2}\|\mathcal{T}(X_i)\|\bigg)I,
        \end{align*} where $(i)$ follows from the previously proven property, while in $(ii)$ we use $(2p)! \geq p!p^{p+1}$. Next, following similar steps as in \cite[Lemma 6]{jin2019short}, replacing the full expectations and probabilities with the ones conditioned on $\mathcal{F}_X$ and using the fact that $Z_1,\ldots,Z_n$ are conditionally independent, it readily follows that, for any $\theta > 0$ and $\delta \in (0,1)$, we have
        \begin{equation*}
            \Prob_X\bigg(\|\sum_{i \in [n]}Z_i\| > 2\theta\sum_{i \in [n]}\sigma_i^2 + \frac{1}{2}\sum_{i \in [n]}\rho_i\|\mathcal{T}(X_i)\| + \frac{1}{\theta}\log(\nicefrac{2d}{\delta})\bigg) \leq \delta. 
        \end{equation*} Choosing $\theta = \sqrt{\frac{\log(\nicefrac{2d}{\delta})}{2\sum_{i \in [n]}\sigma_i^2}}$, we get that, for any $\delta \in (0,1)$
        \begin{equation*}
            \Prob_X\bigg(\|\sum_{i \in [n]}Z_i\| > 2\sqrt{2\log(\nicefrac{2d}{\delta})\sum_{i \in [n]}\sigma_i^2} + \frac{1}{2}\sum_{i \in [n]}\rho_i\|\mathcal{T}(X_i)\|\bigg) \leq \delta, 
        \end{equation*} or equivalently, for any $\epsilon > 0$
        \begin{align*}
            \Prob_X\Big(\|\sum_{i \in [n]}Z_i\| > \epsilon\Big) &\leq 2d\exp\bigg(-\Big(\epsilon - \frac{1}{2}\sum_{i \in [n]}\rho_i\|\mathcal{T}(X_i)\|\Big)^2/\Big(8\sum_{i \in [n]}\sigma_i^2\Big) \bigg). 
        \end{align*} Consider now the averaged quantity $\overline{Z} = \frac{1}{n}\sum_{i \in [n]}Z_i$ and denote by $\sigma^2 = \frac{1}{n}\sum_{i \in [n]}\sigma_i^2$ and $c = \frac{1}{2}\sum_{i \in [n]}\rho_i\|\mathcal{T}(X_i)\|$. From the above inequality, it then follows that
        \begin{equation}\label{eq:prob-bdd-avg}
            \Prob_X\big(\|\overline{Z}\| > \epsilon\big) = \Prob_X\Big(\|\sum_{i \in [n]}Z_i\| > n\epsilon\Big) \leq 2d\exp\bigg(-\frac{(n\epsilon - c)^2}{8n\sigma^2} \bigg).
        \end{equation} Using the layer cake representation again, we then get
        \begin{align*}
            \E_X\|\overline{Z}\|^{2p} &= \int_0^{\infty}\Prob_X\big(\|\overline{Z}\|^{2p} > u\big)du = p\int_0^\infty t^{p-1}\Prob_X\big(\|\overline{Z}\|^{2} > t\big)dt \\
            &\leq 2dp\int_0^\infty t^{p-1}\exp\bigg(-\frac{(n\sqrt{t} - c)^2}{8n\sigma^2} \bigg)dt \\
            &\stackrel{(i)}{\leq} 2dp\exp\bigg(\frac{c^2}{8n\sigma^2}\bigg)\bigg(\frac{16\sigma^2}{n}\bigg)^p\Gamma(p) \\
            &\stackrel{(ii)}{\leq} 2d\exp\bigg(\frac{c^2}{8n\sigma^2}\bigg)p^{p+1}\bigg(\frac{16\sigma^2}{n}\bigg)^p,
        \end{align*} where in $(i)$ we used Lemma \ref{lm:bdd-integral}, while $(ii)$ follows by again using $\Gamma(p) \leq p^p$. Finally, let $k > 0$ be a free parameter. Using the above bound and Taylor's expansion, we then get
        \begin{align*}
            \E_X\bigg[\exp\bigg(\frac{\|\overline{Z}\|^2}{k^2} \bigg) \bigg] = 1 + \sum_{p = 1}^{\infty}\frac{\E_X\|\overline{Z}\|^{2p}}{k^{2p}p!} \leq 2d\exp\bigg(\frac{c^2}{8n\sigma^2}\bigg)\Bigg(1 + \frac{1}{e}\sum_{p = 1}^{\infty}p\bigg(\frac{16e\sigma^2}{nk^2}\bigg)^p \Bigg),
        \end{align*} where we used $p! \geq e\big(\nicefrac{p}{e}\big)^p$ in the last inequality. Recalling that $\sum_{p = 1}^\infty pa^p = \frac{a}{(1-a)^2}$ for any $|a| < 1$, it can then be verified that choosing $k = \frac{4\sigma\sqrt{6}}{\sqrt{n}}$, it follows that $\frac{1}{e}\sum_{p = 1}^\infty p\Big(\frac{16e\sigma^2}{nk^2}\Big)^p \leq e- 1$, therefore we get
        \begin{equation*}
            \E_X\bigg[\exp\bigg(\frac{n\|\overline{Z}\|^2}{96\sigma^2}\bigg)\bigg] \leq 2d\exp\Bigg(1 + \frac{\sum_{i \in [n]}\rho_i^2\|\mathcal{T}(X_i)\|^2}{32\sigma^2} \Bigg).
        \end{equation*}
    \end{enumerate}
\end{proof}

Next, we prove Lemma \ref{lm:avg-noise-properties}. For ease of notation, let $\E_t[\cdot] = \E[\cdot\vert\Ft]$ and $\Prob_t(\cdot) = \Prob(\cdot\vert\Ft)$.

\begin{proof}[Proof of Lemma \ref{lm:avg-noise-properties}]
    Note that the noise vectors $\{\zit\}_{i \in [n]}$ are independent, zero-mean and $(\sigma_i,\rho)$-sub-Gaussian conditioned on $\Ft$, with $\mathcal{T}(\xit) = \nabla f(\xit)$. We then proceed as follows.
    \begin{enumerate}
     \item For any $\Ft$-measurable vector $v \in \R^d$, we have
    \begin{align*}
        \E_t[\exp(\langle v, \ozt)] &\stackrel{(i)}{=} \prod_{i \in [n]}\E_t\bigg[\exp\bigg( \frac{1}{n}\langle v, \zit\rangle\bigg)\bigg] \stackrel{(ii)}{\leq} \prod_{i \in [n]}\exp\bigg( \frac{3\sigma_i^2\|v\|^2}{4n^2} + \alpha_t^{2+\varepsilon}\rho\|\nabla f(\xit)\|\bigg) \\
        &= \exp\bigg(\frac{3\sigma^2\|v\|^2}{4n} + \alpha_t^{2+\varepsilon}\rho\sum_{i \in [n]}\|\nabla f(\xit)\|\bigg) \\
        &\stackrel{(iii)}{\leq} \exp\bigg(\frac{3\sigma^2\|v\|^2}{4n} + \alpha_t^{2+\varepsilon}n\rho + \alpha_t^{2+\varepsilon}\rho\sum_{i \in [n]}\|\nabla f(\xit)\|^2\bigg),
    \end{align*} where $(i)$ follows from conditional independence of $\zit$'s, $(ii)$ follows from Lemma \ref{lm:noise-properties}, while in $(iii)$ we applied the inequality $x \leq 1 + x^2$ to each $\|\nabla f(\xit)\|$.

    \item Follows directly from the fourth property of Lemma \ref{lm:noise-properties}, with $\rho_i = \alpha_t^{2+\varepsilon}\rho$, for all $i \in [n]$.
    
    \end{enumerate}
\end{proof}

The next result provides the bound on the integral used in the proof of Lemma \ref{lm:noise-properties}.

\begin{lemma}\label{lm:bdd-integral}
    For any $a,b,c > 0$ and any $p \in \N$, it holds that 
    \begin{equation*}
        \int_0^\infty t^{p-1}\exp\Bigg(-\frac{(a\sqrt{t}-b)^2}{c}\Bigg)dt \leq \exp\bigg(\frac{b^2}{c}\bigg)\bigg(\frac{2c}{a^2}\bigg)^p\Gamma(p).
    \end{equation*}
\end{lemma}

\begin{proof}
    Expanding the expression in the exponent, we get
    \begin{equation*}
        \exp\Bigg(-\frac{(a\sqrt{t}-b)^2}{c}\Bigg) = \exp\bigg(-\frac{a^2t}{c} + \frac{2ab\sqrt{t}}{c} - \frac{b^2}{c}\bigg) \leq \exp\bigg(-\frac{a^2t}{2c} + \frac{b^2}{c} \bigg),
    \end{equation*} where we used Proposition \ref{prop:Young} with $\epsilon = 2$ in the last inequality. Plugging the above inequality into the integral, we then get 
    \begin{align*}
        \int_0^\infty t^{p-1}\exp\Bigg(-\frac{(a\sqrt{t}-b)^2}{c}\Bigg)dt &\leq \exp\bigg(\frac{b^2}{c}\bigg)\int_0^\infty t^{p-1}\exp\bigg(-\frac{a^2t}{2c}\bigg)dt \\
        &\stackrel{(i)}{=} \exp\bigg(\frac{b^2}{c}\bigg)\bigg(\frac{2c}{a^2}\bigg)^p\int_0^\infty s^{p-1}\exp(-s)ds \\
        &\stackrel{(ii)}{=} \exp\bigg(\frac{b^2}{c}\bigg)\bigg(\frac{2c}{a^2}\bigg)^p\Gamma(p), 
    \end{align*} where in $(i)$ we used the substitution $s = \frac{a^2t}{2c}$, while $(ii)$ follows from the definition of the $\Gamma$ function.
    
\end{proof}

Finally, we provide an important bound on the MGF, used in the HP analysis for P\L{} costs. 

\begin{lemma}\label{lm:mgf-bound-str-cvx}
    Let $\{\Xt\}_{t \geq 2}$ be a sequence of random variables initialized by a deterministic value $X^1 > 0$, such that, for some constants $a,t_0,C_1,C_2 > 0$ and every $t \geq 1$ 
    \begin{equation}\label{eq:iter-ineq}
        \E[\exp(\Xtp)] \leq \E\bigg[\exp\bigg(\Big(1-\frac{a}{t+t_0}\Big)\Xt + \frac{C_1}{t+t_0} + \frac{C_2}{(t+t_0)^2} \bigg)\bigg].
    \end{equation} If $a > 1$ and $t_0 \geq a$, we then have
    \begin{align}\label{eq:resulting-MGF-bound}
        \E[\exp(\Xtp)] &\leq \exp\bigg(\frac{(t_0+1)^aX_1}{(t+1+t_0)^a} + \frac{C_1}{a} + \frac{3C_2/(a-1)}{t+1+t_0} \bigg).
    \end{align} 
\end{lemma}

\begin{proof}
    For ease of notation, let $b_t \coloneqq 1 - \frac{a}{t + t_0}$ and note that $b_t \in (0,1)$ for every $t \geq 1$, since $t_0 \geq a$. Starting from \eqref{eq:iter-ineq} and applying Proposition \ref{prop:Jensen}, we have
    \begin{align*}
        \E[&\exp(\Xtp)] \leq \exp\bigg(\frac{C_1}{t+t_0} + \frac{C_2}{(t+t_0)^2} \bigg)\bigg(\E\big[\exp\big(X^t\big)\big]\bigg)^{b_t} \leq \exp\bigg(\frac{C_1}{t+t_0} + \frac{C_2}{(t+t_0)^2} \bigg) \\
        &\times\Bigg(\E\bigg[\exp\bigg(b_{t-1}X^{t-1} + \frac{C_1}{t-1+t_0} + \frac{C_2}{(t-1+t_0)^2}\bigg)\bigg]\Bigg)^{b_t} \\
        &\leq \exp\bigg(C_1\bigg(\frac{1}{t+t_0} + \frac{b_t}{t-1+t_0} \bigg) + C_2\bigg(\frac{1}{(t+t_0)^2} + \frac{b_t}{(t-1+t_0)^2} \bigg) \bigg)\bigg(\E[\exp(X^{t-1})]\bigg)^{b_tb_{t-1}},
    \end{align*} where the second inequality follows from \eqref{eq:iter-ineq}. Iterating the above inequality and using the fact that $X^1$ is deterministic, we get
    \begin{equation}\label{eq:step-1}
         \E[\exp(\Xtp)] \leq \exp\bigg(X^1\prod_{k = 1}^tb_k + C_1\sum_{k = 1}^t\frac{1}{k+t_0}\prod_{s = k+1}^tb_s + C_2\sum_{k = 1}^t\frac{1}{(k+t_0)^2}\prod_{s = k+1}^tb_s \bigg).
    \end{equation} We now proceed as follows. Define $B_k \coloneqq \prod_{s = k+1}^tb_s = \prod_{s = k+1}^t\Big(1 - \frac{a}{s + t_0} \Big)$, with $B_t = 1$, and note that $B_{k-1} = b_kB_k = \Big(1 - \frac{a}{k+t_0} \Big)B_k$, or equivalently, $\frac{B_k}{k+t_0} = \frac{1}{a}(B_k - B_{k-1})$. Using the said identity and summing up the first $t$ terms, it follows that
    \begin{equation}\label{eq:step-2}
        \sum_{k = 1}^t\frac{1}{k+t_0}\prod_{s = k+1}^tb_s = \sum_{k = 1}^t\frac{B_k}{k+t_0} = \frac{1}{a}(B_t - B_0) \leq \frac{B_t}{a} = \frac{1}{a},
    \end{equation} where the inequality follows from $B_0 > 0$. Next, using Proposition \ref{prop:ran}, we get
    \begin{align}
        \sum_{k = 1}^t\frac{1}{(k+t_0)^2}&\prod_{s = k+1}^tb_s \leq \sum_{k = 1}^t\frac{1}{(k+t_0)^2}\frac{(k+1+t_0)^a}{(t+t_0+1)^a} = \frac{1}{(t+t_0+1)^a}\sum_{k = 1}^t\frac{1}{(k+t_0)^{2-a}}\bigg(1 + \frac{1}{k+t_0}\bigg)^a \nonumber \\
        &\stackrel{(i)}\leq \frac{1}{(t+t_0+1)^a}\bigg(1 + \frac{1}{1+t_0}\bigg)^a\sum_{k = 1}^t\frac{1}{(k+t_0)^{2-a}} \stackrel{(ii)}{\leq} \frac{e^{\frac{a}{1+t_0}}}{(t+t_0+1)^a}\sum_{k = 1}^t\frac{1}{(k+t_0)^{2-a}} \nonumber \\
        &\stackrel{(iii)}{\leq} \frac{3}{(t+t_0+1)^a}\sum_{k = 1}^t\frac{1}{(k+t_0)^{2-a}} \stackrel{(iv)}{\leq} \frac{3/(a-1)}{(t+t_0+1)} \label{eq:step-3}
    \end{align} where $(i)$ follows from the fact that $k \geq 1$, in $(ii)$ we use $(1 + a) \leq e^a$ for any $a \in \R$, $(iii)$ follows from $t_0 \geq a$, while in $(iv)$ we use the Darboux sum approximation. Plugging \eqref{eq:step-2}-\eqref{eq:step-3} into \eqref{eq:step-1} and using Proposition \ref{prop:ran} to bound $\prod_{k = 1}^tb_k$, completes the proof. 
\end{proof}

\begin{remark}
    A similar result was provided in \cite[Lemma 5]{armacki2025dsgd}, where the authors allow multiple additive terms in \eqref{eq:iter-ineq}, of the form $C_i/(t+t_0)^i$, $i \in [M]$, for some $M \geq 2$, while requiring $a \in (1,2]$ and showing an exponential dependence on $a$ in the resulting MGF bound, via a multiplicative factor $2^a$. We allow for a general $a > 1$ and establish a much sharper bound in \eqref{eq:resulting-MGF-bound} for $M = 2$, removing the exponential dependence on $a$.  
\end{remark}

\section{Analysis Setup}\label{sup:analysis-setup}

In this section we define some notation useful for the analysis. To begin, recall that $\git$ denotes the estimator of $\nabla f_i(\xit)$ returned by the \sfo to agent $i$ at time $t$, with $\zit = \git - \nabla f_i(\xit)$ being the induced noise. Similarly, recall the network-averaged quantities $\ogt = \frac{1}{n}\sum_{i \in [n]}\git$, $\onab f_t = \frac{1}{n}\sum_{i \in [n]}\nabla f_i(\xit)$ and $\ozt = \frac{1}{n}\sum_{i \in [n]}\zit$, which satisfy $\ogt = \onab f_t + \ozt$. From the gradient tracker and model updates \eqref{eq:gt-track-update}-\eqref{eq:gt-model-update}, the definitions of network-averaged models, the fact that the weight matrix is doubly stochastic and that $y_i^0 = g_i^0 = 0$, for all $i \in [n]$, it can be readily seen that $\oyt = \ogt$ and $\oxtp = \oxt - \alpha_t\oyt = \oxt - \alpha_t\ogt$. 

We now introduce some vectorized notation, commonly used in the analysis of decentralized methods, e.g., \cite{ran-improved,adaptive-learning-ali,nedic-subgrad}. Let $\bxt \coloneq col(x_1^t,\ldots,x_n^t) \in \R^{nd}$ and $\byt \coloneq col(y_1^t,\ldots,y_n^t) \in \R^{nd}$ denote the column vector stacking agents' local models and trackers. Using this notation, we can then compactly represent the update rule \eqref{eq:gt-track-update}-\eqref{eq:gt-model-update} as 
\begin{align}
    \by^t &= \bW(\by^{t-1} + \bgt - \bg^{t-1}), \label{eq:gt-track-vector}\\
    \bx^{t+1} &= \bW(\bx^t - \alpha_t \byt) \label{eq:gt-model-vector}, 
\end{align} where $\bW = W \otimes I_d \in \R^{nd \times nd}$, $\otimes$ denotes the Kronecker product, while $\bg^t \coloneqq col(g^t_1,\ldots,g^t_n)$. Next, recall the ideal consensus matrix $J = \frac{1}{n}\mathbf{1}_n\mathbf{1}_n^\top \in \R^{n \times n}$. The relationship between $W$ and $J$ plays an important role in the analysis, and we now list some known results, e.g., \cite{chung1997spectral,cvetkovic_rowlinson_simic_1997}.

\begin{proposition}\label{prop:mixing-mx}
    Let Assumption \ref{asmpt:network} hold. Then, the following are true.
    \begin{enumerate}
        \item $W\mathbf{1}_n = J\mathbf{1}_n = \mathbf{1}_n$.
        \item $\|W - J\| = \lambda$, where $\lambda \in [0,1)$ is the second largest singular value of $W$.
        \item $WJ = JW = J$.
    \end{enumerate}
\end{proposition}

Next, define $\obxt \coloneqq \mathbf{1}_n \otimes \oxt \in \R^{nd}$, $\obyt \coloneqq \mathbf{1}_n \otimes \oyt \in \R^{nd}$, $\bJ \coloneqq J \otimes I_d \in \R^{nd \times nd}$ and note that $\obxt = \bJ\bxt$ and $\obyt = \bJ\byt = \obgt$, hence it follows that $\obxtp = \obxt - \alpha_t\obgt$.

\section{Proofs for Non-convex Costs}\label{sup:non-conv}

In this section we provide results for general non-convex costs. Recall that for non-convex costs we use a fixed-step size, i.e., $\alpha_t \equiv \alpha$, for all $t \geq 1$. We start by proving Lemma \ref{lm:descent-inequality}.

\begin{proof}[Proof of Lemma \ref{lm:descent-inequality}]
    The proof follows similar steps as in \cite[Lemma 3]{ran-improved}. Using the third property in Proposition \ref{prop:descent-ineq} with $x = \oxtp$ and $y = \oxt$, we get
    \begin{align}\label{eq:descent-intermed}
        f(\oxtp) &\leq f(\oxt) - \alpha\langle \nabla f(\oxt), \og_t \rangle + \frac{\alpha^2L}{2}\|\og_t\|^2 \nonumber \\ 
        &\stackrel{(i)}{\leq} f(\oxt) - \alpha\langle \nabla f(\oxt), \onab f_t \rangle - \alpha\langle \nabla f(\oxt), \ozt \rangle + \alpha^2L\|\onab f_t\|^2 +\alpha^2L\|\ozt\|^2 \nonumber \\ 
        &\stackrel{(ii)}{=} f(\oxt) - \frac{\alpha}{2}\|\nabla f(\oxt)\|^2 - (1 - 2\alpha L)\frac{\alpha}{2}\|\onab f_t\|^2 \nonumber \\ 
        &\quad\quad\quad\quad + \frac{\alpha}{2}\|\onab f_t - \nabla f(\oxt)\|^2 - \alpha\langle \nabla f(\oxt), \ozt \rangle + \alpha^2L\|\ozt\|^2,
    \end{align} where $(i)$ follows by applying Proposition \ref{prop:Young} with $\theta = 1$, while $(ii)$ follows from the identity $\langle a,b \rangle = \frac{1}{2}\big(\|a\|^2 + \|b\|^2 - \|a-b\|^2 \big)$. Recalling the definition of $\onab f_t$, we get
    \begin{equation}\label{eq:bound-grad-diff}
        \|\onab f_t - \nabla f(\oxt)\|^2 = \Big\|\frac{1}{n}\sum_{i \in [n]}\big[\nabla f_i(\xit) - \nabla f_i(\oxt)\big]\Big\|^2 \leq \frac{L^2}{n}\sum_{i \in [n]}\|\xit - \oxt\|^2,
    \end{equation} where we used Proposition \ref{prop:Jensen} and $L$-smoothness of each $f_i$ in the last inequality. Plugging \eqref{eq:bound-grad-diff} in \eqref{eq:descent-intermed} and noting that the choice $\alpha \leq \frac{1}{4L}$ implies $1 - 2\alpha L \geq \frac{1}{2}$, the claim follows.
\end{proof} 

To prove Lemma \ref{lm:consensus-bound}, we first provide some one-step recursions on the model and gradient tracker consensus gaps. The first is taken from \cite{ran-improved}.

\begin{lemma}\label{lm:consensus-one-step-model}
    Let Assumption \ref{asmpt:network} hold. Then, for any $t \geq 1$, we have
    \begin{equation*}
        \|\bxtp - \obxtp\|^2 \leq \frac{1 + \lambda^2}{2}\|\bxt - \obxt\|^2 + \frac{2\alpha^2\lambda^2}{1 - \lambda^2}\|\byt - \obyt\|^2.
    \end{equation*}
\end{lemma}

The next result characterizes the one-step consensus gap of the gradient tracker. Unlike the MSE analysis, where the authors directly analyze the expected consensus gap of the gradient tracker, e.g., \cite[Lemma 8]{ran-improved}, we provide a deterministic inequality, which is later used when bounding the MGF. This comes at the cost of slightly worse constants, e.g., the authors in \cite{ran-improved} establish that the constant dictating the ``contraction'' of the tracker is $\frac{1+\lambda^2}{2}$, i.e., $\|\bytp - \obytp\|^2 \leq \frac{1+\lambda^2}{2}\|\byt - \obyt\|^2 + \ldots$, whereas we get $\frac{3+\lambda^2}{4}$. 

\begin{lemma}\label{lm:consensus-one-step-tracker}
    Let Assumptions \ref{asmpt:network} and \ref{asmpt:cost} hold. If $\alpha \leq \frac{(1-\lambda^2)^{\nicefrac{3}{2}}}{4\lambda^2L\sqrt{6}}$, we have, for all $t\geq 1$ 
    \begin{align*}
        \|\bytp - \obytp\|^2 &\leq \frac{3+\lambda^2}{4}\|\byt - \obyt\|^2 + \frac{24\lambda^2L^2}{1-\lambda^2}\|\bxt - \obxt\|^2 + \frac{4\lambda^2}{1-\lambda^2}\|\bztp - \bzt\|^2 + \frac{12\alpha^2\lambda^2L^2}{1-\lambda^2}\|\obgt\|^2.
    \end{align*}
\end{lemma}

\begin{proof}
    Recalling the gradient tracker update rule \eqref{eq:gt-track-update} and using Proposition \ref{prop:Young}, we then have
    \begin{align}\label{eq:tracker-bd-pt1}
        \|\bytp - \obytp\|^2 &\leq (1+\theta)\lambda^2\|\byt - \obyt\|^2 + 2(1+\theta^{-1})\lambda^2\big(\|\bztp - \bzt\|^2 + \|\nbftp - \nbft\|^2\big) \nonumber \\
        &\leq \frac{1+\lambda^2}{2}\|\byt - \obyt\|^2 + \frac{4\lambda^2}{1-\lambda^2}\|\bztp - \bzt\|^2 + \frac{4\lambda^2L^2}{1-\lambda^2}\|\bxtp - \bxt\|^2,
    \end{align} where the second inequality follows from Assumption \ref{asmpt:cost} and by choosing $\theta = \frac{1-\lambda^2}{2\lambda^2}$. To bound $\bxtp - \bxt$, we proceed as follows
    \begin{align}\label{eq:tracker-bd-pt2}
        \|\bxtp - \bxt\|^2 &\leq 3\|\bxtp - \obxtp\|^2 + 3\|\obxtp-\obxt\|^2 + 3\|\bxt - \obxt\|^2 \nonumber \\
        &\leq 6\|\bxt - \obxt\|^2 + \frac{6\alpha^2\lambda^2}{1-\lambda^2}\|\byt - \obyt\|^2 + 3\alpha^2\|\obgt\|^2,
    \end{align} where the second inequality follows by applying Lemma \ref{lm:consensus-one-step-model} and using the fact that $\obyt = \obgt$. Plugging \eqref{eq:tracker-bd-pt2} into \eqref{eq:tracker-bd-pt1}, we get
    \begin{align*}
        \|\bytp &- \obytp\|^2 \leq \bigg(\frac{1+\lambda^2}{2} + \frac{24\alpha^2\lambda^4L^2}{(1-\lambda^2)^2} \bigg)\|\byt - \obyt\|^2 + \frac{24\lambda^2L^2}{1-\lambda^2}\|\bxt - \obxt\|^2 \\ &\quad\quad\quad\quad\quad\quad\quad\quad+ \frac{4\lambda^2}{1-\lambda^2}\|\bztp - \bzt\|^2 + \frac{12\alpha^2\lambda^2L^2}{1-\lambda^2}\|\obgt\|^2 \\
        &\leq \frac{3+\lambda^2}{4}\|\byt - \obyt\|^2 + \frac{4\lambda^2}{1-\lambda^2}\Big(6L^2\|\bxt - \obxt\|^2 + \|\bztp - \bzt\|^2 + 3\alpha^2L^2\|\obgt\|^2 \Big),
    \end{align*} where the second inequality follows from the choice of the step-size.
\end{proof}

We are now ready to prove Lemma \ref{lm:consensus-bound}. The proof follows a Linear Time-Invariant (LTI) system approach established in \cite{ran-improved}, with the difference that in \cite{ran-improved} the authors consider the expected process and its dynamics, while here we work directly with deterministic recursions from Lemmas \ref{lm:consensus-one-step-model} and \ref{lm:consensus-one-step-tracker}, containing stochastic noise vectors.

\begin{proof}[Proof of Lemma \ref{lm:consensus-bound}]
     Define $\Xt \coloneqq \frac{1}{nL^2}\begin{bmatrix} L^2\|\bxt - \obxt\|^2 \\ \|\byt - \obyt\|^2 \end{bmatrix}$. Using Lemmas \ref{lm:consensus-one-step-model}-\ref{lm:consensus-one-step-tracker}, we get the following LTI system $\Xtp \leq \bA \Xt + \Ytp$, where the inequality holds component-wise, with 
    \begin{equation*}
        \bA \coloneqq \begin{bmatrix} \frac{1+\lambda^2}{2} & \frac{2\alpha^2\lambda^2L^2}{1-\lambda^2} \\
    \frac{24\lambda^2}{1-\lambda^2} & \frac{3+\lambda^2}{4}\end{bmatrix} \: \text{ and } \: \Ytp \coloneqq \frac{4\lambda^2}{n(1-\lambda^2)}\begin{bmatrix} 0 \\ \frac{\|\bztp - \bzt\|^2}{L^2} + 3\alpha^2\|\obgt\|^2 \end{bmatrix}.   
    \end{equation*} Unrolling the recursion, we get $\Xtp \leq \bA^{t}X^1 + \sum_{k = 0}^t\bA^{t-k}Y^{k+1}$, where we note that $\bz^0 = \obg^0 = \mathbf{0}$ (since our model is initialized at index $t = 1$). Summing up the first $T$ terms, we get
    \begin{align}\label{eq:system}
        \sum_{t \in [T]}X^{t} \leq \sum_{t  = 0}^{T-1}\bAt X^{1} + \sum_{t \in [T]}\sum_{k = 0}^{t-1}\bA^{t-1-k}Y^{k+1} &= \sum_{t  = 0}^{T-1}\bA^kX^{1} + \sum_{t \in [T]}\Bigg(\sum_{k = 0}^{T-t}\bA^{k}\Bigg)Y^{t} \nonumber \\
        &\leq \bigg(\sum_{k = 0}^{\infty}\bA^k\bigg)\bigg(X^1 + \sum_{t \in [T]}\Yt\bigg),
    \end{align} where the last inequality follows from the fact that all quantities are (component-wise) non-negative. Next, if we denote the largest eigenvalue of $\bA$ by $\varrho(\bA)$ and use the fact that if $\varrho(\bA) < 1$, then $\sum_{k = 0}^{\infty}\bA^k = \big(\bI_2 - \bA \big)^{-1}$, e.g., \cite{Horn_Johnson_2012}, our aim then is to show that $\varrho(\bA) < 1$. To that end, it suffices to show the existence of a component-wise positive vector $a = \begin{bmatrix} a_1 & a_2\end{bmatrix}^\top$, such that $\bA a < a$, which is equivalent to the following system of inequalities
    \begin{equation}\label{eq:mx-conditions}
        \begin{bmatrix}
            \frac{1+\lambda^2}{2}a_1 + \frac{2\alpha^2\lambda^2L^2}{1-\lambda^2}a_2 \\
            \frac{24\lambda^2}{1-\lambda^2}a_1 + \frac{3+\lambda^2}{4}a_2
        \end{bmatrix} < \begin{bmatrix} 
            a_1 \\ a_2
        \end{bmatrix} \iff \begin{aligned}
            \frac{a_1}{a_2} &> \frac{4\alpha^2\lambda^2L^2}{(1-\lambda^2)^2} \\
            \frac{a_1}{a_2} &< \frac{(1-\lambda^2)^2}{96\lambda^2}
        \end{aligned}.
    \end{equation} It can be seen that the system \eqref{eq:mx-conditions} is guaranteed to be satisfiable (i.e., a choice of $a_1,a_2 > 0$ satisfying \eqref{eq:mx-conditions} exists), as long as the interval $\Big(\frac{4\alpha^2\lambda^2L^2}{(1-\lambda^2)^2}, \frac{(1-\lambda^2)^2}{96\lambda^2} \Big)$ is non-empty, which is the case if $\alpha \leq \frac{(1-\lambda^2)^2}{8\lambda^2L\sqrt{7}}$. Choosing such $\alpha$, we proceed to estimate $(\bI_2 - \bA)^{-1}$, using the formula
    \begin{equation}\label{eq:inverse}
        \big(\bI_2 - \bA\big)^{-1} = \frac{\textup{adj}\big(\bI_2 - \bA\big)}{\textup{det}\big(\bI_2 - \bA\big)} = \frac{1}{\frac{(1-\lambda^2)^2}{8} - \frac{48\alpha^2\lambda^4L^2}{(1-\lambda^2)^2}}\begin{bmatrix}
            \frac{1-\lambda^2}{4} & \frac{2\alpha^2\lambda^2L^2}{1-\lambda^2} \\
            \frac{24\lambda^2}{1-\lambda^2} & \frac{1-\lambda^2}{2}
        \end{bmatrix} \leq \begin{bmatrix}
            \frac{4}{1-\lambda^2} & \frac{32\alpha^2\lambda^2L^2}{(1-\lambda^2)^3} \\
            \frac{384\lambda^2}{(1-\lambda^2)^3} & \frac{8}{1-\lambda^2}
        \end{bmatrix},
    \end{equation} where the last inequality follows by noting that choosing the step-size $\alpha \leq \frac{(1-\lambda^2)^2}{16\lambda^2L\sqrt{3}}$, we have $\frac{(1-\lambda^2)^2}{8} - \frac{48\alpha^2\lambda^2L^2}{(1-\lambda^2)^2} \geq \frac{(1-\lambda^2)^2}{16}$. Plugging \eqref{eq:inverse} into \eqref{eq:system}, we finally get
    \begin{align*}
        &\frac{1}{n}\sum_{t \in [T]}\|\bxt - \obxt\|^2 \leq \frac{4\|\bx^1-\obx^1\|^2}{n(1-\lambda^2)} + \frac{32\alpha^2\lambda^2\|\by^1 - \oby^1\|^2}{n(1-\lambda^2)^3} + \frac{128\alpha^2\lambda^4}{n(1-\lambda^2)^4}\sum_{t \in [T]}\Big(\|\bzt - \bz^{t-1}\|^2 + 3\alpha^2L^2\|\obg^{t-1}\|^2 \Big) \\ 
        &\leq \frac{4\|\bx^1-\obx^1\|^2}{n(1-\lambda^2)} + \frac{32\alpha^2\lambda^2\|\by^1 - \oby^1\|^2}{n(1-\lambda^2)^3} + \frac{512\alpha^2\lambda^4}{n(1-\lambda^2)^4}\sum_{t \in [T]}\|\bzt\|^2 + \frac{768\alpha^4\lambda^4L^2}{n(1-\lambda^2)^4}\sum_{t \in [T]}\Big(\|\onbft\|^2 + \|\obzt\|^2 \Big),
    \end{align*} where the second inequality follows from the definition of $\obgt$ and that $\bz^0 = \obg^0 = \mathbf{0}$.
\end{proof}

\begin{remark}\label{rmk:sharp-constants}
    As already mentioned, unlike the MSE analysis in \cite{ran-improved}, where consensus gaps are established on the expected quantities, Lemmas \ref{lm:consensus-bound} and \ref{lm:consensus-one-step-tracker} provide deterministic recursions (in the almost sure sense), that contain quantities like the stochastic noise. This is necessary, as we work with the MGF, which is not a linear quantity like the expectation,\footnote{In the sense that $\E[aX + bY] = a\E[X] + b\E[Y]$ always holds, but $\log\E[\exp(aX + bY)] = a\log\E[\exp(X)] + b\log\E[\exp(Y)]$ only holds under specific conditions, e.g., $X,Y$ independent and $a,b \in (0,1)$.} making the analysis significantly more involved. Moreover, the MSE analysis leverages the fact that some inner products vanish (which is not the case in HP analysis, recall Remark \ref{rmk:noise}), facilitating sharper bounds. For example, \cite[Lemma 11]{ran-improved} shows that
    \begin{equation*}
        \frac{1}{n}\sum_{t \in [T]}\E\big[\|\bxt - \obxt\|^2\big] = \bigO\bigg(\underbrace{\frac{\alpha^4}{(1-\lambda^2)^4}\sum_{t \in [T]}\E\big[\|\onbft\|^2\big]}_{\text{gradient offset}} + \underbrace{\frac{\alpha^2\sigma^2T}{(1-\lambda^2)^3}}_{\text{noise}} + \underbrace{\frac{\alpha^2\|\nbf^1\|^2}{(1-\lambda^2)^3}}_{\text{heterogeneity}} \bigg), 
    \end{equation*} while the (deterministic) bound from Lemma \ref{lm:consensus-bound}, ignoring some terms for simplicity, gives
    \begin{equation*}
        \frac{1}{n}\sum_{t \in [T]}\|\bxt - \obxt\|^2 = \bigO\bigg(\underbrace{\frac{\alpha^4}{(1-\lambda^2)^4}\sum_{t \in [T]}\|\onbft\|^2}_{\text{gradient offset}} + \underbrace{\frac{\alpha^2}{(1-\lambda^2)^4}\sum_{t \in [T]}\|\bzt\|^2}_{\text{noise}} + \underbrace{\frac{\alpha^2\|\nbf^1\|^2}{(1-\lambda^2)^3}}_{\text{heterogeneity}} \bigg). 
    \end{equation*} We can see that the noise related term has a worse dependence on the network connectivity, by a factor of $(1-\lambda^2)^{-1}$, which leads to a slightly worse transient time in the final rate.
\end{remark}

We are now ready to prove the main theorem. 

\begin{proof}[Proof of Theorem \ref{thm:main-non-cvx}]
    Start from Lemma \ref{lm:descent-inequality}, rearrange and sum up the first $T$ terms, to get
    \begin{equation}\label{eq:start}
        \frac{\alpha}{2}\sum_{t \in [T]}\|\nabla f(\oxt)\|^2 \leq \Delta_f - \frac{\alpha}{4}\sum_{t \in [T]}\|\onft\|^2 - \alpha\sum_{t \in [T]}\langle \nabla f(\oxt),\ozt\rangle + \alpha^2L\sum_{t \in [T]}\|\ozt\|^2 + \frac{\alpha L^2}{2n}\sum_{t \in [T]}\|\bxt - \obxt\|^2.
    \end{equation} Next, define for all $t \in [T]$ the random variables 
    \begin{align*}
        \Ct &\coloneqq c_1\langle \nabla f(\oxt),\ozt \rangle - \frac{3c_1^2\sigma^2\|\nabla f(\oxt)\|^2}{4n} - \alpha^{2+\varepsilon}\rho\sum_{i \in [n]}\|\nabla f(\xit)\|^2, \\
        \Dt &\coloneqq c_2\|\ozt\|^2 - \frac{\alpha^{4+2\varepsilon}\rho^2}{32\sigma^2}\sum_{i \in [n]}\|\nabla f(\xit)\|^2, \\
        \Et &\coloneqq c_3\|\bzt\|^2 - \alpha^{2+\varepsilon}\rho\sum_{i \in [n]}\|\nabla f(\xit)\|^2,
    \end{align*} with deterministic initialization $C^0 = D^0 = E^0 = 0$, for some $c_1,c_2,c_3 > 0$, where $c_2 \leq \frac{n}{96\sigma^2}$ and $c_3 \leq \frac{1}{\sigmax^2}$. Let $\oCT \coloneqq \sum_{t = 0}^T\Ct$, $\oDT \coloneqq \sum_{t = 0}^T\Dt$ and $\oET \coloneqq \sum_{t  = 0}^T\Et$ and consider first the MGF of $\Ct$ conditioned on $\Ft$. From Lemma \ref{lm:avg-noise-properties} it can be seen that
    $\E_t\big[\exp\big(\Ct\big)\big] \leq \exp\big(\alpha^{2+\varepsilon}n\rho\big)$. Using this fact and successively applying the tower rule, it follows that
    \begin{align*}
        \E\big[\exp\big(\oCT \big) \big] = \E\Big[\exp\big(\oCTm\big) \E_T\big[\exp\big(C^T\big) \big] \Big] &\leq \exp\big(\alpha^{2+\varepsilon}n\rho\big)\E\big[\exp\big(\oCTm\big) \big] \\
        &\leq \ldots \leq \exp\big(\alpha^{2+\varepsilon}n\rho T\big).
    \end{align*} Using the same approach, one can show that $\E_t\big[\exp\big(\Dt\big)\big] \leq \exp\Big(\frac{96\sigma^2c_2}{n}\big(1 + \log(2d)\big) \Big)$, for any $t \geq 1$, readily implying that $\E\big[\exp\big(\oDT\big)\big] \leq \exp\Big(\frac{96\sigma^2c_2T}{n}\big(1 + \log(2d)\big)\Big)$. Finally, consider the conditional MGF of $\Et$. Recalling that $\|\bzt\|^2 = \sum_{i \in [n]}\|\zit\|^2$, we get
    \begin{align*}
        \E_t\big[\exp\big(\Et\big) \big] &= \exp\bigg(-\alpha^{2+\varepsilon}\rho\sum_{i \in [n]}\|\nabla f(\xit)\|^2 \bigg)\E_t\big[\exp\big(c_3\|\bzt\|^2 \big) \big] \\
        &= \exp\bigg(-\alpha^{2+\varepsilon}\rho\sum_{i \in [n]}\|\nabla f(\xit)\|^2 \bigg)\prod_{i \in [n]}\E_t\big[\exp\big(c_3\|\zit\|^2 \big) \big] \\
        &\leq \exp\bigg(-\alpha^{2+\varepsilon}\rho\sum_{i \in [n]}\|\nabla f(\xit)\|^2 \bigg)\prod_{i \in [n]}\exp\Big(c_3\sigma_i^2 + \alpha^{2+\varepsilon}\rho\big(1 + \|\nabla f(\xit)\|^2\big)\Big) \\
        &\leq \exp\big(c_3n\sigma^2 + \alpha^{2+\varepsilon}n\rho\big). 
    \end{align*} Therefore, using the tower rule and unrolling the recursion, it again follows that $\E\big[\exp\big(\oET \big)\big] \leq \exp\big(n(c_3\sigma^2 + \alpha^{2+\varepsilon}\rho)T\big)$. Now, setting $c_1 = 3\alpha$ and using \eqref{eq:start}, we get
    \begin{align}\label{eq:intermed-1}
        \frac{\alpha}{2}\sum_{t \in [T]}\Bigg[&\bigg(1 - \frac{9\alpha\sigma^2}{2n} \bigg)\|\nabla f(\oxt)\|^2 - \frac{\alpha^{2+\varepsilon}\rho}{3}\sum_{i \in [n]}\|\nabla f(\xit)\|^2\Bigg]  \leq \Delta_f - \oCT/3 \nonumber \\
        &- \frac{\alpha}{4}\sum_{t \in [T]}\|\onft\|^2 +\alpha^2L\sum_{t \in [T]}\|\ozt\|^2 + \frac{\alpha L^2}{2n}\sum_{t \in [T]}\|\bxt - \obxt\|^2
    \end{align} First, note that choosing $\alpha \leq \frac{n}{9\sigma^2}$, we have $1 - \frac{9\alpha\sigma^2}{2n} \geq \frac{1}{2}$. Next, note that $\|\nabla f(\xit)\|^2 \leq 2\|\nabla f(\oxt)\|^2 + 2L^2\|\xit - \oxt\|^2$, implying $\|\nabla f(\oxt)\|^2 \geq \frac{1}{2n}\sum_{i \in [n]}\|\nabla f(\xit)\|^2 - \frac{L^2}{n}\|\bxt - \obxt\|^2$. Combining these two facts, it follows from \eqref{eq:intermed-1} that, if $\alpha \leq \frac{n}{9\sigma^2}$, we have
    \begin{align*}
        \frac{\alpha}{8n}\sum_{t \in [T]}\sum_{i \in [n]}&\bigg(1 - \frac{8n\rho\alpha^{1 + \varepsilon}}{3}\bigg)\|\nabla f(\xit)\|^2  \leq \Delta_f - \oCT/3 \nonumber \\
        &- \frac{\alpha}{4}\sum_{t \in [T]}\|\onft\|^2 +\alpha^2L\sum_{t \in [T]}\|\ozt\|^2 + \frac{\alpha L^2}{n}\sum_{t \in [T]}\|\bxt - \obxt\|^2.
    \end{align*} Using Lemma \ref{lm:consensus-one-step-tracker} on the last term, noting that $\|\onbft\|^2 = n\|\onft\|^2$ and $\|\obzt\|^2 = n\|\ozt\|^2$, we get
    \begin{align*}
        &\frac{\alpha}{8n}\sum_{t \in [T]}\sum_{i \in [n]}\bigg(1 - \frac{8n\rho\alpha^{1 + \varepsilon}}{3} \bigg)\|\nabla f(\xit)\|^2 \leq \Delta_f - \oCT/3 - \bigg(\frac{\alpha}{4} - \frac{768\alpha^5\lambda^4L^4}{(1-\lambda^2)^4}\bigg)\sum_{t \in [T]}\|\onft\|^2 \\
        &+ \alpha^2L\bigg(1 + \frac{768\alpha^3\lambda^4L^3}{(1-\lambda^2)^4} \bigg)\sum_{t \in [T]}\|\ozt\|^2 + \frac{4\alpha L^2\Delta_x}{n(1-\lambda^2)} + \frac{32\alpha^3\lambda^2L^2\|\by^1 - \oby^1\|^2}{n(1-\lambda^2)^3} + \frac{512\alpha^3\lambda^4L^2}{n(1-\lambda^2)^4}\sum_{t \in [T]}\|\bzt\|^2.
    \end{align*} Choosing $\alpha \leq \min\Big\{\frac{1-\lambda^2}{4\lambda L\sqrt[4]{12}}, \frac{(1-\lambda^2)^{\nicefrac{4}{3}}}{4\lambda^{\nicefrac{4}{3}}L\sqrt[3]{12}} \Big\}$ implies $\frac{\alpha}{4} - \frac{768\alpha^5\lambda^4L^4}{(1-\lambda^2)^4} \geq 0$ and $\frac{768\alpha^3\lambda^4L^3}{(1-\lambda^2)^4} \leq 1$, hence 
    \begin{align*}
        &\frac{\alpha}{8n}\sum_{t \in [T]}\sum_{i \in [n]}\bigg(1 - \frac{8n\rho\alpha^{1 + \varepsilon}}{3} \bigg)\|\nabla f(\xit)\|^2 \leq \Delta_f - \oCT/3 + 2\alpha^2L\sum_{t \in [T]}\|\ozt\|^2 \\ &\quad + \frac{4\alpha L^2\Delta_x}{n(1-\lambda^2)} + \frac{32\alpha^3\lambda^2L^2\|\by^1 - \oby^1\|^2}{n(1-\lambda^2)^3} + \frac{512\alpha^3\lambda^4L^2}{n(1-\lambda^2)^4}\sum_{t \in [T]}\|\bzt\|^2.
    \end{align*} Next, using Proposition \ref{prop:Young} with $\theta = \frac{1-\lambda^2}{\lambda^2}$ and the definition of $\by^1$, it follows that $\|\by^1 - \oby^1\|^2 \leq \frac{1}{\lambda^2}\|\nbf^1 - \onbf^1\|^2 + \frac{\lambda^2}{(1-\lambda^2)}\|\bz^1\|^2 = \frac{\Delta_g}{\lambda^2} + \frac{\lambda^2}{(1-\lambda^2)}\|\bz^1\|^2$, therefore 
    \begin{align*}
        &\frac{\alpha}{8n}\sum_{t \in [T]}\sum_{i \in [n]}\bigg(1 - \frac{8n\rho\alpha^{1 + \varepsilon}}{3} \bigg)\|\nabla f(\xit)\|^2 \leq \Delta_f - \oCT/3 + 2\alpha^2L\sum_{t \in [T]}\|\ozt\|^2 \\ &\quad\quad + \frac{4\alpha L^2\Delta_x}{n(1-\lambda^2)} + \frac{32\alpha^3L^2\Delta_g}{n(1-\lambda^2)^3} + \frac{538\alpha^3\lambda^4L^2}{n(1-\lambda^2)^4}\sum_{t \in [T]}\|\bzt\|^2.
    \end{align*} Setting $c_2 = 6\alpha^2L$, $c_3 = \frac{1614\alpha^3\lambda^4L^2}{n(1-\lambda^2)^4}$ and $\alpha \leq \min\Big\{\sqrt{\frac{n}{282e\sigma^2dL}}, \frac{\sqrt[3]{n}(1-\lambda^2)^{4/3}}{\lambda^{4/3}\sigmax^{2/3}L^{2/3}\sqrt[3]{1614}}, \frac{\sigma\sqrt{32}}{\rho}\Big\}$, we get
    \begin{align*}
        &\frac{\alpha}{8n}\sum_{t \in [T]}\sum_{i \in [n]}\bigg(1 - \frac{24n\rho\alpha^{1+\varepsilon}}{3} \bigg)\|\nabla f(\xit)\|^2 \leq \Delta_f - \oCT/3 \\ 
        &\quad + \frac{4\alpha L^2\Delta_x}{n(1-\lambda^2)} + \frac{32\alpha^3L^2\Delta_g}{n(1-\lambda^2)^3} + \oDT/3 + \oET/3.
    \end{align*} Choosing $\alpha \leq \sqrt[1+\varepsilon]{\frac{1}{16n\rho^2}}$, it finally follows that 
    \begin{equation}\label{eq:MGF-ineq}
        \frac{\alpha}{16n}\sum_{t \in [T]}\sum_{i \in [n]}\|\nabla f(\xit)\|^2 \leq \Delta_f + \frac{4\alpha L^2\Delta_x}{n(1-\lambda^2)} + \frac{32\alpha^3\lambda^2L^2\|\nbf^1\|^2}{n(1-\lambda^2)^3} + \frac{1}{3}\Big(\oDT + \oET - \oCT \Big).
    \end{equation} Let $M^T \coloneqq \frac{\alpha}{16n}\sum_{t \in [T]}\sum_{i \in [n]}\|\nabla f(\xit)\|^2 - \Delta_f - \frac{4\alpha L^2\Delta_x}{n(1-\lambda^2)} - \frac{32\alpha^3L^2\Delta_g}{n(1-\lambda^2)^3}$. Using \eqref{eq:MGF-ineq}, we get
    \begin{align*}
        \E\big[\exp\big(M^T \big) \big] &\leq \E\bigg[\exp\bigg(\frac{1}{3}\Big(\oCT + \oDT + \oET\Big) \bigg) \bigg] \\ 
        &\leq \sqrt[3]{\E\Big[ \exp\Big(\oCT \Big)\Big] \E\Big[ \exp\Big(\oDT \Big)\Big] \E\Big[ \exp\Big(\oET \Big)\Big] } \\
        &\leq \exp\bigg(\frac{2\alpha^{2+\varepsilon}n\rho T}{3} + \frac{192\sigma^2\alpha^2LT}{n}\big(1+\log(2d)\big) + \frac{1614\alpha^3\sigma^2\lambda^4L^2T}{(1-\lambda^2)^4} \bigg).
    \end{align*} Applying Markov's inequality, we get $\Prob\big(M^T > \epsilon\big) \leq \exp\Big(-\epsilon + \frac{2\alpha^{2+\varepsilon}n\rho T}{3} + \frac{192\sigma^2\alpha^2LT}{n}\big(1+\log(2d)\big) + \frac{1614\alpha^3\sigma^2\lambda^4L^2T}{(1-\lambda^2)^4} \Big)$, for any $\epsilon > 0$ and $T \geq 2$. It follows from the definition of $M^T$ that, for any $\delta \in (0,1)$, with probability at least $1 - \delta$, we have
    \begin{align*}
          &\frac{\alpha}{16n}\sum_{t \in [T]}\sum_{i \in [n]}\|\nabla f(\xit)\|^2 \leq \log(\nicefrac{1}{\delta}) + \Delta_f + \frac{4\alpha L^2\Delta_x}{n(1-\lambda^2)} + \frac{2\alpha^{2+\varepsilon}n\rho T}{3}\\
          & + \frac{32\alpha^3L^2\Delta_g}{n(1-\lambda^2)^3} + \frac{192\sigma^2\alpha^2LT}{n}\big(1+\log(2d)\big) + \frac{1614\alpha^3\sigma^2\lambda^4L^2T}{(1-\lambda^2)^4}.
    \end{align*} Multiplying both sides by $\frac{16}{\alpha T}$, we get, with probability at least $1 - \delta$
    \begin{align*}
          &\quad\quad\quad\frac{1}{nT}\sum_{t \in [T]}\sum_{i \in [n]}\|\nabla f(\xit)\|^2 \leq \frac{16\big(\log(\nicefrac{1}{\delta}) + \Delta_f\big)}{\alpha T} + \frac{64L^2\Delta_x}{n(1-\lambda^2)T} \\
          &+ \frac{32\alpha^{1+\varepsilon}n\rho}{3} + \frac{3072\alpha\sigma^2L}{n}\big(1+\log(2d)\big) + \frac{512\alpha^2L^2\Delta_g}{n(1-\lambda^2)^3T} + \frac{25824\alpha^2\sigma^2\lambda^4L^2}{(1-\lambda^2)^4}.
    \end{align*} If $\alpha = \min\Big\{C, \frac{\sqrt{n}}{\sqrt{T}}\Big\}$, where $C \leq \min\Big\{\frac{1}{4L},\frac{(1-\lambda^2)^2}{16\lambda^2L\sqrt{3}}, \frac{n}{9\sigma^2}, \frac{1-\lambda^2}{4\lambda L\sqrt[4]{12}},\frac{(1-\lambda^2)^{4/3}}{4\lambda^{4/3}L\sqrt[3]{12}},\frac{\sqrt[3]{n}(1-\lambda^2)^{4/3}}{\lambda^{4/3}\sigmax^{2/3}L^{2/3}\sqrt[3]{1614}}, \linebreak \sqrt{\frac{n}{282e\sigma^2dL}}, \frac{\sigma\sqrt{32}}{\rho},\sqrt[1+\varepsilon]{\frac{1}{16n\rho^2}} \Big\}$, we finally have, for any $T \geq 2$, with probability at least $1 - \delta$
    \begin{align*}
          \frac{1}{nT}\sum_{t \in [T]}&\sum_{i \in [n]}\|\nabla f(\xit)\|^2 \leq \frac{16\big(\log(\nicefrac{1}{\delta}) + \Delta_f + 192\sigma^2L(1+\log(2d))\big)}{\sqrt{nT}} + \frac{32n^{\nicefrac{(3+\varepsilon)}{2}}\rho}{3T^{\nicefrac{(1+\varepsilon)}{2}}} \\
          &+ \frac{64L^2\Delta_x}{n(1-\lambda^2)T} + \frac{16\big(\log(\nicefrac{1}{\delta}) + \Delta_f\big)}{CT} + \frac{25824n\sigma^2\lambda^4L^2}{d(1-\lambda^2)^4T} + \frac{512L^2\Delta_g}{d(1-\lambda^2)^3T^2}.
    \end{align*}
\end{proof}

\section{Proofs for P\L{} Costs}\label{sup:PL}

In this section we provide results and proofs for P\L{} costs, beginning by proving Lemma \ref{lm:descent-inequality-PL}.

\begin{proof}[Proof of Lemma \ref{lm:descent-inequality-PL}]
    Start from \eqref{eq:descent-intermed}-\eqref{eq:bound-grad-diff} and note that $\alpha_t \leq \frac{1}{2L}$ implies $1 - 2\alpha_tL \geq 0$, to get
    \begin{align*}
        f(\oxtp) \leq f(\oxt) - \frac{\alpha_t}{2}\|\nabla f(\oxt)\|^2 - \alpha_t\langle \nabla f(\oxt), \ozt \rangle + \alpha_t^2L\|\ozt\|^2 + \frac{\alpha_t L^2}{2n}\|\bxt - \obxt\|^2.
    \end{align*} The proof is completed by using Assumption \ref{asmpt:PL} and subtracting $\fs$ from both sides.
\end{proof}

Next, we show a modified result for the one-step consensus gap of the gradient tracker. 

\begin{lemma}\label{lm:consensus-one-step-tracker-PL}
    Let Assumptions \ref{asmpt:network}-\ref{asmpt:PL} hold. If $\alpha_t \leq \min\Big\{\frac{1}{2L},\frac{(1-\lambda^2)^{\nicefrac{3}{2}}}{4\lambda^2L\sqrt{6}}\Big\}$, for all $t \geq 1$, we have 
    \begin{align*}
        \|\bytp - \obytp\|^2 &\leq \frac{3+\lambda^2}{4}\|\byt - \obyt\|^2 + \frac{36\lambda^2L^2}{1-\lambda^2}\|\bxt - \obxt\|^2 + \frac{24n\alpha_t^2\lambda^2L^2}{1-\lambda^2}\|\ozt\|^2 \\
        &+ \frac{4\lambda^2}{1-\lambda^2}\|\bztp - \bzt\|^2 + \frac{96n\alpha_t^2\lambda^2L^3}{1-\lambda^2}\big(f(\oxt) - \fs\big).
    \end{align*}
\end{lemma}

\begin{proof}
    Recall that Lemma \ref{lm:consensus-one-step-tracker} gives the following bound on the consensus gap of the tracker
    \begin{align}\label{eq:tracker-bdd-init}
        \|\bytp - \obytp\|^2 &\leq \frac{3+\lambda^2}{4}\|\byt - \obyt\|^2 + \frac{24\lambda^2L^2}{1-\lambda^2}\|\bxt - \obxt\|^2 \nonumber \\
        &+ \frac{4\lambda^2}{1-\lambda^2}\|\bztp - \bzt\|^2 + \frac{12\alpha_t^2\lambda^2L^2}{1-\lambda^2}\|\obgt\|^2.
    \end{align} We now further focus on the term $\|\obgt\|^2$. Using the decomposition $\obgt = \onbft + \obzt$ and recalling that $\onbft = \mathbf{1}_n \otimes \onab f^t$, where $\onab f^t = \frac{1}{n}\sum_{i \in [n]}\nabla f_i(\xit)$, it then follows that
    \begin{align}\label{eq:grad-bd}
        \|\obgt\|^2 \leq 2\|\obzt\|^2 + 2\|\onbft\|^2 &\leq 2\|\obzt\|^2 + 4L^2\|\bxt - \obxt\|^2 + 4n\|\nabla f(\oxt)\|^2 \nonumber \\
        &\leq 2\|\obzt\|^2 + 4L^2\|\bxt - \obxt\|^2 + 8nL(f(\oxt) - \fs),
    \end{align} where in the second inequality we used Proposition \ref{prop:descent-ineq}. Plugging \eqref{eq:grad-bd} into \eqref{eq:tracker-bdd-init} and choosing $\alpha_t \leq \frac{1}{2L}$, completes the proof.
\end{proof}

Next, our goal is to show that the MGF of the model consensus gap $\|\bxt - \obxt\|^2$ decays sufficiently fast. To that end, we first establish the following intermediary result. 

\begin{lemma}\label{lm:bdd-MGF-PL}
    Let Assumptions \ref{asmpt:network}-\ref{asmpt:noise} hold. If $\alpha_t \leq \min\Big\{\sqrt[2+\varepsilon]{\frac{1}{\rho}},\sqrt[2+\varepsilon]{\frac{1}{\sigmax^2\rho}}, \sqrt[2+\varepsilon]{\frac{32\sigma^2}{\rho^2}}, \frac{(1-\lambda^2)^2}{\lambda L\sqrt{6K}}, \frac{\sqrt{n}}{8\sigma\sqrt{3L}}, \linebreak \frac{\sqrt{n}}{6L\sqrt{2}},\frac{\sqrt{n}}{96\sigma\lambda L\sqrt{3}}, \sqrt[1+\varepsilon]{\frac{1}{4n\rho}}, \sqrt[1+\varepsilon]{\frac{\mu}{4n\rho}}, \sqrt[1+\varepsilon]{\frac{\mu}{16n\rho L}}, \sqrt[5+2\varepsilon]{\frac{1}{18\rho^2L^2}}, \frac{n\mu}{384\sigma^2L}, \frac{1}{6}, \frac{1}{6L^2}, \frac{1}{1280\lambda^2L}  \Big\}$, then for any $\nu \leq \min\Big\{1,\frac{n}{96\sigmax^2\lambda^2}\Big\}$, $K \geq \max\Big\{\frac{2L^2}{\mu},640\lambda^2L^2 \Big\}$ and $t \geq 1$, we have 
    \begin{align*}
        \E\bigg[\exp\bigg(\frac{\nu K}{n}\|\bxt - \obxt\|^2\bigg)\bigg] &\leq \exp(\nu a_1) \\
        \E\Big[\exp\big(\nu(f(\oxt) - \fs)\big)\Big] &\leq \exp(\nu a_2)\\
        \E\bigg[\exp\bigg(\frac{\nu(1-\lambda^2)^2}{n}\|\byt - \obyt\|^2\bigg)\bigg] &\leq \exp(\nu a_3),
    \end{align*} where $a_1 = \max\Big\{\frac{K}{n}\|\bx^1 - \obx^{1}\|^2, \frac{1}{2}\big(f(\ox^1) - \fs\big), \frac{1}{2}\big(2 + \log(2d)\big), 256\sigma^2\lambda^2\big(2+\log(2d)\big), \linebreak \frac{8}{3}\lambda^2(1-\lambda^2)^2\Big(\sigma^2 + 4L\big(f(\ox^1) - \fs\big) + \frac{2L^2}{n}\|\bx^1 - \obx^1\|^2\Big)\Big\}$, $a_2 = 2a_1$ and $a_3= \frac{3}{2}a_1$.
\end{lemma}

\begin{proof}
    Define $X^t \coloneqq \frac{K}{n}\|\bxt - \obxt\|^2$, $Y^t \coloneqq \frac{(1-\lambda^2)^2}{n}\|\byt - \obyt\|^2$ and $F^t \coloneqq f(\oxt) - \fs$. We use induction to show that MGFs of $\nu X^t$, $\nu Y^t$ and $\nu F^t$ are all bounded, for $\nu \in (0,1]$. As the basis of induction, consider $t = 1$. Since the initial model is deterministic, we get
    \begin{align*}
        \E\Big[\exp\big(\nu X^1 \big) \Big] \leq \exp\bigg(\frac{\nu K}{n}\|\bx^1 - \obx^1\|^2 \bigg) \: \text{ and } \: \E\Big[\exp\big(\nu F^1\Big)\big] \leq \exp\big(\nu(f(\ox^1) - \fs)\big).
    \end{align*} Next, using the fact that $\by^0 = \bg^0 = \mathbf{0}$, we get
    \begin{align}\label{eq:base-ind-grad-track}
        &\E\bigg[\exp\bigg(\nu Y^1 \bigg) \bigg] = \E\bigg[\exp\bigg(\frac{\nu(1-\lambda^2)^2}{n}\|\tbW\bg^1\|^2 \bigg) \bigg] \nonumber \\
        &\leq \exp\bigg(\frac{2\nu\lambda^2(1-\lambda^2)^2}{n}\|\nbf^1\|^2 \bigg)\E\bigg[\exp\bigg(\frac{2\nu\lambda^2(1-\lambda^2)^2}{n}\|\bz^1\|^2 \bigg) \bigg] \\
        &\leq \exp\bigg(\frac{4\nu\lambda^2(1-\lambda^2)^2}{n}\big(L^2\|\bx^1 - \obx^1\|^2 + n\|\nabla f(\ox^1)\|^2\big)\bigg)\E\bigg[\exp\bigg(\frac{2\nu\lambda^2(1-\lambda^2)^2}{n}\|\bz^1\|^2 \bigg) \bigg] \nonumber \\
        &\leq \exp\bigg(4\nu\lambda^2(1-\lambda^2)^2\bigg(\frac{L^2}{n}\|\bx^1 - \obx^1\|^2 + 2L\big(f(\ox^1)-\fs\big)\bigg)\bigg)\E\bigg[\exp\bigg(\frac{2\nu\lambda^2(1-\lambda^2)^2}{n}\|\bz^1\|^2 \bigg) \bigg] \nonumber.
    \end{align} Using Lemma \ref{lm:noise-properties} and choosing $\nu \leq \frac{n}{2\sigmax^2\lambda^2(1-\lambda^2)^2}$, it follows that
    \begin{align}\label{eq:base-ind-noise-intermed}
        &\E_1\bigg[\exp\bigg(\frac{2\nu\lambda^2(1-\lambda^2)^2}{n}\|\bz^1\|^2\bigg)\bigg] \leq \prod_{i \in [n]}\Bigg(\E_1\bigg[\exp\bigg(\frac{\|z_i^1\|^2}{\sigma_i^2}\bigg)\bigg]\Bigg)^{\frac{2\nu \sigma_i^2\lambda^2(1-\lambda^2)^2}{n}} \nonumber \\
        &\leq \prod_{i \in [n]}\exp\bigg(\frac{2\nu\sigma_i^2\lambda^2(1-\lambda^2)^2}{n}\Big(1 + \alpha_1^{2+\varepsilon}\rho + \alpha_1^{2+\varepsilon}\rho\|\nabla f(x_i^1)\|^2\Big)\bigg) \\
        &\leq \exp\bigg(2\nu\sigma^2\lambda^2(1-\lambda^2)^2\big(1 + \alpha_1^{2+\varepsilon}\rho\big) + \frac{4\nu \alpha_1^{2+\varepsilon}\rho\lambda^2(1-\lambda^2)^2}{n}\sum_{i \in [n]}\sigma_i^2\Big(\|\nabla f(\ox^1)\|^2 + L^2\|x_i^1 - \ox^1\|^2\Big)\bigg) \nonumber \\
        &\leq \exp\bigg(4\nu\lambda^2(1-\lambda^2)^2\bigg(\sigma^2 + 2\alpha_1^{2+\varepsilon}\sigma^2\rho L\big(f(\ox^1) - \fs\big) + \frac{\alpha_1^{2+\varepsilon}\sigmax^2\rho L^2}{n}\|\bx^1 - \obx^1\|^2\bigg)\bigg), \nonumber
    \end{align} where the last inequality follows by choosing $\alpha_1 \leq \sqrt[2+\varepsilon]{\frac{1}{\rho}}$. Recalling that the initialization is deterministic, choosing $\alpha_1 \leq \sqrt[2+\varepsilon]{\frac{1}{\sigmax^2\rho}}$ and using \eqref{eq:base-ind-grad-track}-\eqref{eq:base-ind-noise-intermed}, we get
    \begin{equation*}
        \E\big[\exp\big(\nu Y^1\big)\big] \leq \exp\bigg(4\nu\lambda^2(1-\lambda^2)^2\bigg(\sigma^2 + \frac{2L^2}{n}\|\bx^1 - \obx^1\|^2 + 4L\big(f(\ox^1) - \fs\big)\bigg) \bigg).
    \end{equation*} Now, assume that, for some $t \geq 1$ and $a_1 \geq \frac{K}{n}\|\bx^1 - \obx^1\|^2$, $a_2 \geq f(\ox^1) - \fs$ and $a_3 \geq 4\lambda^2(1-\lambda^2)^2\Big(\sigma^2 + 4L\big(f(\ox^1) - \fs\big) + \frac{2L^2}{n}\|\bx^1 - \obx^1\|^2\Big)$, we have 
    \begin{align*}
        \E\Big[\exp\big(\nu X^t \big)\Big] \leq \exp(\nu a_1), \quad \E\Big[\exp\big(\nu F^t\big)\Big] \leq \exp(\nu a_2) \quad \text{and} \quad \E\Big[\exp\big(\nu Y^t \big) \Big] \leq \exp(\nu a_3).
    \end{align*} Consider the iteration $t + 1$. Using Lemma \ref{lm:consensus-one-step-model}, we have
    \begin{align*}
        &\E\Big[\exp\big(\nu X^{t+1} \big)\Big] \leq \E\bigg[\exp\bigg( \frac{\nu(1+\lambda^2)}{2}X^t + \frac{2\nu\alpha_t^2\lambda^2K}{(1-\lambda^2)^3}Y^t \bigg)\bigg] \\
        &\leq \bigg(\E\Big[\exp\big(\nu X^t\big)\Big]\bigg)^{\frac{1+\lambda^2}{2}} \Bigg(\E\bigg[\exp\bigg(\frac{4\nu\alpha_t^2\lambda^2K}{(1-\lambda^2)^4}Y^t \bigg)\bigg]\Bigg)^{\frac{1-\lambda^2}{2}} \\
        &\leq \bigg(\E\Big[\exp\big(\nu X^t\big)\Big]\bigg)^{\frac{1+\lambda^2}{2}} \bigg(\E\Big[\exp\big(\nu Y^t \big)\Big]\bigg)^{\frac{2\alpha_t^2\lambda^2K}{(1-\lambda^2)^3}} \\ 
        &\leq \exp\bigg(\frac{1+\lambda^2}{2}\nu a_1 + \frac{2\alpha_t^2\lambda^2K}{(1-\lambda^2)^3}\nu a_3 \bigg) \leq \exp(\nu a_1),
    \end{align*} where the second inequality follows from Proposition \ref{prop:Holder} with $p = \frac{2}{1+\lambda^2}$ and $q = \frac{2}{1-\lambda^2}$, the third follows from $\alpha_t \leq \frac{(1-\lambda^2)^2}{2\lambda\sqrt{K}}$, the fourth follows from the induction hypothesis, while the fifth follows from the choice $\alpha_t \leq \frac{(1-\lambda^2)^2}{2\lambda\sqrt{K}}\sqrt{\frac{a_1}{a_3}}$. Next, using Lemma \ref{lm:descent-inequality-PL} and Proposition \ref{prop:Holder}, we get
    \begin{align*}
        \E_t\big[\exp&\big(\nu F^{t+1}\big)\big] \leq \exp\bigg((1-\alpha_t\mu)\nu F^t + \frac{\nu\alpha_tL^2}{2n}\|\bxt - \obxt\|^2 \bigg)\\
        &\times \sqrt{\E_t\big[\exp\big( -2\nu\alpha_t\langle \nabla f(\oxt),\ozt \rangle\big)\big]\E_t\big[\exp\big(2\nu\alpha_t^2L\|\ozt\|^2\big)\big]} \\ 
        &\stackrel{(i)}{\leq} \exp\bigg((1-\alpha_t\mu)\nu F^t + \frac{\nu\alpha_tL^2}{2n}\|\bxt - \obxt\|^2 \bigg)\Big(\E_t\big[\exp\big(-2\alpha_t\langle \nabla f(\oxt),\ozt \rangle \big) \big] \Big)^{\nu/2}\\
        &\times \exp\bigg(\frac{96\nu\alpha_t^2\sigma^2L}{n}\big(1+\log(2d)\big) +  \frac{\nu\alpha_t^{4+2\varepsilon}\rho^2}{64\sigma^2}\sum_{i \in [n]}\|\nabla f(\xit)\|^2\bigg) \\
        &\stackrel{(ii)}{\leq} \exp\bigg((1-\alpha_t\mu)\nu F^t + \frac{\nu\alpha_tL^2}{2n}\|\bxt - \obxt\|^2 + \frac{3\nu\alpha_t^2\sigma^2\|\nabla f(\oxt)\|^2}{2n} + \frac{\nu\alpha_t^{2+\varepsilon}n\rho}{2}\bigg)\\
        &\times \exp\bigg(\frac{96\nu\alpha_t^2\sigma^2L}{n}\big(1+\log(2d)\big) +  \bigg(\frac{\nu\alpha_t^{4+2\varepsilon}\rho^2}{64\sigma^2} + \frac{\nu\alpha_t^{2+\varepsilon}\rho}{2}\bigg)\sum_{i \in [n]}\|\nabla f(\xit)\|^2\bigg) \\
        &\stackrel{(iii)}{\leq} \exp\bigg((1-\alpha_t\mu)\nu F^t + \frac{\nu\alpha_tL^2}{2n}\|\bxt - \obxt\|^2 + \frac{3\nu\alpha_t^2\sigma^2\|\nabla f(\oxt)\|^2}{2n} + \frac{\nu\alpha_t^{2+\varepsilon}n\rho}{2}\bigg)\\
        &\times \exp\bigg(\frac{96\nu\alpha_t^2\sigma^2L}{n}\big(1+\log(2d)\big) + 2\nu\alpha_t^{2+\varepsilon}n\rho\bigg(\|\nabla f(\oxt)\|^2 + \frac{L^2}{n}\|\bxt - \obxt\|^2\bigg)\bigg) \\
        &\stackrel{(iv)}{\leq} \exp\bigg(\bigg(1-\alpha_t\mu + \alpha_t\bigg(4n\alpha_t^{1+\varepsilon}\rho L + \frac{3\alpha_t\sigma^2L}{n}\bigg)\bigg)\nu F^t\bigg) \\
        &\times \exp\bigg( \frac{96\nu\alpha_t^2\sigma^2L}{n}\big(1+ \log(2d)\big) + \frac{\nu\alpha_t^{2+\varepsilon}n\rho}{2} + \bigg(\frac{1}{2} + 2n\alpha_t^{1+\varepsilon}\rho \bigg) \frac{\nu\alpha_tL^2}{n}\|\bxt - \obxt\|^2 \bigg)\\
        &\stackrel{(v)}{\leq} \exp\bigg( \big(1-\alpha_t\mu/2\big)\nu F^t + \frac{\nu\alpha_tL^2}{n}\|\bxt - \obxt\|^2 + \nu\alpha_t^2\bigg(\frac{96\sigma^2L}{n}\big(1+\log(2d)\big) + \frac{\alpha_t^\varepsilon n\rho}{2}\bigg)\bigg),
    \end{align*} where $(i)$ follows from $\nu \in (0,1]$, Proposition \ref{prop:Jensen}, Lemma \ref{lm:avg-noise-properties} and choosing $\alpha_t \leq \frac{\sqrt{n}}{8\sigma\sqrt{3L}}$, in $(ii)$ we use Lemma \ref{lm:avg-noise-properties}, $(iii)$ follows by choosing $\alpha_t \leq \sqrt[2+\varepsilon]{\frac{32\sigma^2}{\rho}}$ and $\|\nabla f(\xit)\|^2 \leq 2\|\nabla f(\oxt)\|^2 + 2L^2\|\xit - \oxt\|^2$, $(iv)$ follows from Proposition \ref{prop:descent-ineq},  while $(v)$ follows from $\alpha_t \leq \min\Big\{\sqrt[1+\varepsilon]{\frac{1}{4n\rho}},\sqrt[1+\varepsilon]{\frac{\mu}{16n\rho L}},\frac{n\mu}{12\sigma^2L} \Big\}$. Taking the full expectation, we have
    \begin{align*}
        \E\big[\exp\big(\nu F^{t+1}\big) \big] &\leq \exp\bigg(\nu\alpha_t^2\bigg(\frac{96\sigma^2L}{n}\big(1+\log(2d)\big) + \frac{\alpha_t^\varepsilon n\rho}{2}\bigg)\bigg) \\
        &\times \E\bigg[\exp\bigg(\big(1-\alpha_t\mu/2\big)\nu F^t + \frac{\nu\alpha_tL^2}{n}\|\bxt - \obxt\|^2 \bigg)\bigg] \\
        &\leq \exp\bigg(\nu\alpha_t^2\bigg(\frac{96\sigma^2L}{n}\big(1+\log(2d)\big) + \frac{\alpha_t^\varepsilon n\rho}{2}\bigg)\bigg) \\
        &\times \bigg(\E\big[\exp\big(\nu F^t\big)\big]\bigg)^{1-\frac{\alpha_t\mu}{2}}\Bigg(\E\bigg[\exp\bigg(\frac{2\nu L^2}{n\mu}\|\bxt - \obxt\|^2\bigg)\bigg]\Bigg)^{\frac{\alpha_t\mu}{2}} \\
        &\leq \exp\bigg(\nu\alpha_t^2\bigg(\frac{96\sigma^2L}{n}\big(1+\log(2d)\big) + \frac{\alpha_t^\varepsilon n\rho}{2}\bigg) + \bigg(1 - \frac{\alpha_t\mu}{2}\bigg)\nu a_2 + \frac{\alpha_t\mu}{2}\nu a_1\bigg) \\
        &\leq \exp\bigg(\frac{\alpha_t\mu\nu a_2}{4} + \bigg(1 - \frac{\alpha_t\mu}{2}\bigg)\nu a_2 + \frac{\alpha_t\mu\nu a_2}{4} \bigg) = \exp\big(\nu a_2\big),
    \end{align*} where the second inequality follows from Proposition \ref{prop:Holder} with $p = \Big(1-\frac{\alpha_t\mu}{2}\Big)^{-1}$ and $q = \frac{2}{\alpha_t\mu}$, the third follows from $K \geq \frac{2L^2}{\mu}$ and the induction hypothesis, while the fourth inequality follows from the choice $\alpha_t \leq \min\Big\{\frac{n\mu}{384\sigma^2L},\sqrt[1+\varepsilon]{\frac{\mu}{4n\rho}}\Big\}$ and $a_2 \geq \max\big\{2+\log(2d),2a_1\big\}$. Finally, starting from Lemma \ref{lm:consensus-one-step-tracker-PL}, we get
    \begin{align}\label{eq:bounding-mgf-step-1}
        &\E_{t+1}\Big[\exp\big(\nu Y^{t+1}\big)\Big]\leq \exp\bigg(\frac{(3+\lambda^2)}{4}\nu Y^t + 12\nu\lambda^2L^2(1-\lambda^2)\bigg(\frac{3}{n}\|\bxt - \obxt\|^2 + 2\alpha_t^2\|\ozt\|^2\bigg)\bigg) \nonumber \\
        &\times \exp\bigg(8\nu\lambda^2(1-\lambda^2)\bigg(12\alpha_t^2L^3F^t + \frac{\|\bzt\|^2}{n}\bigg)\bigg)\E_{t+1}\bigg[\exp\bigg(\frac{8\nu\lambda^2(1-\lambda^2)}{n}\|\bztp\|^2\bigg)\bigg] \nonumber \\ 
        &\leq \exp\bigg(\frac{3+\lambda^2}{4}\nu Y^t + 4\nu\lambda^2L^2(1-\lambda^2)\bigg(\frac{9}{n}\|\bxt - \obxt\|^2 + 6\alpha_t^2\|\ozt\|^2 + 24\alpha_t^2LF^t + \frac{2}{nL^2}\|\bzt\|^2\bigg)\bigg) \nonumber \\
        &\times \exp\bigg(8\nu\lambda^2(1-\lambda^2)\bigg(\sigma^2 + \alpha_t^{2+\varepsilon}\rho\sigma^2 + 4\alpha_t^{2+\varepsilon}\sigma^2\rho LF^{t+1} + \frac{2\alpha_t^{2+\varepsilon}\sigmax^2\rho L^2}{n}\|\bxtp - \obxtp\|^2\bigg)\bigg) \nonumber \\ 
        &\leq \exp\bigg(\frac{3+\lambda^2}{4}\nu Y^t + 4\nu\lambda^2L^2(1-\lambda^2)\bigg(\frac{9}{n}\|\bxt - \obxt\|^2 + 6\alpha_t^2\|\ozt\|^2 + 24\alpha_t^2LF^t + \frac{2}{nL^2}\|\bzt\|^2\bigg)\bigg) \nonumber \\
        &\times \exp\bigg(8\nu\lambda^2(1-\lambda^2)\bigg(2\sigma^2 + 4\alpha_tLF^{t+1} + \frac{2\alpha_tL^2}{n}\|\bxtp - \obxtp\|^2\bigg)\bigg),
    \end{align} where the second inequality follows from $\nu \leq \frac{n}{8\sigmax^2\lambda^2(1-\lambda^2)}$ and Assumption \ref{asmpt:noise}, while the third follows from the step-size choice $\alpha_t \leq \min\Big\{\sqrt[2+\varepsilon]{\frac{1}{\rho}},\sqrt[1+\varepsilon]{\frac{1}{\sigmax^2\rho}}\Big\}$. Taking the expectation conditioned on $\Ft$, using the tower rule and Proposition \ref{prop:Holder} with $p = \frac{4}{3+\lambda^2}$ and $q = \frac{4}{1-\lambda^2}$, we then have
    \begin{align*}
        &\E_t\Big[\exp\big(\nu Y^{t+1}\big)\Big] \leq \exp\bigg(4\nu\lambda^2(1-\lambda^2)\bigg(4\sigma^2 + \frac{9L^2}{n}\|\bxt - \obxt\|^2 + 24\alpha_t^2L^3F^t\bigg)\bigg)\E_t\bigg[\exp\bigg(\frac{3+\lambda^2}{4}\nu Y^t\bigg)\bigg] \\
        &\times \E_t\bigg[\exp\bigg(8\nu\lambda^2(1-\lambda^2)\bigg(3\alpha_t^2L^2\|\ozt\|^2 + \frac{\|\bzt\|^2}{n} + 4\alpha_tLF^{t+1} + \frac{2\alpha_tL^2}{n}\|\bxtp - \obxtp\|^2\bigg)\bigg)\bigg] \\
        &\leq \exp\bigg(4\nu\lambda^2(1-\lambda^2)\bigg(4\sigma^2 + \frac{9L^2}{n}\|\bxt - \obxt\|^2 + 24\alpha_t^2L^3F^t\bigg)\bigg)\bigg(\E_t\big[\exp\big(\nu Y^t\big)\big]\bigg)^{\frac{3+\lambda^2}{4}} \\
        &\times \Bigg(\E_t\bigg[\exp\bigg(32\nu\lambda^2\bigg(3\alpha_t^2L^2\|\ozt\|^2 + \frac{\|\bzt\|^2}{n} + 4\alpha_tLF^{t+1} + \frac{2\alpha_tL^2}{n}\|\bxtp - \obxtp\|^2\bigg)\bigg)\bigg]\Bigg)^{\frac{1-\lambda^2}{4}}.
    \end{align*} Next, applying Proposition \ref{prop:gen-Holder} with $k = 3$ to the last expectation, it follows that
    \begin{align*}
        \E_t\Big[&\exp\big(\nu Y^{t+1}\big)\Big] \leq \exp\bigg(4\nu\lambda^2(1-\lambda^2)\bigg(4\sigma^2 + \frac{9L^2}{n}\|\bxt - \obxt\|^2 + 24\alpha_t^2L^3F^t\bigg)\bigg) \\
        &\times \bigg(\E_t\big[\exp\big(\nu Y^t\big)\big]\bigg)^{\frac{3+\lambda^2}{4}}\Bigg(\E_t\big[\exp\big(288\nu\alpha_t^2\lambda^2L^2\|\ozt\|^2\big)\big]\E_t\bigg[\exp\bigg(\frac{96\nu\lambda^2}{n}\|\bzt\|^2\bigg)\bigg]\Bigg)^{\frac{1-\lambda^2}{12}} \\
        &\times\Bigg(\E_t\bigg[\exp\bigg(96\nu\lambda^2\bigg( 4\alpha_tLF^{t+1} + \frac{2\alpha_tL^2}{n}\|\bxtp - \obxtp\|^2\bigg)\bigg)\bigg]\Bigg)^{\frac{1-\lambda^2}{12}} \\
         &\leq \exp\bigg(4\nu\lambda^2(1-\lambda^2)\bigg(8\sigma^2 + \frac{576\alpha_t^2\sigma^2L^2}{n}\big(1+\log(2d)\big) + \frac{10L^2}{n}\|\bxt - \obxt\|^2 + 16\alpha_tLF^t\bigg)\bigg) \\ &\times\bigg(\E_t\big[\exp\big(\nu Y^t\big)\big]\bigg)^{\frac{3+\lambda^2}{4}}\Bigg(\E_t\bigg[\exp\bigg(192\nu\lambda^2\bigg( 2\alpha_tLF^{t+1} + \frac{\alpha_tL^2}{n}\|\bxtp - \obxtp\|^2\bigg)\bigg)\bigg]\Bigg)^{\frac{1-\lambda^2}{12}},
    \end{align*} where the second inequality follows from Lemma \ref{lm:avg-noise-properties} and choosing the parameter $\nu \leq \frac{n}{96\sigmax^2\lambda^2}$ and the step-size as $\alpha_t \leq \min\Big\{\sqrt[1 +\varepsilon]{\frac{1}{\sigmax^2\rho}},\sqrt[5+2\varepsilon]{\frac{1}{18\rho^2L^2}},\sqrt{\frac{n}{27648\sigma^2\lambda^2L^2}},\frac{1}{6},\frac{1}{6L^2} \Big\}$. Taking the full expectation, it follows that
    \begin{align*}
        &\E\big[\exp\big(\nu Y^{t+1}\big)\big] \leq \exp\bigg(32\nu\sigma^2\lambda^2(1-\lambda^2) + \frac{2304\nu\alpha_t^2\sigma^2\lambda^2(1-\lambda^2)L^2}{n}\big(1+\log(2d)\big) \bigg) \\
        &\E\Bigg[\exp\bigg(8\nu\lambda^2(1-\lambda^2)L\bigg(\frac{5L}{n}\|\bxt - \obxt\|^2 + 8\alpha_tF^t\bigg)\bigg(\E_t\big[\exp\big(\nu Y^t\big)\big]\bigg)^{\frac{3+\lambda^2}{4}} \\
        &\times \Bigg(\E_t\bigg[\exp\bigg(192\nu\lambda^2\bigg( 2\alpha_tLF^{t+1} + \frac{\alpha_tL^2}{n}\|\bxtp - \obxtp\|^2\bigg)\bigg)\bigg]\Bigg)^{\frac{1-\lambda^2}{12}} \Bigg]. 
    \end{align*} Applying Proposition \ref{prop:Holder} with $p = \frac{4}{3+\lambda^2}$ and $q = \frac{4}{1-\lambda^2}$, we get
    \begin{align*}
        &\E\big[\exp\big(\nu Y^{t+1}\big)\big] \leq \exp\bigg(32\nu\sigma^2\lambda^2(1-\lambda^2) + \frac{2304\nu\alpha_t^2\sigma^2\lambda^2(1-\lambda^2)L^2}{n}\big(1+\log(2d)\big) \bigg) \\
        &\times\bigg(\E\big[\exp\big(\nu Y^t\big)\big]\bigg)^{\frac{3+\lambda^2}{4}} \Bigg(\E\Bigg[\exp\bigg(32\nu\lambda^2L\bigg(\frac{5L}{n}\|\bxt - \obxt\|^2 + 8\alpha_tF^t\bigg) \\
        &\times \Bigg(\E_t\bigg[\exp\bigg(192\nu\lambda^2\bigg( 2\alpha_tLF^{t+1} + \frac{\alpha_tL^2}{n}\|\bxtp - \obxtp\|^2\bigg)\bigg)\bigg]\Bigg)^{\frac{1}{3}} \Bigg]\Bigg)^{\frac{1-\lambda^2}{4}} \\
        &\stackrel{(i)}{\leq} \exp\bigg(32\nu\sigma^2\lambda^2(1-\lambda^2) + \frac{2304\nu\alpha_t^2\sigma^2\lambda^2(1-\lambda^2)L^2}{n}\big(1+\log(2d)\big) + \frac{3+\lambda^2}{4}\nu a_3\bigg) \\
        &\times\Bigg(\E\bigg[\exp\bigg(64\nu\lambda^2L\bigg(\frac{5L}{n}\|\bxt - \obxt\|^2 + 6\alpha_tF^t\bigg)\bigg] \\ 
        &\times\E\Bigg[\E_t\bigg[\exp\bigg(192\nu\lambda^2\bigg( 2\alpha_tLF^{t+1} + \frac{\alpha_tL^2}{n}\|\bxtp - \obxtp\|^2\bigg)\bigg)\bigg]\Bigg)^{\frac{2}{3}} \Bigg]\Bigg)^{\frac{1-\lambda^2}{8}} \\
        &\stackrel{(ii)}{\leq} \exp\bigg(32\nu\sigma^2\lambda^2(1-\lambda^2) + \frac{2304\nu\alpha_t^2\sigma^2\lambda^2(1-\lambda^2)L^2}{n}\big(1+\log(2d)\big) + \frac{3+\lambda^2}{4}\nu a_3\bigg) \\
        &\times\Bigg(\sqrt{\E\bigg[\exp\bigg(\frac{640\nu\lambda^2L^2}{n}\|\bxt - \obxt\|^2\bigg)\bigg]\E\bigg[\exp\bigg(768\nu\alpha_t\lambda^2LF^t\bigg)\bigg]} \\ 
        &\times\E\Bigg[\E_t\bigg[\exp\bigg(192\nu\lambda^2\bigg( 2\alpha_tLF^{t+1} + \frac{\alpha_tL^2}{n}\|\bxtp - \obxtp\|^2\bigg)\bigg)\bigg]\Bigg)^{\frac{2}{3}} \Bigg]\Bigg)^{\frac{1-\lambda^2}{8}} \\
        &\stackrel{(iii)}{\leq} \exp\bigg(32\nu\sigma^2\lambda^2(1-\lambda^2) + \frac{2304\nu\alpha_t^2\sigma^2\lambda^2(1-\lambda^2)L^2}{n}\big(1+\log(2d)\big) + \frac{3+\lambda^2}{4}\nu a_3 + \frac{1-\lambda^2}{16}\nu a_1 \bigg) \\
        &\times \exp\Big(48\alpha_t\lambda^2(1-\lambda^2)L\nu a_2\Big)\Bigg(\E\bigg[\exp\bigg(192\nu\lambda^2\bigg( 2\alpha_tLF^{t+1} + \frac{\alpha_tL^2}{n}\|\bxtp - \obxtp\|^2\bigg)\bigg)\bigg]\Bigg)^{\frac{1-\lambda^2}{12}}
    \end{align*} where the $(i)$ follow by applying Proposition \ref{prop:Holder} with $p = q = 2$ and using the induction hypothesis, in $(ii)$ we again used Proposition \ref{prop:Holder} with $p = q = 2$, while $(iii)$ follows from the induction hypothesis, choosing $K \geq 640\lambda^2L^2$, $\alpha_t \leq \frac{1}{768\lambda^2L}$ and by using Proposition \ref{prop:Jensen} and the tower rule. Finally, applying Proposition \ref{prop:Holder} with $p = q = 2$ once more, choosing $\alpha_t \leq \frac{1}{768\lambda^2L}$ and $K \geq 384\lambda^2L^2$ and using the induction hypothesis, we then get
    \begin{align*}
        &\E\big[\exp\big(\nu Y^{t+1}\big)\big] \leq \exp\bigg(32\nu\sigma^2\lambda^2(1-\lambda^2)\bigg(1 + \frac{72\alpha_t^2L^2}{n}\big(1+\log(2d)\big)\bigg) + \frac{3+\lambda^2}{4}\nu a_3 + \frac{1-\lambda^2}{16}\nu a_1\bigg) \\
        &\times \exp\Big(48\alpha_t\lambda^2(1-\lambda^2)L\nu a_2\Big)\Bigg(\E\bigg[\exp\bigg(768\nu\alpha_t\lambda^2LF^{t+1}\bigg)\bigg]\E\bigg[\exp\bigg(\frac{384\nu\alpha_t\lambda^2L^2}{n}\|\bxtp - \obxtp\|^2\bigg)\bigg]\Bigg)^{\frac{1-\lambda^2}{24}} \\
        &\leq \exp\bigg(32\nu\sigma^2\lambda^2(1-\lambda^2)\bigg(1 + \frac{72\alpha_t^2L^2}{n}\big(1+\log(2d)\big)\bigg) + \frac{3+\lambda^2}{4}\nu a_3 + \frac{1-\lambda^2}{8}\nu a_1 + 80\alpha_t\lambda^2(1-\lambda^2)L\nu a_2\bigg).
    \end{align*} Choosing $\alpha_t \leq \min\Big\{\frac{\sqrt{n}}{6L\sqrt{2}}, \frac{1}{960L\lambda^2}\frac{a_3}{a_2}\Big\}$ and $a_3 \geq \max\Big\{384\sigma^2\lambda^2\big(2+\log(2d)\big), \frac{3}{2}a_1 \Big\}$, we finally get
    \begin{equation*}
        \E\big[\exp\big(\nu Y^{t+1}\big)\big] \leq \exp\bigg(\frac{1-\lambda^2}{12}\nu a_3 + \frac{3+\lambda^2}{4}\nu a_3 + \frac{1-\lambda^2}{12}\nu a_3 + \frac{1-\lambda^2}{12}\nu a_3 \bigg) = \exp(\nu a_3),
    \end{equation*} proving the induction. To complete the proof, we now verify that $a_1,a_2,a_3$ can be chosen so that all the conditions are satisfied. To that end, recall that we have the following conditions $a_1 \geq \frac{K}{n}\|\bx^1 - \obx^1\|^2$, $a_2 \geq \max\big\{f(\ox^1) - \fs,2+\log(2d),2a_1\big\}$ and $a_3 \geq \Big\{4\lambda^2(1-\lambda^2)^2\Big(\sigma^2 + 4L\big(f(\ox^1) - \fs\big) + \frac{2L^2}{n}\|\bx^1 - \obx^1\|^2\Big),384\sigma^2\lambda^2\big(2+\log(2d)\big),\frac{3}{2}a_1\Big\}$, as well as
    \begin{align*}
        \alpha_t \leq \frac{(1-\lambda^2)^2}{2\lambda\sqrt{K}}\sqrt{\frac{a_1}{a_3}} \quad \text{ and } \quad \alpha_t \leq \frac{1}{960\lambda^2L}\frac{a_3}{a_2}.
    \end{align*} It can be readily verified that choosing $a_1 = \max\Big\{\frac{K}{n}\|\bx^1 - \obx^{1}\|^2, \frac{1}{2}\big(f(\ox^1) - \fs\big), \frac{1}{2}\big(2 + \log(2d)\big), \linebreak 256\sigma^2\lambda^2\big(2+\log(2d)\big), \frac{8}{3}\lambda^2(1-\lambda^2)^2\Big(\sigma^2 + 4L\big(f(\ox^1) - \fs\big) + \frac{2L^2}{n}\|\bx^1 - \obx^1\|^2\Big)\Big\}$, $a_2 = 2a_1$ and $a_3 = \frac{3}{2}a_1$, all the conditions are satisfied if $\alpha_t \leq \min\Big\{\frac{(1-\lambda^2)^2}{\lambda\sqrt{6K}}, \frac{1}{1280\lambda^2L} \Big\}$, where we recall that $K \geq \max\Big\{\frac{2L^2}{\mu},640\lambda^2L^2\Big\}$. 
\end{proof}

\begin{remark}\label{rmk:trans-time-PL-analysis}
    Unlike the MSE analysis in \cite{ran-improved}, where the authors consider $\frac{1}{nL^2}\E\big[\|\byt - \obyt\|^2\big]$, here we bound the MGF of $\|\byt - \obyt\|^2$ scaled by $\frac{(1-\lambda^2)^2}{n}$. This is necessary, as otherwise we would introduce a dependence on $(1-\lambda^2)^2$ in $\nu$, see, e.g., equation \eqref{eq:base-ind-grad-track} ahead. Recalling that the leading term in the rate in Theorem \ref{thm:main-PL} is of the form $\bigO\Big(\frac{\nu^{-1}}{n(t+t_0)}\Big)$, this would result in propagating the effect of network connectivity in the leading term, making it impossible to achieve the known optimal network-independent rate $\bigO\Big(\frac{1}{n(t+t_0)}\Big)$, for any $t \geq 1$.
\end{remark}

We are now ready to prove Lemma \ref{lm:MGF-decay-PL}.

\begin{proof}[Proof of Lemma \ref{lm:MGF-decay-PL}]
    We again proceed via induction. For $t = 1$, recalling that agents have the same initialization (implying that $\Delta_x = 0$), we have $\E\bigg[\exp\bigg(\frac{\nu K_1}{n}\|\bx^1 - \obx^1\|^2 \bigg)\bigg] = 1 \leq \exp\big(\nu\alpha_1^2K_1C)$, for any $C \geq 0$. Define $X^t \coloneqq \frac{\|\bxt - \obxt\|^2}{n}$ and assume that, for some $t \geq 1$ and $C \geq 0$, we have $\E\big[\exp\big(\nu K_tX^t\big)\big] \leq \exp\big(\nu\alpha_t^2K_tC \big)$. Consider iteration $t+1$. Use Lemma \ref{lm:consensus-one-step-model} and define $Y^t \coloneqq \frac{(1-\lambda^2)^2}{n}\|\byt - \obyt\|^2$, to get
    \begin{align}\label{eq:mgf-cons-intermed}
        &\E\big[\exp\big(\nu K_{t+1}X^{t+1}\big)\big] \leq \E\bigg[\exp\bigg(\nu K_{t+1}\bigg(\frac{1+\lambda^2}{2}X^t + \frac{2\alpha_t^2\lambda^2L^2}{(1-\lambda^2)^3}Y^t\bigg)\bigg)\bigg] \nonumber \\
        &\leq \Bigg(\E\bigg[\exp\bigg(\frac{p(1+\lambda^2)(t+t_0+2)}{2(t+t_0+1)}\nu K_tX^t \bigg)\bigg] \Bigg)^{\frac{1}{p}}\Bigg(\E\bigg[\exp\bigg(\frac{2q\alpha_t^2\lambda^2L^2K_{t+1}}{(1-\lambda^2)^3}\nu Y^t \bigg)\bigg] \Bigg)^{\frac{1}{q}},
    \end{align} where the second inequality follows from Proposition \ref{prop:Holder}, for some $p, q \in [1,\infty]$, such that $\frac{1}{p} + \frac{1}{q} = 1$. We now want to choose $p$ such that $\frac{p(1+\lambda^2)(t+t_0+2)}{2(t+t_0+1)} \leq 1$. To that end, note that
    \begin{equation}\label{eq:upper-bnd}
        \frac{p(1+\lambda^2)(t+t_0+2)}{2(t+t_0+1)} \leq \frac{p(1+\lambda^2)(t_0+2)}{2(t_0+1)} = p\bigg(1 - \frac{1-\lambda^2}{2}\bigg)\bigg(1+\frac{1}{t_0+1}\bigg) \leq p\frac{3+\lambda^2}{4},
    \end{equation} where the last inequality follows from the choice $t_0 \geq \frac{3+\lambda^2}{1-\lambda^2}$. Therefore, choosing $p = \frac{4}{3+\lambda^2} > 1$, it follows from \eqref{eq:upper-bnd} that $\frac{p(1+\lambda^2)(t+t_0+2)}{2(t+t_0+1)} \leq 1$. Noting that the choice of $p$ implies $q = \frac{4}{1-\lambda^2} > 1$, plugging everything back in \eqref{eq:mgf-cons-intermed} and using Proposition \ref{prop:Jensen}, we get
    \begin{align*}
        \E\big[\exp\big(\nu K_{t+1}X^{t+1}\big)\big] &\leq \bigg(\E\big[\exp\big(\nu K_tX^t\big)\big]\bigg)^{\frac{(1+\lambda^2)(t+t_0+2)}{2(t+t_0+1)}}\Bigg(\E\bigg[\exp\bigg(\frac{8\alpha_t^2\lambda^2L^2K_{t+1}}{(1-\lambda^2)^4}\nu Y^t\bigg)\bigg]\Bigg)^{\frac{1-\lambda^2}{4}} \\
        &\stackrel{(i)}{\leq} \exp\bigg(\frac{1+\lambda^2}{2}\nu\alpha_t^2K_{t+1}C\bigg)\bigg(\E\big[\exp\big(\nu Y^t \big)\big]\bigg)^{\frac{2\alpha_t^2\lambda^2L^2K_{t+1}}{(1-\lambda^2)^3}} \\
        &\stackrel{(ii)}{\leq} \exp\bigg(\frac{1+\lambda^2}{2}\nu\alpha_t^2K_{t+1}C + \frac{2\alpha_t^2\lambda^2L^2K_{t+1}}{(1-\lambda^2)^3}\nu a_3 \bigg) \\
        &\stackrel{(iii)}{=} \exp\bigg(\nu\alpha_{t+1}^2K_{t+1}C\bigg(\frac{t+t_0+1}{t+t_0}\bigg)^2\bigg(\frac{1+\lambda^2}{2} + \frac{2a_3\lambda^2L^2}{(1-\lambda^2)^3C} \bigg)\bigg)
    \end{align*} where $(i)$ follows from the induction hypothesis, Proposition \ref{prop:Jensen} and the fact that $\alpha_t \leq \frac{(1-\lambda^2)^2}{2\lambda L\sqrt{2K_{t+1}}}$ (follows from $t_0 \geq \frac{16a^2\lambda^2\kappa^2K}{(1-\lambda^2)^4}$), in $(ii)$ we use Lemma \ref{lm:bdd-MGF-PL}, while $(iii)$ follows from the step-size schedule $\alpha_t = \frac{a}{\mu(t+t_0)}$. To complete the proof, we need to show that
    \begin{equation}\label{eq:induc-condition}
        \bigg(\frac{t+t_0+1}{t+t_0} \bigg)^2\bigg(\frac{1+\lambda^2}{2} + \frac{2a_3\lambda^2L^2}{(1-\lambda^2)^3C} \bigg) \leq \bigg(\frac{t_0+1}{t_0} \bigg)^2\bigg(\frac{1+\lambda^2}{2} + \frac{2a_3\lambda^2L^2}{(1-\lambda^2)^3C} \bigg) \leq 1.
    \end{equation} If $C \geq \frac{8a_3\lambda^2L^2}{(1-\lambda^2)^4}$, then \eqref{eq:induc-condition} is satisfied if $\Big(\frac{t_0+1}{t_0}\Big)^2\Big(\frac{1+\lambda^2}{2} + \frac{1 - \lambda^2}{4}\Big) = \Big(\frac{t_0+1}{t_0}\Big)^2\Big(\frac{3+\lambda^2}{4}\Big) \leq 1$. Noting that $\Big(\frac{t_0+1}{t_0}\Big)^2 = 1 + \frac{2t_0 + 1}{t_0^2} \leq 1 + \frac{3}{t_0}$ if $t_0 \geq 1$, it follows that \eqref{eq:induc-condition} is satisfied if
    \begin{equation*}
        \bigg(1 + \frac{3}{t_0}\bigg)\bigg(\frac{3+\lambda^2}{4}\bigg) = \bigg(1 + \frac{3}{t_0}\bigg)\bigg(1 - \frac{1- \lambda^2}{4}\bigg) \leq 1,
    \end{equation*} which clearly holds for $t_0 \geq \frac{12}{1-\lambda^2}$. Therefore, the proof is complete by setting $C = \frac{8a_3\lambda^2L^2}{(1-\lambda^2)^4}$.
\end{proof}

We are now ready to prove Theorem \ref{thm:main-PL}.

\begin{proof}[Proof of Theorem \ref{thm:main-PL}]
    Start from Lemma \ref{lm:descent-inequality-PL} and define $F_t = n(t+t_0)(f(\ox_{t})-\fs)$, to get
    \begin{align*}
        F_{t+1} &\leq (1 - \alpha_t\mu)\frac{t+t_0+1}{t+t_0}F_t - \alpha_t n(t+t_0+1)\langle \nabla f(\oxt), \oz_t \rangle \\
        &+ \alpha_t^2n(t+t_0+1)L\|\oz_t\|^2 + \frac{\alpha_t(t+t_0+1)L^2}{2}\|\bxt - \obxt\|^2.
    \end{align*} Next, consider the MGF of $F_{t+1}$ conditioned on $\Ft$. Let $\nu \in (0,1]$, we then have
    \begin{align}\label{eq:proof-PL-init-recursion}
        \E_t\big[&\exp(\nu F_{t+1})\big] \stackrel{(a)}{\leq} \exp\bigg((1 - \alpha_t\mu)\frac{t+t_0+1}{t+t_0}\nu F_t + \frac{\alpha_t\nu(t+t_0+1) L^2}{2}\|\bxt - \obxt\|^2 \bigg) \nonumber \\
        &\times \E_t\Big[\exp(- \alpha_t\nu n(t+t_0+1)\langle \nabla f(\oxt), \oz_t \rangle + \alpha_t^2\nu n(t+t_0+1) L\|\oz_t\|^2) \Big] \nonumber \\
        &\stackrel{(b)}{\leq} \exp\bigg((1 - \alpha_t\mu)\frac{t+t_0+1}{t+t_0}\nu F_t + \frac{\alpha_t\nu(t+t_0+1) L^2}{2}\|\bxt - \obxt\|^2 \bigg) \nonumber \\
        &\times \sqrt{\E_t\Big[\exp(-2 \alpha_t\nu n(t+t_0+1)\langle \nabla f(\oxt), \oz_t \rangle)\Big]\E_t\Big[\exp(2\nu\alpha_t^2n(t+t_0+1) L\|\oz_t\|^2) \Big]} \nonumber \\
        &\stackrel{(c)}{\leq} \exp\bigg((1 - \alpha_t\mu)\frac{t+t_0+1}{t+t_0}\nu F_t + \frac{\alpha_t\nu(t+t_0+1) L^2}{2}\|\bxt - \obxt\|^2 + \frac{\alpha_t^{2+\varepsilon}n\rho}{2} \bigg) \\
        &\times \exp\bigg(\frac{3\nu^2\alpha_t^2n(t+t_0+1)^2\sigma^2\|\nabla f(\oxt)\|^2}{2} + 96\nu\alpha_t^2\sigma^2(t+t_0+1)L\big(1+\log(2d)\big)\bigg) \nonumber \\ 
        &\times\exp\bigg(3\nu\alpha_t^{6+2\varepsilon}\rho^2(t+t_0+1)L\sum_{i \in [n]}\|\nabla f(\xit)\|^2 + \frac{\alpha_t^{2+\varepsilon}\rho}{2}\sum_{i \in [n]}\|\nabla f(\xit)\|^2\bigg) \nonumber \\
        &\stackrel{(d)}{\leq} \exp\bigg(\nu\Big(b_t^\prime F_t + c_t^\prime\|\bxt - \obxt\|^2 + d_t^\prime\Big) + \frac{\alpha_t^{2+\varepsilon}n\rho}{2} + \frac{2\alpha_t^{2+\varepsilon}\rho LF_t}{t+t_0} + \alpha_t^{2+\varepsilon}\rho L^2\|\bxt - \obxt\|^2 \bigg), \nonumber
    \end{align} where in $(a)$ we used the fact that $F_t$ and $\|\bxt-\obxt\|^2$ are $\Ft$-measurable, $(b)$ follows from Proposition \ref{prop:Holder} with $p = q = 2$, in $(c)$ we use Lemma \ref{lm:avg-noise-properties}, Proposition \ref{prop:Jensen} and impose the condition $\alpha_t \leq \frac{1}{8\sigma\sqrt{3(t+t_0+1)L}}$, while $(d)$ follows from $\alpha_t \leq \sqrt[4+2\varepsilon]{\frac{1}{12\rho^2L}}$, Proposition \ref{prop:descent-ineq} and the definition of $F_t$, with $b_t^\prime = \Big(1 - \alpha_t\mu + 4\nu\alpha_t^2(t+t_0+1)L\Big)\frac{t+t_0+1}{t+t_0}$, $c_t^\prime = \alpha_t(t+t_0+1)L^2$ and $d_t^\prime = 96\alpha_t^2\sigma^2(t+t_0+1)L\big(1+\log(2d)\big)$. Note that choosing $t_0 \geq \frac{1-\nu^2}{\nu^2}$, we get $\frac{2\alpha_t^{2+\varepsilon}\rho LF^t}{t+t_0} \leq 2\nu^2\alpha_t^{2+\varepsilon}LF^t\frac{t+t_0+1}{t+t_0}$ and $\alpha_t^{2+\varepsilon}\rho L^2\|\bxt - \obxt\|^2 \leq \nu^2\alpha_t^{2+\varepsilon}\rho(t+t_0+1)L^2\|\bxt - \obxt\|^2 \leq \nu\alpha_t^{2+\varepsilon}\rho(t+t_0+1)L^2\|\bxt - \obxt\|^2$, where the last inequality follows from $\nu \leq 1$. Requiring $\alpha_t \leq \min\Big\{\sqrt[1+\varepsilon]{\frac{1}{\rho}},\sqrt[\varepsilon]{\frac{1}{n\rho}}\Big\}$, defining $b_t \coloneqq \Big(1 - \alpha_t\mu + 6\nu\alpha_t^2(t+t_0+1)L\Big)\frac{t+t_0+1}{t+t_0}$, $c_t \coloneqq 2\alpha_t(t+t_0+1)L^2$, $d_t = \alpha_t^2(t+t_0+1)\big(1/2 + 96\sigma^2\big(1+\log(2d)\big)L\big)$ and plugging the above relations into \eqref{eq:proof-PL-init-recursion}, we then get
    \begin{equation*}
        \E_t\big[\exp\big(\nu F_{t+1}\big) \big] \leq \exp\Big(\nu\Big(b_tF_t + c_t\|\bxt - \obxt\|^2 + d_t\Big)\Big)
    \end{equation*} Taking the full expectation and applying Proposition \ref{prop:Holder}, we get
    \begin{align}\label{eq:bdd-mgf-str-cvx}
        \E\big[\exp(\nu F_{t+1})\big] &\leq \exp\big(\nu d_t\big) \E\Big[\exp\Big(\nu b_tF_t + \nu c_t\|\bxt-\obxt\|^2 \Big) \Big] \nonumber \\
        &\leq \exp\big(\nu d_t\big) \sqrt[p]{\E\big[\exp\big(\nu pb_tF_t\big)\big]}\sqrt[q]{\E\big[\exp\big(\nu qc_t\|\bxt-\obxt\|^2 \big)\big]}
    \end{align} for some $p,q \in [1,\infty]$. We now analyze $pb_t$. Recalling the definition of $b_t$, we get
    \begin{equation*}
        pb_t = p\bigg(1 - \frac{a}{t+t_0} + \frac{6\nu a^2(t+t_0+1)L}{\mu^2(t+t_0)^2}\bigg)\frac{t+t_0+1}{t+t_0} \leq p\bigg(1 - \frac{a(\mu^2 - 12\nu a^2L)}{\mu^2(t+t_0)}\bigg)\frac{t+t_0+1}{t+t_0}.
    \end{equation*} Choosing $\nu \leq \frac{\mu^2}{24a^2L}$ and $p = 1 + \frac{\alpha_t\mu}{4}$, it follows that
    \begin{equation}\label{eq:pbt-bound}
        pb_t \leq p\bigg(1-\frac{a}{2(t+t_0)} \bigg)\frac{t+t_0+1}{t+t_0} \leq \bigg(1 - \frac{a}{4(t+t_0)}\bigg)\bigg(1 +\frac{1}{t+t_0}\bigg) \leq 1,
    \end{equation} where the last inequality follows since $a \geq 4$. Next, note that the choice of $p = 1+\frac{\alpha_t\mu}{4}$ implies that $q = 1 + \frac{4}{\alpha_t\mu}$. From the definition of $c_t$, we then have
    \begin{align}\label{eq:qct-bound}
        qc_t = 2\bigg(1 + \frac{4}{\alpha_t\mu}\bigg)\alpha_t(t+t_0+1)L^2 = 2\bigg(\alpha_t + \frac{4}{\mu} \bigg)(t+t_0+1)L^2 \leq \frac{4aL^2}{\mu}(t+t_0+1), 
    \end{align} where the inequality follows from $\alpha_t \leq \frac{a}{\mu}$ and $a \geq 6$. Using \eqref{eq:pbt-bound}-\eqref{eq:qct-bound} in \eqref{eq:bdd-mgf-str-cvx}, it follows that
    \begin{align}\label{eq:setup-for-Lemma-5}
        \E\big[\exp(\nu F_{t+1})\big] &\leq \exp\big(d_t\nu\big)\sqrt[p]{\big(\E\big[\exp(\nu F_{t})\big]\big)^{pb_t}}\sqrt[q]{\E\big[\exp\big(\nu(t+t_0+1)4a\kappa L\|\bxt-\obxt\|^2 \big)\big]} \nonumber \\
        &= \exp\big(d_t\nu)\big(\E\big[\exp(\nu F_{t})\big]\big)^{b_t}\sqrt[q]{\E\big[\exp\big(\nu (t+t_0+1)4a\kappa L\|\bxt-\obxt\|^2 \big)\big]}.
    \end{align} Using Lemma \ref{lm:MGF-decay-PL} with $K = 4na\kappa L$, we get $\E\big[\exp\big(\nu q c_t\|\bxt - \obxt\|^2\big)\big] \leq \E\bigg[\exp\bigg(\frac{\nu K_t}{n}\|\bxt - \obxt\|^2\bigg)\bigg] \leq \exp\Big(\nu\alpha_t^2K_tC\Big)$. Noting that $\frac{1}{q} = \frac{\alpha_t\mu}{4+\alpha_t\mu} \leq \frac{\alpha_t\mu}{4}$, we finally get
    \begin{equation*}
        \sqrt[q]{\E\big[\exp\big(\nu q c_t\|\bxt - \obxt\|^2\big)\big]} \leq \exp\bigg(\frac{\nu\mu\alpha_t^3K_tC}{4}\bigg) \leq \exp\bigg(\frac{2\nu na^4\kappa^2C}{\mu(t+t_0)^2} \bigg). 
    \end{equation*} Define $G_1 \coloneqq \frac{a^2}{\mu^2} + \frac{192\kappa a^2\sigma^2(1+\log(2d))}{\mu}$, $G_2 \coloneqq \frac{2na^4\kappa^2C}{\mu}$ and plug into \eqref{eq:setup-for-Lemma-5}, to get
    \begin{equation*}
        \E\big[\exp\big(\nu F_{t+1}\big)\big] \leq \Big(\E\big[\exp\big(\nu F_t \big)\big]\Big)^{b_t}\exp\bigg(\frac{\nu G_1}{t+t_0} + \frac{\nu G_2}{(t+t_0)^2}\bigg).
    \end{equation*} Recalling the definition of $b_t$ and \eqref{eq:pbt-bound}, it follows that $b_t \leq 1 - \frac{a/2-1}{t+t_0}$, therefore we can bound the MGF of $\nu F_{t+1}$ using Lemma \ref{lm:mgf-bound-str-cvx} and choosing $a \geq 6$, to  finally get
    \begin{align}\label{eq:bdd-mgf}
        \E\big[\exp\big( \nu F_{t+1} \big) \big] &\leq \exp\bigg(\frac{n\nu(t_0+1)^{a/2} \Delta_f}{(t+1+t_0)^{a/2-1}} + \frac{2\nu G_1}{a-2} + \frac{6\nu G_2/(a-4)}{t+1+t_0}\bigg).
    \end{align} Applying Markov's inequality, we then get, for any $\epsilon > 0$
    \begin{equation*}
        \Prob\lp f(\oxtp) - \fs > \epsilon \rp = \Prob\big( \nu F_{t+1} > \nu n(t+1+t_0)\epsilon \big) \leq \exp\lp -\nu n(t+1+t_0)\epsilon \rp\E\big[\big(\nu F_{t+1} \big) \big].
    \end{equation*} Using \eqref{eq:bdd-mgf}, it can be readily verified that choosing 
    \begin{align}\label{eq:eps-1}
        \epsilon_t^1 &= \frac{\nu^{-1}\log(\nicefrac{1}{\delta}) + 2G_1/(a-2)}{n(t+t_0)} + \frac{6G_2/(a-4)}{n(t+t_0)^2} + \frac{(t_0+1)^{a/2}\Delta_f}{(t+t_0)^{a/2}}, 
    \end{align} for any $\delta \in (0,1)$, results in $\Prob\lp f(\oxt) - \fs > \epsilon_t^1 \rp \leq \delta$. Next, using Proposition \ref{prop:descent-ineq} with $x = \xit$ and $y = \oxt$, we get
    \begin{align*}
        f(\xit) &\leq f(\oxt) + \langle \nabla f(\oxt), \xit - \oxt\rangle + \frac{L}{2}\|\xit-\oxt\|^2 \\
        &\stackrel{(i)}{\leq} f(\oxt) + \frac{1}{2L}\|\nabla f(\oxt)\|^2 + \frac{L}{2}\|\xit - \oxt\|^2 + \frac{L}{2}\|\xit-\oxt\|^2 \\
        &\stackrel{(ii)}{\leq} f(\oxt) + f(\oxt) - \fs + L\|\xit-\oxt\|^2,
    \end{align*} where in $(i)$ we used Proposition \ref{prop:Young} with $\epsilon = L$, while $(ii)$ follows from Proposition \ref{prop:descent-ineq}. Subtracting $\fs$ from both sides and averaging over all agents $i \in [n]$, we get
    \begin{equation}\label{eq:global-local-cost}
        \frac{1}{n}\sum_{i \in [n]}\big(f(\xit)-\fs\big) \leq 2\big(f(\oxt)-\fs\big) + \frac{L}{n}\|\bxt-\obxt\|^2.
    \end{equation} We now consider two events, $A_{t,\epsilon} \coloneqq \lcb \omega: f(\oxt)-\fs > \epsilon \rcb$ and $B_{t,\epsilon} \coloneqq \lcb \omega: \frac{L}{n}\|\bxt - \oxt\|^2 > \epsilon \rcb$. From the previous analysis, we know that, for any $\delta \in (0,1)$ and $\epsilon_t^1$ from \eqref{eq:eps-1}, we have $\Prob\big( A_{t,\epsilon_t^1}\big) \leq \delta$. Using Markov's inequality and Lemma \ref{lm:MGF-decay-PL} with $K = nL$, we have, for any $\epsilon > 0$
    \begin{align*}
        \Prob\bigg( \frac{L}{n}\|\bxt - \obxt\|^2 > \epsilon \bigg) &= \Prob\bigg(\frac{\nu K_{t}}{n}\| \bxt - \obxt\|^2 > \nu n(t+t_0+1)\epsilon\bigg) \\
        &\leq \exp\big(-\epsilon\nu n(t+t_0+1)\big)\E\bigg[\exp\bigg(\frac{\nu K_{t}}{n}\|\bxt - \obxt\|^2 \bigg) \bigg] \\
        &\leq \exp\Big( -\epsilon\nu n(t+t_0+1) + \nu\alpha_t^2K_tC\Big) \\
        &= \exp\Big(n\nu(t+t_0+1)\big(\alpha_t^2LC - \epsilon \big) \Big).
    \end{align*} Therefore, it can be readily seen that, choosing 
    \begin{align}\label{eq:eps-2}
        \epsilon_t^2 = \frac{\nu^{-1}\log(\nicefrac{1}{\delta})}{n(t+t_0+1)} + \frac{a^2\kappa C}{\mu(t+t_0)^2},  
    \end{align} we get $\Prob\big(B_{t,\epsilon_t^2}\big) \leq \delta$, for any $\delta \in (0,1)$. Finally, let $C_t \coloneqq \lcb \omega: \frac{1}{n}\sum_{i \in [n]}\big(f(\xit)-\fs\big) > 2\epsilon_t^1+\epsilon_t^2 \rcb$. From \eqref{eq:global-local-cost} it readily follows that, for any $\delta \in (0,\nicefrac{1}{2})$, we have $\Prob(C_t) \leq \Prob\big(A_{t,\epsilon_t^1}\cap B_{t,\epsilon_t^2}\big) \leq \Prob\big(A_{t,\epsilon_t^1}\big) + \Prob\big(B_{t,\epsilon_t^2}\big) \leq 2\delta$. Therefore, we get, for any $\delta \in (0,1)$, with probability at least $1 - \delta$
    \begin{align*}
        \frac{1}{n}\sum_{i \in [n]}\big(f(\xit)-\fs\big) = \mathcal{O}\bigg(\frac{\nu^{-1}\log(\nicefrac{2}{\delta}) + \sigma^2\kappa\big(1+\log(2d)\big)/\mu}{n(t+t_0)} + \frac{\kappa C(1 + \kappa)}{\mu(t+t_0)^2} + \frac{(t_0+1)^{a/2}\Delta_f}{(t+t_0)^{a/2}} \bigg).
    \end{align*} Finally, it can be verified that the conditions on $a$, $t_0$ and $\nu$ in the statement of the theorem ensure that all the step-size conditions are satisfied, completing the proof.
\end{proof}

\section{On the Transient Time}\label{sup:trans-time}

In this section we give details on the transient time resulting from our bounds. We focus on the dependence on the network connectivity and number of agents, ignoring other problem constants, as in \cite{unified-refined}. For non-convex costs, the bound in Theorem \ref{thm:main-non-cvx} is of the form
\begin{equation}\label{eq:tran-t-non-conv-1}
    \bigO\bigg( \frac{1}{\sqrt{nT}} + \frac{n^{\nicefrac{(3+\varepsilon)}{2}}\rho}{T^{\nicefrac{(1+\varepsilon)}{2}}} + \frac{1}{CT} + \frac{1}{n(1-\lambda^2)T} + \frac{n}{(1-\lambda^2)^4T} + \frac{n}{(1-\lambda^2)^3T^2} \bigg),    
\end{equation} where the problem related constant $C$ depends on the network connectivity via the expression $C = \bigO\big(\min\big\{(1-\lambda^2)^2, 1-\lambda^2, (1-\lambda^2)^{4/3}, \sqrt[3]{n}(1-\lambda^2)^{4/3} \big\}\big)$. Since we are interested in the worst-case transient time, corresponding to $\lambda \approx 1$, it follows that $C = \bigO\big((1-\lambda^2)^2\big)$. In this case, the bound in \eqref{eq:tran-t-non-conv-1} becomes
\begin{equation}\label{eq:tran-t-non-conv-2}
    \bigO\bigg( \frac{1}{\sqrt{nT}} + \frac{n^{\nicefrac{(3+\varepsilon)}{2}}\rho}{T^{\nicefrac{(1+\varepsilon)}{2}}} + \frac{1}{(1-\lambda^2)^2T} + \frac{1}{n(1-\lambda^2)T} + \frac{n}{(1-\lambda^2)^4T} + \frac{n}{(1-\lambda^2)^3T^2} \bigg).    
\end{equation} It is now evident that, for any time $T \geq \bigO\Big(\max\Big\{n^{\frac{4+\varepsilon}{\varepsilon}}\rho^{\frac{2}{\varepsilon}},\frac{n}{(1-\lambda^2)^4}, \frac{1}{(1-\lambda^2)^2}, \frac{n^3}{(1-\lambda^2)^8}, \frac{n^{2/3}}{(1-\lambda^2)^2} \Big\}\Big) \linebreak = \bigO\Big(\max\Big\{n^{\frac{4+\varepsilon}{\varepsilon}}\rho^{\frac{2}{\varepsilon}},\frac{n^3}{(1-\lambda)^8}\Big\}\Big)$, the bound in \eqref{eq:tran-t-non-conv-2} becomes $\bigO\Big(\frac{1}{\sqrt{nT}}\Big)$, implying a transient time of order $\bigO\Big(\max\Big\{n^{\frac{4+\varepsilon}{\varepsilon}}\rho^{\frac{2}{\varepsilon}},\frac{n^3}{(1-\lambda^2)^8}\Big\}\Big)$. The first term is induced by the relaxed sub-Gaussianity condition and vanishes as $\varepsilon \rightarrow \infty$, or when $\rho = 0$ (i.e., standard sub-Gaussianity). As discussed in the main body, the term $\bigO\Big(\frac{n^3}{(1-\lambda^2)^8}\Big)$ is slightly worse than the transient time stemming from the \gtdsgd-specific analysis in \cite{ran-improved}, which is of order $\bigO\Big(\frac{n^3}{(1-\lambda^2)^6} \Big)$. However, as mentioned in Remark \ref{rmk:sharp-constants}, this stems from the slightly less tight constants in our intermediate results, which themselves stem from the need to first provide deterministic bounds that can be used in the MGF, while in the MSE analysis it is possible to directly bound the expected quantities. For P\L{} costs, the bound in Theorem \ref{thm:main-PL} is of the form
\begin{align}\label{eq:tran-t-str-cvx-1}
    \mathcal{O}\bigg(\frac{1}{n(T+t_0)} + \frac{C}{(T+t_0)^2} + \frac{t_0^{a/2}}{(T+t_0)^{a/2}} \bigg),
\end{align} where $a \geq 6$ is a user-specified constant, while $t_0$, $C$ are of the form $t_0 = \Omega\Big(\max\Big\{\frac{1}{1-\lambda^2},\frac{1}{(1-\lambda^2)^2}, \linebreak \frac{n}{(1-\lambda^2)^4}, \sqrt[1+\varepsilon]{\rho n}\Big\}\Big)$ and $C = \Omega\Big(\max\Big\{\frac{1}{(1-\lambda^2)^2}, \frac{1}{(1-\lambda^2)^4} \Big\}\Big)$. Since we are interested in the worst-case transient time, we set $t_0 = \Omega\Big(\frac{n}{(1-\lambda^2)^4}\Big)$ and $C = \Omega\Big(\frac{1}{(1-\lambda^2)^4} \Big)$, and for ease of exposition, let $a = 2b$, for a user-specified $b \geq 3$, so that the bound in \eqref{eq:tran-t-str-cvx-1} becomes
\begin{align}\label{eq:tran-t-str-cvx-2}
    &\quad\quad\bigO\bigg(\frac{1}{n(T+t_0)} + \frac{n^b}{(1-\lambda^2)^{4b}(T+t_0)^b} \bigg).
\end{align} It can be seen that for $T \geq \frac{n^{\frac{b+1}{b-1}}}{(1-\lambda^2)^{\frac{4b}{b-1}}}$, the bound in \eqref{eq:tran-t-str-cvx-2} becomes $\mathcal{O}\Big(\frac{1}{n(t+t_0)}\Big)$, implying a transient time of order $\mathcal{O}\bigg(\frac{n^{\frac{b+1}{b-1}}}{(1-\lambda^2)^{\frac{4b}{b-1}}}\bigg)$, for any $b \geq 3$. Choosing $b = \frac{2+\epsilon}{\epsilon}$, for any $\epsilon \in (0,1]$, we get a transient time of order $\bigO\Big(\frac{n^{1+\epsilon}}{(1-\lambda^2)^{4+2\epsilon}}\Big)$, which can be made arbitrarily close to $\bigO\Big(\frac{n}{(1-\lambda^2)^4} \Big)$, by choosing a small $\epsilon$. This is close to the transient time $\bigO\Big(\frac{n}{(1-\lambda^2)^3}\Big)$ stemming from the MSE rate in \cite{ran-improved}. However, this comes at the cost of slightly larger multiplicative constants in Theorem \ref{thm:main-PL}, recalling that the final rate is of the form $\bigO\Big(\frac{b}{n(t+t_0)}\Big) = \bigO\Big(\frac{2}{n\epsilon(t+t_0)}\Big)$. 

\end{document}